\documentclass{article}
\usepackage[preprint,helvetica]{iclr2027_conference}
\usepackage{times}
\usepackage{amsmath,amssymb,mathtools,amsthm,booktabs,tabularx, dsfont}
\usepackage{algorithm,algpseudocode}
\usepackage{graphicx}
\usepackage{xcolor}
\usepackage{tikz}
\usetikzlibrary{calc,arrows.meta}
\usepackage{url}
\usepackage{thmtools}
\usepackage{optidef}

\usepackage{hyperref}
\definecolor{mydarkblue}{RGB}{0,0,139}

\usepackage{pgfplots}
\usepackage{pgfplotstable}
\pgfplotsset{compat=newest}
\usepgfplotslibrary{groupplots,fillbetween}

\pgfdeclareplotmark{rank simplex star}{%
  \pgfpathmoveto{\pgfpointpolar{90}{\pgfplotmarksize}}%
  \pgfpathlineto{\pgfpointpolar{126}{0.42\pgfplotmarksize}}%
  \pgfpathlineto{\pgfpointpolar{162}{\pgfplotmarksize}}%
  \pgfpathlineto{\pgfpointpolar{198}{0.42\pgfplotmarksize}}%
  \pgfpathlineto{\pgfpointpolar{234}{\pgfplotmarksize}}%
  \pgfpathlineto{\pgfpointpolar{270}{0.42\pgfplotmarksize}}%
  \pgfpathlineto{\pgfpointpolar{306}{\pgfplotmarksize}}%
  \pgfpathlineto{\pgfpointpolar{342}{0.42\pgfplotmarksize}}%
  \pgfpathlineto{\pgfpointpolar{18}{\pgfplotmarksize}}%
  \pgfpathlineto{\pgfpointpolar{54}{0.42\pgfplotmarksize}}%
  \pgfpathclose
  \pgfusepathqfillstroke
}

\newtheorem{theorem}{Theorem}
\newtheorem{proposition}[theorem]{Proposition}

\theoremstyle{definition}

\theoremstyle{remark}

\newcommand{\R}{\mathbb R}

\newcommand{\cN}{\mathcal N}
\newcommand{\cA}{\mathcal A}

\newcommand{\cC}{\mathcal C}
\newcommand{\cI}{\mathcal I}
\newcommand{\cT}{\mathcal T}
\newcommand{\cV}{\mathcal V}
\newcommand{\cW}{\mathcal W}
\newcommand{\cU}{\mathcal U}

\newcommand{\phibar}{\overline{\Phi}}

\newcommand{\Id}{\mathbb I_d}
\newcommand{\Jsel}{J}

\newcommand{\RankWeights}{\textsc{RankWeights}}
\newcommand{\pa}{\operatorname{pa}}
\newcommand{\dd}{\,\mathrm d}

\title{Rank-Aware Speculative Sampling\\for Diffusion Draft Trees}
\author{Marcello Bullo, Yanxiao Liu, \"{O}yk\"{u} S\i la G\"{u}ner, Arpan Mukherjee, Deniz G\"{u}nd\"{u}z \\
Imperial College London \\
\texttt{\{m.bullo21,y.liu2,o.guner25,a.mukherjee,d.gunduz\}@imperial.ac.uk}}

\begin{document}
\maketitle

\begin{abstract}
Speculative sampling accelerates diffusion generation by verifying inexpensive draft states in parallel while preserving the target law. Recent tree-based methods allocate the parallel compute budget more effectively than single-chain drafts, as demonstrated by Diffusion Greedy Rejection Sampling (D-GRS). D-GRS generates $K$ conditionally independent candidates per node, and sequentially tests them in their generation order. Yet the sampled candidates admit an informative ranking without additional target-model evaluations. To exploit this, we introduce Rank-Aware Speculative Sampling (RASS), a verification rule for speculative draft trees based on rank-aware list coupling. RASS orders draft candidates along the proposal–target mean displacement and samples a rank with weights optimized to minimize total variation between the selected-proposal and target laws. Finally, the selected candidate is maximally coupled with the target, with residual correction ensuring \emph{exact} sampling for any choice of rank weights.
We evaluate RASS on a Gaussian-mixture target, unconditional pixel-space generation on FFHQ, conditional generation on CIFAR-10, and latent diffusion with Stable Diffusion 3.5 using COCO2014 prompts. Measured by the ratio of standard to speculative sampling’s target-model evaluation counts, RASS improves on D-GRS across the evaluated settings, with gains reaching approximately 20\% on CIFAR-10 at matched compute budgets.
\end{abstract}

\section{Introduction}
\label{sec:intro}







Diffusion and flow-based models \citep{sohl2015deep,ho2020denoising,song2021score,lipman2022flow} produce state-of-the-art samples across modalities \citep{yang2023diffusion}; however, each sample requires tens to hundreds of sequential evaluations of a large network, incurring significant delay at inference. To circumvent this challenge, a body of literature trades exactness for speed through fewer-step solvers~\citep{song2021ddim,zhang2023fast,lu2025dpmpp} or distilled few-step generators \citep{salimans2022progressive,song2023consistency,sabour2025align,yin2024improved}. Distinctly, \citet{debortoli2025speculative} introduces \emph{speculative diffusion sampling}, which uses a cheaper drafter to generate proposals that are subsequently corrected to yield \emph{exact} target samples, thereby reducing the number of expensive target-model evaluations and accelerating inference. The central mechanism enabling speculative sampling is parallel verification of multiple drafted states. Its speed-up therefore hinges on how effectively the available parallel compute budget is allocated.


Reflection maximal coupling (RMC), introduced by \citet{debortoli2025speculative}, drafts a single chain in which an early rejection discards all states succeeding the inexact draft state. \citet{bullo2026drafttrees} show that \emph{draft trees}, which spend the budget on $K$ candidate children per node, use the same batch more effectively, and propose diffusion greedy rejection sampling (D-GRS), a verification rule adapted from relative entropy coding \citep{harsha2010communication,flamich2023adaptive} that tests the $K$ children sequentially in their drafting order. Each rejection updates the unassigned target mass and the next child is tested against it.

\begin{figure}[t]
  \centering
  \includegraphics[width=0.8\linewidth]{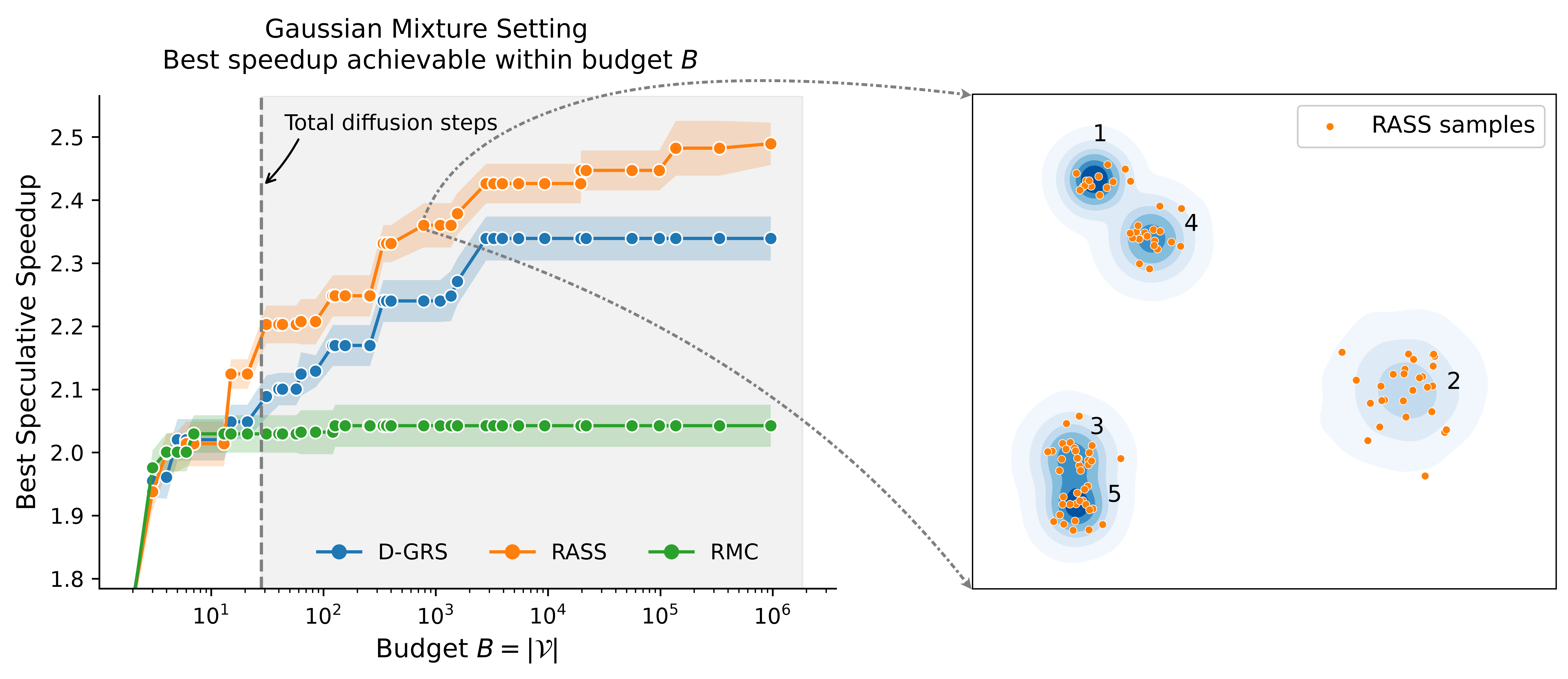}  
\caption{Target-call efficiency and sample visualization on the Gaussian-mixture target ($\varepsilon=0.3$). \textbf{Left:} Best observed mean target-call speedup within each parallel budget $B=|\mathcal V|$, with shaded $95\%$ confidence intervals. The vertical dashed line marks the number of diffusion steps. RASS achieves higher speedups than D-GRS and RMC throughout
the gray region.
\textbf{Right:} RASS samples overlaid on target-density contours, with numbers identifying the five mixture components. Experimental details are provided in Section~\ref{sec:experiments}.}
  \label{fig:intro-frontier}
\end{figure}


We elicit a fundamental limitation of D-GRS: proposals are verified in the order in which they are gleaned, imposing a sequence constraint that is not required by the underlying problem. Indeed, the $K$ children of a node are conditionally independent and identically distributed (i.i.d.) draws from the same Gaussian proposal and are therefore exchangeable at verification. Enforcing an arbitrary order thus discards exploitable symmetry and can be strictly suboptimal relative to list coupling. Building on this observation, we introduce \textbf{rank-aware speculative sampling (RASS)}, a verification rule for diffusion draft trees that exploits the order statistics of the drafted proposals.

More specifically, RASS ranks proposals by their target-to-proposal likelihood ratios, prioritizing candidates that are most compatible with the target distribution. A naive strategy would always select the highest-ranked proposal, but this breaks exactness in an important regime: when the proposal and target distributions coincide, all drafts are distributionally equivalent and no rank should be preferred. RASS therefore places a discrepancy-dependent distribution over proposal ranks, interpolating between uniform selection when proposal and target are close and increasingly concentrated selection on the highest-likelihood proposal as their discrepancy grows.

\paragraph{Contributions.}
Our contributions are threefold: 

(i) We identify a fundamental limitation of the state-of-the-art speculative diffusion sampler D-GRS~\citep{bullo2026drafttrees}: its verification rule is based on sequence coupling, even though the children of a draft-tree node are exchangeable and list coupling can only improve upon an imposed ordering. Motivated by this gap, we introduce \textbf{rank-aware speculative sampling (RASS)}, which ranks i.i.d.\ proposals by their target-to-proposal likelihood ratios and uses this ranking to construct a stronger proposal distribution for verification (Section~\ref{sec:rass}).

(ii) We derive a universal upper bound on the acceptance probability achievable by any list-coupling scheme in our setting (Section~\ref{sec:analysis}, Theorem~\ref{prop:optimal-list-acceptance}). Existing bounds based on the hockey-stick divergence allow arbitrary dependence among the proposals~\citep{phan2026multi}; however, such coupling is unavailable in speculative sampling, where the drafts are generated i.i.d.\ from the proposal distribution. Exploiting this independence constraint, we obtain a strictly tighter characterization of the attainable acceptance probability.

(iii) We prove that RASS preserves the target law for every feasible choice of rank weights (Theorem~\ref{thm:exact}). Numerical comparisons show up to approximately $13\%$ higher per-node acceptance probability than D-GRS with the same number of drafts. Across a Gaussian-mixture target, CIFAR-10, FFHQ, and Stable Diffusion~3.5, RASS consistently improves mean target-call speedup, with gains reaching $20.9\%$ at matched draft trees and parallel compute budgets. On CIFAR-10, it exceeds D-GRS's best observed mean speedup using $81.5\%$ less parallel budget for every tested churn value.

\section{Background}
\label{sec:background}

\paragraph{Notation.}
We use uppercase letters for random variables and corresponding lowercase letters for their realizations, and write $X_{1:n}=(X_1,\ldots,X_n)$. The state dimension is $d$, $\Id$ denotes the identity matrix, and $\|\!\cdot\!\|$ is the Euclidean norm. We write $\mathcal N(\mu,\Sigma)$ for the Gaussian law with mean $\mu$ and covariance $\Sigma$, and $\mathcal N(x;\mu,\Sigma)$ for its density at $x$ when $\Sigma$ is positive definite. The standard Gaussian density and cumulative distribution function are denoted by $\phi$ and $\Phi$, with $\bar\Phi=1-\Phi$. We write $\mathcal U(a,b)$ for the uniform law on $(a,b)$. For probability laws $P$ and $Q$, $\Pi(P,Q)$ denotes the set of joint laws with marginals $P,Q$, and $P^{\otimes K}$ is the joint law of $K$ independent draws from $P$. For probability densities $p,q$, their total variation distance is $\mathrm{TV}(p,q)=\frac12\int |p(x)-q(x)|\,dx$. Finally, $(x)_+=\max\{x,0\}$ denotes the positive part.

\paragraph{Standard Diffusion Sampling.}
Let $X_0\sim q_{\mathrm{data}}$ and consider the forward diffusion $\dd X_t=f_t(X_t)\dd t+g_t\dd B_t$, $t\in[0,1]$
with marginals $q_t$ and score $s_t(x)=\nabla_x\log q_t(x)$, where the dynamics are chosen such that $q_1\approx\cN(0,\Id)$ \citep{song2021score}. A family of reverse processes with the same marginals is
\begin{equation}
    \dd Y_t=b_t(Y_t)\dd t+\varepsilon g_{1-t}\dd W_t,\qquad
    b_t(x)=-f_{1-t}(x)+\frac{1+\varepsilon^2}{2}g_{1-t}^2s_{1-t}(x),
    \label{eq:reverse-sde}
\end{equation}
where $\varepsilon\ge0$ controls the stochasticity \citep{song2021score,karras2022elucidating,albergo2023stochastic}. 
In practice, $b_t$ (or, equivalently, $s_{1-t}$) is approximated by a neural network trained using denoising score matching or flow matching \citep{song2021score,lipman2022flow}. At inference time, we use the learned drift $b_t^q$ to simulate the reverse process with Euler--Maruyama on the uniform grid $t_n=n\gamma$, $n=0,\ldots,N$, where $\gamma=1/N$:
\begin{equation}
Y_{n+1}
=Y_n+\gamma b^q_{t_n}(Y_n)
+\sqrt{\gamma},\varepsilon g_{1-t_n}\xi_n,
\qquad \xi_n\sim\cN(0,\Id).
\label{eq}
\end{equation}
Hence the target transition at step $n$ is
\begin{equation}
    Q_n(\cdot\mid y)
    =\cN\!\left(m_n^q(y),\sigma_n^2\Id\right),\qquad
    m_n^q(y)=y+\gamma b^q_{t_n}(y),\quad
    \sigma_n=\sqrt{\gamma}\,\varepsilon g_{1-t_n}.
    \label{eq:target-kernel}
\end{equation}
Evaluating $m_n^q$ requires one neural-network evaluation, so generating an $N$-step trajectory requires $N$ sequential evaluations. 

\paragraph{Speculative diffusion sampling.}
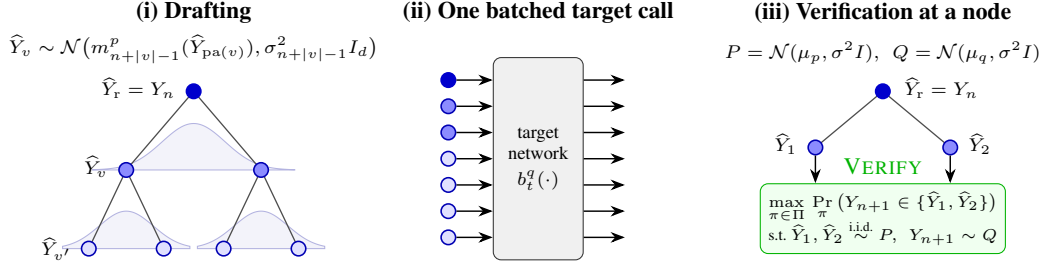
\begin{figure}[t]
  \centering
  \resizebox{\linewidth}{!}{
\begin{tikzpicture}[
  x=1cm,y=1cm,
  node/.style={circle,draw,minimum size=5.5pt,inner sep=0pt,fill=white,line width=0.6pt},
  hi/.style={node,draw=orange!90!black,fill=orange!35,line width=0.8pt},
  edge/.style={draw=black!70,line width=0.5pt},
  hedge/.style={draw=orange!90!black,line width=0.9pt},
  hedgeblk/.style={draw=black,line width=0.9pt},
  feed/.style={->,>=Stealth,draw=blue!70!black,line width=0.55pt},
  blkfeed/.style={->,>=Stealth,draw=black,line width=0.55pt},
  hfeed/.style={->,>=Stealth,draw=orange!90!black,line width=0.6pt},
  box/.style={draw=white!20!black,rounded corners=3pt,fill=black!6,inner sep=4pt,align=center},
  verifybox/.style={draw=green!70!black,rounded corners=3pt,fill=green!6,inner sep=4pt,align=center},
  lab/.style={font=\scriptsize},
  tlab/.style={font=\footnotesize\bfseries},
  d0/.style={node,draw=blue!80!black,fill=blue!80!black},
  d1/.style={node,draw=blue!70!black,fill=blue!45},
  d2/.style={node,draw=blue!70!black,fill=blue!12},
  bell/.style={draw=blue!60!black,fill=blue!60!black,fill opacity=0.05,draw opacity=0.3,line width=0.4pt},
]
\begin{scope}[shift={(0,0)}]
\node[tlab] at (0,1.05) {(i) Drafting};
\coordinate (r) at (0,0);
\coordinate (a) at (-0.9,-1.05); \coordinate (b) at (0.9,-1.05);
\coordinate (aa) at (-1.4,-2.1); \coordinate (ab) at (-0.4,-2.1);
\coordinate (ba) at (0.4,-2.1);  \coordinate (bb) at (1.4,-2.1);
\draw[bell] plot[domain=-1.35:1.35,samples=50,smooth] (\x,{-1.05+0.62*exp(-\x*\x/(2*0.42*0.42))}) -- (1.35,-1.05) -- (-1.35,-1.05) -- cycle;
\draw[bell] plot[domain=-1.75:-0.05,samples=50,smooth] (\x,{-2.1+0.5*exp(-(\x+0.9)*(\x+0.9)/(2*0.36*0.36))}) -- (-0.05,-2.1) -- (-1.75,-2.1) -- cycle;
\draw[bell] plot[domain=0.05:1.75,samples=50,smooth] (\x,{-2.1+0.5*exp(-(\x-0.9)*(\x-0.9)/(2*0.36*0.36))}) -- (1.75,-2.1) -- (0.05,-2.1) -- cycle;
\draw[edge] (r)--(a) (r)--(b) (a)--(aa) (a)--(ab) (b)--(ba) (b)--(bb);
\node[d0] at (r) {}; \node[d1] at (a) {}; \node[d1] at (b) {};
\node[d2] at (aa) {}; \node[d2] at (ab) {}; \node[d2] at (ba) {}; \node[d2] at (bb) {};
\node[lab,anchor=east] at ($(r)+(-0.12,0.05)$) {$\widehat Y_{\mathrm r}=Y_n$};
\node[lab,anchor=east] at ($(a)+(-0.12,0.05)$) {$\widehat Y_{v}$};
\node[lab,anchor=east] at (-1.5,-2.1) {$\widehat Y_{v'}$};
\node[lab,anchor=south,align=center] at ($(r)+(0,0.25)$) {$\widehat Y_{v}\sim\mathcal N\big(m^p_{n+|v|-1}(\widehat Y_{\mathrm{pa}(v)}),\sigma^2_{n+|v|-1}I_d\big)$};
\end{scope}
\begin{scope}[shift={(4.6,0)}]
\node[tlab] at (0,1.05) {(ii) One batched target call};
\foreach \sty/\y/\nm in {d0/0.15/r, d1/-0.2/a, d1/-0.55/b, d2/-.9/aa, d2/-1.25/ab, d2/-1.60/ba, d2/-1.95/bb}
  \node[\sty] (s\nm) at (-1.2,\y) {};
\node[box,font=\scriptsize,minimum height=2.7cm,minimum width=1.2cm] (net) at (0,-.9)
  {target\\ network\\[2pt] $b^q_{t}(\cdot)$};
\foreach \nm in {r,a,b,aa,ab,ba,bb} {
  \draw[blkfeed] (s\nm) -- (net.west |- s\nm);
  \draw[blkfeed] (net.east |- s\nm) -- ($(net.east |- s\nm)+(0.55,0)$);
}
\end{scope}
\begin{scope}[shift={(9.2,0)}]
\node[tlab] at (0,1.05) {(iii) Verification at a node};
\coordinate (r) at (0,0);
\coordinate (a) at (-0.9,-0.75); \coordinate (b) at (0.9,-0.75);
\draw[edge] (r)--(a) (r)--(b);
\node[d0] at (r) {}; \node[d1] (c1) at (a) {}; \node[d1] (c2) at (b) {};
\node[lab,anchor=east] at ($(r)+(1.4,0.05)$) {$\widehat Y_{\mathrm r}=Y_n$};
\node[lab,anchor=east] at ($(a)+(-0.12,0.05)$) {$\widehat Y_1$};
\node[lab,anchor=west] at ($(b)+(0.12,0.05)$) {$\widehat Y_2$};
\node[lab,anchor=south] at ($(r)+(0,0.25)$) {$P=\mathcal N(\mu_p,\sigma^2I),\ \ Q=\mathcal N(\mu_q,\sigma^2I)$};
\node[verifybox,minimum width=2.0cm,inner sep=3pt] (ver) at (0,-1.7)
  {\tiny$\displaystyle\max_{\pi\in\Pi}\ \Pr_\pi\big(Y_{n+1}\in\{\widehat Y_1,\widehat Y_2\}\big)$\\[-2.pt]
   \tiny s.t.\ $\widehat Y_1,\widehat Y_2\overset{\text{i.i.d.}}{\sim}P,\ \ Y_{n+1}\sim Q$};
\node[green!70!black] at (0,-1.) {\footnotesize\textsc{Verify}};
\draw[->,>=Stealth,draw=black,line width=0.55pt] (c1) -- (c1 |- ver.north);
\draw[->,>=Stealth,draw=black,line width=0.55pt] (c2) -- (c2 |- ver.north);
\end{scope}
\end{tikzpicture}}
  \caption{One speculative round on a $(K,L)=(2,2)$ draft tree. (i) Candidates are drawn recursively from the proposal kernel. (ii) Target drifts required for verification are evaluated in one batch. (iii) Verification walks down the tree while candidates are accepted. The first rejection is corrected to an exact target draw and terminates the round.}
  \label{fig:stages}
\end{figure}
Speculative diffusion sampling reduces this sequential dependence by generating candidate future states with an inexpensive proposal and evaluating the target model at these candidates in parallel \citep{debortoli2025speculative}. A subsequent verification step accepts or corrects the candidates while preserving exactly the law of the target Markov chain. We call this three-stage procedure a \emph{speculative round}, illustrated in Figure~\ref{fig:stages}. In the remainder of this section, we summarize its three stages following \citet{bullo2026drafttrees}, and we defer the full algorithm to  Algorithm~\ref{alg:speculative-template}.

(i) \underline{Drafting}: To generate candidate future states inexpensively, speculative diffusion sampling uses a proposal transition. Analogously to \eqref{eq:target-kernel}, we consider $P_n(\cdot\mid y)=\cN\!\left(m_n^p(y),\sigma_n^2\Id\right)$, where $m_n^p(y)=y+\gamma b^p_{t_n}(y)$ and $b^p_{t_n}$ is an inexpensive approximation to the target drift $b^q_{t_n}$. The proposal thus shares the target covariance but uses a cheaper mean computation. In this work, following \citep{debortoli2025speculative}, we use a \emph{delayed-target proposal}, which holds a cached target-model evaluation fixed throughout drafting and updates it between rounds using the available parallel evaluations (see Appendix~\ref{app:delayed-drift-prefetch}). The drafting procedure can be organized as a finite rooted tree $\cT=(\cV,\pa)$ of candidate states \citep{bullo2026drafttrees}, where $\cV$ is the node set, $\mathrm r\in\cV$ is the root, and $\pa:\cV\setminus\{\mathrm r\}\to\cV$ maps each non-root node to its parent. For $u\in\cV$, let $C(u):=\{v\in\cV:\pa(v)=u\}$ denote its set of children, and let $|u|$ denote its depth, with $|\mathrm r|=0$. The chain construction of \citep{debortoli2025speculative} is recovered as the degenerate case with branching factor one. The root $\mathrm r$ contains the current state $Y_n$, and each node $v\neq \mathrm r$ at depth $|v|=\ell$ represents a candidate state for step $n+\ell$. Conditional on its parent, this candidate is drawn from the proposal transition at step $n+\ell-1$. Writing $\mathbf y_A := (y_v)_{v\in A}$ for the collection of states indexed by $A\subseteq\cV$, the class of draft models we consider is of the form
\begin{equation}
  p_{\cT}\big(\mathbf y_{\cV\setminus\{\mathrm r\}}\mid y_{\mathrm r}\big)
  =
  \prod_{v\in\cV\setminus\{\mathrm r\}}
    \cN\Big(
      y_v;\;
      m^p_{n+|v|-1}\big(y_{\pa(v)}\big),\;
      \sigma^2_{n+|v|-1}\Id
    \Big).
  \label{eq:tree-proposal}
\end{equation}
In particular, children of the same parent are conditionally i.i.d. proposals for the same target transition. For a full tree with branching factor $K$ and lookahead depth $L$, every non-leaf node has $K$ children, and depth $\ell$ contains $K^\ell$ candidate states. Hence drafting produces $|\cV|-1 = \sum_{\ell=1}^{L} K^\ell$ candidate states.

(ii) \underline{Parallel evaluation}:
Verification requires target-drift evaluations only at the internal nodes
$\cI(\cT):=\{u:\cC(u)\neq\emptyset\}$, and these evaluations can be batched into a single target-network call. One may therefore define the parallel compute budget using only the internal nodes, $B=|\cI|$. This reduces the verification budget, but introduces a caveat: if all drafted nodes are rejected and the algorithm returns a residual sample, the target drift at that new root state is unavailable and must be evaluated before the next draft tree can be generated. An alternative is to verify the entire tree, including the leaves, corresponding to the larger budget $B=|\cV|$. Even though this spends additional parallel compute, it provides target-drift evaluations at the leaves in advance. Consequently, if the current tree is fully rejected, the score evaluated at a leaf can be reused as the delayed proposal drift for the next tree, whose root is the residual sample. Full-tree verification also makes same-timestep sibling evaluations available for reuse after a leaf rejection. Since the relative benefit of these two strategies is problem dependent, we follow~\citet{bullo2026drafttrees} and adopt the full-tree budget $B=|\cV|$ in the main experiments, deferring the internal-node alternative $B=|\cI|$ to Appendix~\ref{app:internal-node-budget}. Speculative sampling therefore trades a larger parallel batch for fewer \emph{sequential} target-model evaluations.

(iii) \underline{Verification}:
Verification proceeds from the root down the tree, one level at a time,
along the accepted branch. We denote the nodewise verification rule by $\textsc{Verify}$. Consider a visited node with proposal
$P=\cN(\mu_p,\sigma^2\Id)$, target
$Q=\cN(\mu_q,\sigma^2\Id)$, and drafted children
$\widehat Y_{1:K}\stackrel{\mathrm{i.i.d.}}{\sim}P$.
Given $(\mu_p,\mu_q,\sigma,\widehat Y_{1:K})$,
$\textsc{Verify}$ returns a state $Y$ satisfying $Y\sim Q$. If a child is accepted, then $Y=\widehat Y_j$ for some $j\in\{1,\ldots,K\}$, and verification continues from that child. Otherwise, $\textsc{Verify}$ returns a correction drawn from the appropriate residual distribution, which is committed as the next state and terminates the speculative round. 
The performance of $\textsc{Verify}$ is determined by the probability of accepting one of the drafted children. Formally, this gives the
\emph{ maximal list-coupling} objective
\begin{equation}\label{eq:max_list_coupling}
  \max_{\pi\in\Pi(P^{\otimes K},Q)}
  \Pr_{\pi}\!\left(
    Y\in\{\widehat Y_1,\ldots,\widehat Y_K\}
  \right).
\end{equation}
D-GRS instantiates $\textsc{Verify}$ by greedily testing candidates in
drafting order against the remaining target mass
\citep{bullo2026drafttrees}. For $K=1$, the problem reduces to classical
maximal coupling, as used by \citet{debortoli2025speculative}.
For $P$ and $Q$ being isotropic Gaussians with the same covariance, the verification problem reduces to one dimension. Define
\begin{equation}
  \Delta:=(\mu_q-\mu_p)/\sigma,
  \qquad
  \delta:=\|\Delta\|,
  \qquad
  e:=\Delta/\delta,
  \quad \delta>0,
\end{equation}
and, for every node, decompose each standardized child \((\widehat Y_k-\mu_p)/\sigma\) into its projection onto \(e\) and its orthogonal component as
\begin{equation}
  \widehat S_k
  :=e^\top(\widehat Y_k-\mu_p)/\sigma,
  \qquad
  Z_{\perp,k}
  :=(\widehat Y_k-\mu_p)/\sigma- e\widehat S_k, \qquad k=1,\ldots,K.
  \label{eq:scalar_reduction}
\end{equation}
Under \(P\), \(\widehat S_{1:K}\overset{\mathrm{i.i.d.}}{\sim}\mathcal N(0,1)\), whereas the target projection has density \(q_\delta(s):=\phi(s-\delta)\).
In both cases, the orthogonal component follows
$\cN(0,\Id-ee^\top)$ independently of the projection. Hence, $\textsc{Verify}$ needs to generate a scalar \(S\) with target density $\mathcal N(\delta,1)$, after which the output is reconstructed as
$Y=\mu_p+\sigma(Se+Z_\perp)$. Henceforth, we take this reduction as implicit and regard \(\textsc{Verify}\) as operating directly on the scalar projections \(\widehat S_{1:K}\). Our contribution, presented in Section~\ref{sec:rass}, is a new instantiation of $\textsc{Verify}$.

\section{Rank-Aware Speculative Sampling}
\label{sec:rass}
We now introduce a rank-aware instantiation of $\textsc{Verify}$. RASS ranks the drafted children projections 
and selects one according to a rank distribution. Then, it maximally couples the resulting selected-child distribution with the target scalar distribution.

\paragraph{Rank-based selection.}
Define $\widehat U_k:=\Phi(\widehat S_k)$ for $k=1,\ldots,K$.
Since $\widehat S_k\overset{\mathrm{i.i.d.}}{\sim}\cN(0,1)$,
the probability integral transform gives
$\widehat U_k\overset{\mathrm{i.i.d.}}{\sim}\cU(0,1)$.
Let $\widehat U_{(1)}<\cdots<\widehat U_{(K)}$ denote the corresponding order statistics, and let $j(r)$ be the index of the child with rank $r$, such that $\widehat U_{(r)}=\widehat U_{j(r)}$.
Since $\Phi$ is strictly increasing, this ordering is also the ordering
of the projections $\widehat S_{1:K}$. Thus, $j(r)$ identifies the
child with rank $r$ along $e$, and rank $K$ corresponds
to the child with the largest projection toward the target mean.
Let $\cW_K$ denote the $K$-dimensional probability simplex. Given rank weights
$\omega\in\cW_K$ (which depend only on $(\delta,K)$), RASS samples
\begin{equation}
    R\sim\mathrm{Categorical}(\omega),
    \qquad
    \Jsel:=j(R),
    \label{eq:rank-draw}
\end{equation}
with the selected proposal being $\widehat Y_{\Jsel}$. While $R$ is independent of the drafted children, the selected index $\Jsel$ depends on their projections through $j$. The selected scalar projection is then $\widehat S_{\Jsel}=\Phi^{-1}(\widehat U_{(R)})$, with $\Phi^{-1}$ denoting the inverse CDF of the standard normal distribution.
The next proposition characterizes the distribution induced by this 
rank-based selection.

\begin{restatable}[Law of selected child $\widehat S_J$]{proposition}{childlawprop}
\label{prop:selected-law}
For any $\omega\in\cW_K$, the selected projection
$\widehat S_{\Jsel}$ has density
\begin{equation}
    p_\omega(s)
    =\phi(s)\,b_\omega\!\left(\Phi(s)\right),
    \qquad\textrm{with}\quad
    b_\omega(a)
    :=
    \sum_{r=1}^{K}\omega_r f_r(a), 
    \label{eq:selected-density}
\end{equation}
where $f_r$ denotes the density of a
$\mathrm{Beta}(r,K+1-r)$ random variable.
Furthermore, conditional on $\widehat S_{1:K}$ and $R$, the selected orthogonal component
satisfies $Z_{\perp,\Jsel}\sim\cN(0,\Id-ee^\top)$.
\end{restatable}
Two limiting choices of $\omega$ illustrate the effect of rank selection. First, uniform weights,
$\omega_r=1/K$, satisfy $b_\omega(a)=1$. Consequently, $p_\omega=\phi$, so uniform rank selection recovers the original proposal distribution.
Second, the top-rank choice $\omega_K=1$ selects the child with the largest projection \(\widehat S_{(K)}=\max(\widehat S_1,\ldots,\widehat S_K)\), whose density is $ p_\omega(s) = K\phi(s)\Phi(s)^{K-1}$. This choice biases the selected proposal toward larger projections along the target direction. For small \(\delta\), this bias can overshoot the target, reducing the overlap with the target relative to the original proposal \(\phi\). As delineated in what follows, we adopt a weight selection scheme that \emph{adaptively} interpolates between uniform and top-rank selection.

\paragraph{Verification by maximal coupling.}
Rank-based selection converts the \(K\) drafted children into a single rank-selected candidate $\widehat Y_J$. For fixed rank weights \(\omega\), its scalar projection \(\widehat S_J\) follows the induced density \(p_\omega\), which in general differs from the target density \(q_\delta\). Following the verification procedure of Section~\ref{sec:background}, RASS maximally couples \(p_\omega\) with the target \(q_\delta\), maximizing the probability of accepting the selected child. Equivalently, RASS chooses the rank mixture so that the induced selected-child measure is as close as possible to the target, thereby maximizing the maximal-coupling acceptance probability.
Specifically, let $Y \sim Q$ denote the target-distributed output of the verifier. For a given selected child $\widehat Y_{\Jsel}=y$ with scalar projection $\widehat S_J = s$, the acceptance probability is
\begin{equation}
  a_\omega(s) := \Pr\bigl(Y = \widehat Y_J \mid \widehat S_J = s\bigr)
  = \min\bigl\{1,\, q_\delta(s)/p_\omega(s)\bigr\},
  \qquad s \in \mathbb{R}.
  \label{eq:acceptance-rule}
\end{equation}
%
Therefore, the acceptance rate
is
\begin{equation}
\alpha_K(\delta;\omega):=\Pr(\widehat Y_J=Y)=\!\!\int_{\mathbb{R}}a_\omega(s)
p_\omega(s)\dd s =\!\!\int_{\mathbb R}\min\{{q_\delta(s),p_\omega(s)}\}\dd s =1-\mathrm{TV}(q_\delta,p_\omega).
\end{equation}
On rejection, the remaining mass is \(q_\delta(s)-\min\{q_\delta(s),p_\omega(s)\}
=(q_\delta(s)-p_\omega(s))_+\). Hence, to guarantee exactness, RASS draws $S\sim (q_\delta(s)-p_\omega(s))_+/\mathrm{TV}(q_\delta,p_\omega)$, and reconstructs \(Y=\mu_p+\sigma(Se+Z_{\perp,1})\), with the orthogonal component of the first child in drafting order.
\begin{algorithm}[t]
\caption{$\textsc{Verify} \gets$ RASS}
\label{alg:rass-verify}
\begin{algorithmic}[1]
\Require $\mu_p,\mu_q,\sigma>0$; children
$\widehat Y_{1:K}\overset{\mathrm{i.i.d.}}{\sim}
\mathcal N(\mu_p,\sigma^2I_d)$
\Ensure $(Y,\texttt{accepted},j)$
\State Compute $\widehat S_{1:K}$ and $Z_{\perp,1}$ as in~\eqref{eq:scalar_reduction}, $\delta\gets\|\mu_q-\mu_p\|/\sigma$ 
\If{$\delta=0$}
    \State \Return $(\widehat Y_1,\texttt{true},1)$
\EndIf
\State Compute $\omega^\star$ by solving \eqref{eq:rank-optimization}  \Comment{(See Appendix~\ref{app:rankweights-impl}) }
\State $R\sim\textrm{Categorical}(\omega^\star)$,\quad
       $J\gets j(R)$
\State $A\sim\textrm{Bernoulli}\!\left(
       \min\{1, q_\delta(\widehat S_{J})/p_\omega(\widehat S_{J})\}\right)$
\If{$A=1$}
    \State \Return $(\widehat Y_J,\texttt{true},J)$
\EndIf
\State Draw $S\sim (q_\delta(s)-p_\omega(s))_+/\mathrm{TV}(q_\delta,p_\omega)$ \Comment{(See Appendix~\ref{app:residual})}
\State \Return
$\bigl(\mu_p+\sigma(Se+Z_{\perp,1}),
       \texttt{false},K+1\bigr)$
\end{algorithmic}
\end{algorithm}
The rank weights \(\omega\) determine the selected-proposal density \(p_\omega\) and hence the acceptance probability \(\alpha_K(\delta;\omega)\) in (11). For fixed \((\delta,K)\), we choose
\begin{equation}
  \omega^\star(\delta,K)
  \in
  \arg\max_{\omega\in\cW_K}
  \alpha_K(\delta;\omega).
  \label{eq:rank-optimization}
\end{equation} 
The RASS verification rule is reported in Algorithm~\ref{alg:rass-verify}, and implementation details for sampling the residual distribution, and the computational complexity for solving \eqref{eq:rank-optimization} are deferred to Appendix~\ref{app:residual} and \ref{app:rankweights-impl}, respectively.

\section{Theoretical Analysis}
\label{sec:analysis}
In this section, we first establish that RASS preserves the target law for every feasible rank distribution. We then characterize the maximum acceptance probability attainable from \(K\) i.i.d. drafts, providing a common benchmark for RASS and D-GRS.
Proofs are deferred to Appendix~\ref{ap:th_analysis}.

\begin{restatable}[Exactness]{theorem}{exactnesstheorem}
\label{thm:exact}
For every $K\ge1$ and every $\delta\ge0$, the RASS verifier is exact for any feasible rank weights $\omega\in\cW_K$ used by the verifier. 
\end{restatable}

Exactness holds for any choice of rank weights, allowing us to optimize acceptance without changing the target law. To assess how effectively RASS uses the draft list, we first determine the best acceptance probability achievable by any exact verifier under the same drafting distribution.
\begin{restatable}[Optimal list acceptance]{proposition}{optachievablelistcouplingprop}
    \label{prop:optimal-list-acceptance}
    In the Gaussian setting of Section~\ref{sec:background}, the maximum probability
    of accepting a draft while preserving the target distribution is
    \[
        \alpha_K^{\mathrm{up}}(\delta)
        =
        1-\sup_{s\in\mathbb R}
        \left\{\Phi(s)^K-\Phi(s-\delta)\right\},
    \]
    where $\Phi$ denotes the standard Gaussian CDF. 
\end{restatable} 
Note that Proposition~\ref{prop:optimal-list-acceptance} characterizes the maximal acceptance probability achievable by any list-coupling scheme when the $K$ proposals are sampled \emph{independently}. This yields a tighter bound than~\citet[Theorem~2.2]{phan2026multi}, which allows arbitrary coupling among the draft samples---an assumption that is not available in our setting. Proposition~\ref{prop:optimal-list-acceptance} reflects a simple limitation: a verifier can return a draft above a threshold $s$ only if the list contains one. The target assigns probability $1-\Phi(s-\delta)$ to this region, whereas the list contains such a draft with probability $1-\Phi(s)^K$. Any shortfall must therefore be supplied through rejection and correction. Maximizing this shortfall over $s$ gives a lower bound on rejection.
Thus, $\alpha_K^{\mathrm{up}}(\delta)$ provides a common benchmark for
RASS and D-GRS, measuring how much room remains to improve
acceptance with the same draft budget.
Figure~\ref{fig:iid-comparison} compares the acceptance probabilities of RASS and D-GRS with the optimal i.i.d. list-coupling bound $\alpha_K^{\mathrm{up}}(\delta)$, numerically demonstrating higher acceptance for RASS across the tested positive mean displacements, with the largest gains at intermediate values of $\delta$.

\begin{figure}[t]
  \centering
\begingroup
\providecommand{\FigureThreeDataFile}{figures/analysis/data/acceptance-data.csv}
\pgfplotstableread[col sep=comma]{\FigureThreeDataFile}\FigureThreeData

\definecolor{figureThreeBlue}{HTML}{1F77B4}
\definecolor{figureThreeOrange}{HTML}{FF7F0E}
\pgfplotsset{
  figure three rass/.style={color=figureThreeOrange,line width=0.95pt},
  figure three dgrs/.style={color=figureThreeBlue,dashed,line width=0.85pt},
  figure three upper/.style={color=black,line width=0.8pt}
}

\newcommand{\FigureThreePanel}[1]{%
  \addplot[figure three upper,forget plot]
    table[x=delta,y=upper_bound_k#1] {\FigureThreeData};
  \addplot[figure three rass,forget plot]
    table[x=delta,y=rass_k#1] {\FigureThreeData};
  \addplot[figure three dgrs,forget plot]
    table[x=delta,y=dgrs_k#1] {\FigureThreeData};
}

\begin{tikzpicture}
  \begin{groupplot}[
    group style={
      group size=3 by 1,
      horizontal sep=0.035\linewidth,
      ylabels at=edge left,
      yticklabels at=edge left
    },
    width=0.28\linewidth,
    height=0.14\linewidth,
    scale only axis,
    xmin=0,xmax=4,
    ymin=0,ymax=1.025,
    xtick={0,1,2,3,4},
    ytick={0,0.2,0.4,0.6,0.8,1},
    xlabel={Mean displacement $\delta$},
    ylabel={Acceptance probability},
    axis lines=left,
    axis line style={black!65,-},
    tick align=outside,
    tick style={black!65},
    tick label style={font=\scriptsize},
    label style={font=\small},
    title style={font=\small, yshift=-7pt},
    yticklabel style={
      /pgf/number format/fixed,
      /pgf/number format/precision=1,
      /pgf/number format/fixed zerofill
    },
    scaled ticks=false,
    grid=major,
    major grid style={black!10,line width=0.25pt},
    no markers
  ]
    \nextgroupplot[
      title={$K=2$},
      legend to name=figureThreeAcceptanceLegend,
      legend columns=3,
      legend style={
        draw=none,
        font=\small,
        cells={anchor=west},
        /tikz/every even column/.append style={column sep=0.6em}
      }
    ]
    \FigureThreePanel{2}
    \addlegendimage{figure three rass}
    \addlegendentry{RASS}
    \addlegendimage{figure three dgrs}
    \addlegendentry{D-GRS}
    \addlegendimage{figure three upper}
    \addlegendentry{$\alpha_K^{\mathrm{up}}(\delta)$}

    \nextgroupplot[title={$K=4$}]
    \FigureThreePanel{4}

    \nextgroupplot[title={$K=8$}]
    \FigureThreePanel{8}
  \end{groupplot}
  \node[anchor=north] at ([yshift=4mm]current bounding box.north)
    {\pgfplotslegendfromname{figureThreeAcceptanceLegend}};
\end{tikzpicture}
\endgroup
  \par\vspace{-1.em}
  \def\RankSimplexDir{figures/analysis/data/}
\begingroup
\providecommand{\RankSimplexDir}{}
\pgfplotstableread[col sep=comma]{\RankSimplexDir rank-simplex-points.csv}\RankSimplexPoints

\definecolor{rankSimplexBlue}{HTML}{1F77B4}
\definecolor{rankSimplexOrange}{HTML}{FF7F0E}
\definecolor{rankSimplexPurple}{HTML}{9467BD}
\pgfplotsset{
  rank simplex dgrs/.style={color=rankSimplexBlue,line width=0.9pt},
  rank simplex path/.style={color=rankSimplexOrange,line width=1.2pt,
    line cap=round,dash pattern=on 0pt off 2.2pt},
  rank simplex unif/.style={only marks,mark=*,mark size=2.3pt,
    mark options={fill=white,draw=black,line width=0.4pt}},
  rank simplex max/.style={only marks,mark=triangle*,mark size=2.8pt,
    mark options={fill=rankSimplexPurple,draw=black,line width=0.4pt}},
  rank simplex star/.style={only marks,mark=rank simplex star,mark size=4pt,
    mark options={fill=rankSimplexOrange,draw=black,line width=0.4pt}},
}
\tikzset{
  rank simplex value/.style={anchor=north,inner sep=0pt,font=\scriptsize},
  rank simplex vertex/.style={inner sep=1pt,font=\scriptsize},
}

\newcommand{\RankSimplexPanel}[1]{%
  \addplot[patch,shader=interp,forget plot]
    table[col sep=comma,x=x,y=y,point meta=\thisrow{c#1}]
      {\RankSimplexDir rank-simplex-mesh.csv};
  \addplot[rank simplex dgrs,unbounded coords=jump,forget plot]
    table[col sep=comma,x=x,y=y] {\RankSimplexDir rank-simplex-dgrs-#1.csv};
  \draw[black,line width=0.4pt,line join=round]
    (axis cs:0,0) -- (axis cs:1,0) -- (axis cs:0.5,0.866025) -- cycle;
  \addplot[rank simplex path,forget plot]
    table[col sep=comma,x=x,y=y] {\RankSimplexDir rank-simplex-path.csv};
  \addplot[rank simplex unif,forget plot] coordinates {(0.5,0.288675)};
  \addplot[rank simplex max,forget plot] coordinates {(0.5,0.866025)};
  \foreach \col in {delta,star_x,star_y,lab_unif,lab_max,lab_dgrs,lab_rass}{%
    \pgfplotstablegetelem{#1}{\col}\of\RankSimplexPoints
    \global\expandafter\let\csname RankSimplex\col#1\endcsname\pgfplotsretval
  }%
  \addplot[rank simplex star,forget plot] coordinates
    {(\csname RankSimplexstar_x#1\endcsname,\csname RankSimplexstar_y#1\endcsname)};
  \node[anchor=south,font=\small] at (axis cs:0.5,0.91)
    {$\delta=\csname RankSimplexdelta#1\endcsname$};
  \node[rank simplex value] at (axis cs:0.5,-0.13) {%
    \begin{tabular}{@{}l@{\hspace{0.9em}}l@{}}
      $\omega_{\mathrm{unif}}$: \csname RankSimplexlab_unif#1\endcsname &
      $\omega_{\mathrm{max}}$: \csname RankSimplexlab_max#1\endcsname \\
      D-GRS: \csname RankSimplexlab_dgrs#1\endcsname &
      RASS: \csname RankSimplexlab_rass#1\endcsname
    \end{tabular}};
}

\begin{tikzpicture}
  \begin{groupplot}[
    group style={group size=4 by 1,horizontal sep=0.015\linewidth},
    scale only axis,width=0.2\linewidth,height=0.1608\linewidth,
    xmin=-0.02,xmax=1.02,ymin=-0.04,ymax=0.9,
    hide axis,clip=false,
    colormap={rank simplex greys}{rgb255=(255,255,255) rgb255=(240,240,240)
      rgb255=(217,217,217) rgb255=(189,189,189) rgb255=(150,150,150)
      rgb255=(115,115,115) rgb255=(82,82,82) rgb255=(37,37,37)
      rgb255=(0,0,0)},
    point meta min=0.5,point meta max=1,
  ]
    \nextgroupplot[
      legend to name=rankSimplexLegend,
      legend columns=5,
      legend style={draw=none,font=\scriptsize,cells={anchor=west},
        /tikz/every even column/.append style={column sep=0.9em}}
    ]
    \RankSimplexPanel{0}
    \node[rank simplex vertex,anchor=north] at (axis cs:0,-0.01) {$r=1$};
    \node[rank simplex vertex,anchor=north] at (axis cs:1,-0.01) {$r=2$};
    \node[rank simplex vertex,anchor=west] at (axis cs:0.55,0.866) {$r=3$};
    \addlegendimage{rank simplex unif}
    \addlegendentry{$\omega_{\mathrm{unif}}$ (RMC)}
    \addlegendimage{rank simplex max}
    \addlegendentry{$\omega_{\mathrm{max}}$}
    \addlegendimage{rank simplex star}
    \addlegendentry{RASS $\omega^\star$}
    \addlegendimage{rank simplex dgrs}
    \addlegendentry{D-GRS acceptance}
    \addlegendimage{rank simplex path}
    \addlegendentry{$\omega^\star(\delta)$ path}

    \nextgroupplot
    \RankSimplexPanel{1}

    \nextgroupplot
    \RankSimplexPanel{2}

    \nextgroupplot[
      colorbar,
      colorbar style={
        hide axis=false,xmin=0,xmax=1,clip=true,ytick pos=right,
        width=0.012\linewidth,
        xshift=0.005\linewidth,
        tick align=outside,
        major tick length=2pt,
        ytick={0.5,0.75,1},
        yticklabel style={font=\scriptsize,
          /pgf/number format/fixed,
          /pgf/number format/precision=2,
          /pgf/number format/fixed zerofill},
        ylabel={$\alpha_3(\delta;\omega)\,/\,\alpha^{\mathrm{up}}_3(\delta)$},
        ylabel style={font=\scriptsize},
      }
    ]
    \RankSimplexPanel{3}
  \end{groupplot}
  \node[anchor=north,inner sep=0pt] at ([xshift=0.0075\linewidth,yshift=-2.4em]group c2r1.south east)
    {\pgfplotslegendfromname{rankSimplexLegend}};
\end{tikzpicture}
\endgroup

  \caption{RASS and D-GRS acceptance compared with $\alpha^{\mathrm{up}}_K(\delta)$. \emph{Top:} acceptance versus $\delta$ for $K\in\{2,4,8\}$, with optimized RASS weights from (12). \emph{Bottom:} relative acceptance $\alpha_3(\delta;\omega)/\alpha^{\mathrm{up}}_3(\delta)$ on $\mathcal{W}_3$ at four displacements. The solid blue contour encloses weights outperforming D-GRS; the dotted curve traces $\omega^\star(\delta)$ for $\delta\in(0,2.5]$. Acceptance values are reported below each panel.}
  \label{fig:iid-comparison}
\end{figure}
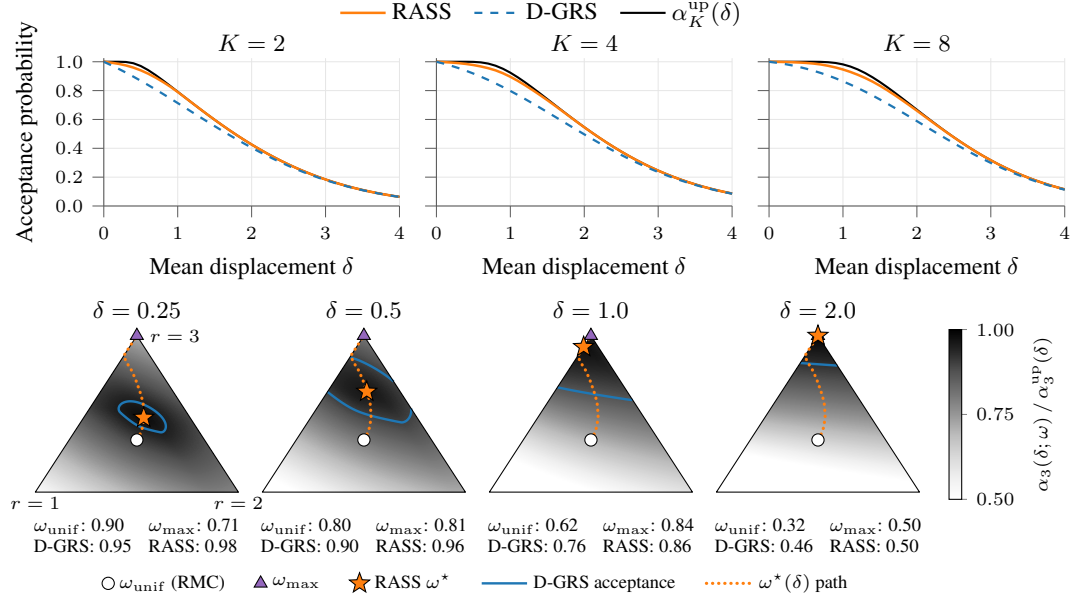

\section{Experiments}
\label{sec:experiments}
We compare RASS with D-GRS~\citep{bullo2026drafttrees} and RMC~\citep{debortoli2025speculative} on a Gaussian-mixture target, class-conditional CIFAR-10 and unconditional FFHQ generation in pixel space, and latent text-to-image generation with Stable Diffusion 3.5 (SD3.5) on COCO 2014 captions. Within each experiment, all methods use the same target discretization and delayed-velocity proposal, reusing target evaluations between speculative rounds. All three methods preserve the discretized target law, with exactness established for RASS in Theorem~\ref{thm:exact}. Consistently with this guarantee, Table~\ref{tab:quality} reports comparable FID and CLIP scores across methods.
RASS and D-GRS use identical tree configurations, while RMC uses a chain matched to the same budget and truncated at the sampling horizon. Following Section~\ref{sec:background}, parallel target evaluation includes the leaves, with budget $B=|\mathcal V|$.
We report mean per-trajectory target-call speedup over $100$ trajectories per configuration for the Gaussian-mixture and SD3.5 experiments and $500$ for CIFAR-10 and FFHQ.
Shaded bands indicate $95\%$ confidence intervals.
Implementation details and results with verification restricted to internal nodes are reported in Appendix~\ref{app:exp-details}.



\paragraph{Gaussian-mixture target.}
The target distribution is an equally weighted mixture of five isotropic Gaussians in $\R^{512}$. We sample uniformly the component means from $[-2,2]^{512}$ and the standard deviations from $[0.1,0.25]$. We use a 30-step schedule, churn value $\varepsilon=0.3$, and the full grid $K,L\in\{2,\dots,7\}$. Figure~\ref{fig:intro-frontier} shows that RASS achieves higher mean speedup than D-GRS at every evaluated budget, reaching a $6.4\%$ higher maximum. Within the tested grid, RASS exceeds D-GRS's highest observed mean speedup using only $28\%$ of its budget. In Appendix~\ref{app:internal-node-budget}, we report results for $\varepsilon\in\{0.06, 0.1,0.6\}$. 

\paragraph{Pixel-space diffusion.}

\begin{figure}[t]
  \centering
  \input{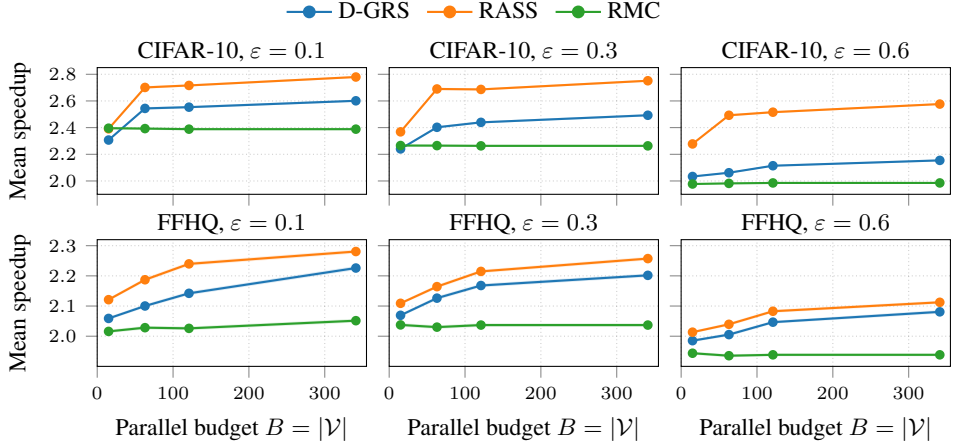}
\pgfplotstableread[col sep=comma]{figures/pixel-space/data/pixel-space-speedup-results-full-tree.csv}\PixelData
\begin{tikzpicture}
  \begin{groupplot}[
    group style={
      group size=3 by 2,
      horizontal sep=0.3cm,vertical sep=0.6cm,
      xlabels at=edge bottom,xticklabels at=edge bottom,
      ylabels at=edge left,yticklabels at=edge left
    },
    width=0.255\linewidth,height=0.12\linewidth,
    scale only axis,
    xmin=0,xmax=355,xtick={0,100,200,300},
    xlabel={Parallel budget $B=|\mathcal{V}|$},
    ylabel={Mean speedup},
    axis lines=box,tick pos=left,tick align=outside,
    tick label style={font=\scriptsize},
    label style={font=\small},title style={font=\small,inner sep=0pt,yshift=2pt},
    every axis title shift=0pt,
    scaled ticks=false,
    yticklabel style={/pgf/number format/fixed,
      /pgf/number format/precision=1,/pgf/number format/fixed zerofill},
    unbounded coords=discard,filter discard warning=false,
  ]
    \nextgroupplot[
        title={CIFAR-10, $\varepsilon=0.1$},
        ymin=1.9,
        ymax=2.85,
        ytick={2.0,2.2,2.4,2.6,2.8},
        legend to name=pixelLegend, 
        legend columns=3,
        legend style={draw=none,font=\small,/tikz/every even column/.append style={column sep=0.4em}},
        grid=major,
        major grid style={densely dotted,draw=black,opacity=0.2}
    ]
    \PixelPanel{cifar10}{0.1}
    \addlegendimage{pixelDGRS,thick,mark=*}\addlegendentry{D-GRS}
    \addlegendimage{pixelRASS,thick,mark=*}\addlegendentry{RASS}
    \addlegendimage{pixelRMC,thick,mark=*}\addlegendentry{RMC}

    \nextgroupplot[
        title={CIFAR-10, $\varepsilon=0.3$},
        ymin=1.9,
        ymax=2.85,
        ytick={2.0,2.2,2.4,2.6,2.8},
        grid=major,
        major grid style={densely dotted,draw=black,opacity=0.2}
    ]
    \PixelPanel{cifar10}{0.3}

    \nextgroupplot[
        title={CIFAR-10, $\varepsilon=0.6$},
        ymin=1.9,
        ymax=2.85,
        ytick={2.0,2.2,2.4,2.6,2.8},
        grid=major,
        major grid style={densely dotted,draw=black,opacity=0.2}
    ]
    \PixelPanel{cifar10}{0.6}

    \nextgroupplot[
        title={FFHQ, $\varepsilon=0.1$},
        ymin=1.9,
        ymax=2.32,
        ytick={2.0,2.1,2.2,2.3},
        grid=major,
        major grid style={densely dotted,draw=black,opacity=0.2}
    ]
    \PixelPanel{ffhq}{0.1}

    \nextgroupplot[
        title={FFHQ, $\varepsilon=0.3$},
        ymin=1.9,
        ymax=2.32,
        ytick={2.0,2.1,2.2,2.3},
        grid=major,
        major grid style={densely dotted,draw=black,opacity=0.2}
    ]
    \PixelPanel{ffhq}{0.3}

    \nextgroupplot[title={FFHQ, $\varepsilon=0.6$},
        ymin=1.9,
        ymax=2.32,
        ytick={2.0,2.1,2.2,2.3},
        grid=major,
        major grid style={densely dotted,draw=black,opacity=0.2}
    ]
    \PixelPanel{ffhq}{0.6}
  \end{groupplot}
  \node[anchor=south] at ([yshift=0mm]current bounding box.north){\pgfplotslegendfromname{pixelLegend}};

\end{tikzpicture}
  \caption{Mean target-call speedup on CIFAR-10 (top) and FFHQ (bottom) across parallel budgets and churn values. Curves average over $500$ images per configuration, and shading shows $95\%$ confidence intervals.}
  \label{fig:pixel-speedup-full-tree}
\end{figure}

We use the pretrained denoisers of \citet{karras2022elucidating}, with the velocity parameterization described in Appendix~\ref{app:delayed-drift-prefetch}.
CIFAR-10 generates $32\times32$ images conditioned on uniformly sampled class labels, while FFHQ generates unconditional $64\times64$ images. We use a $100$-step schedule and evaluate $(K,L)\in\{(2,3),(2,5),(3,4),(4,4)\}$ at $\varepsilon\in\{0.1,0.3,0.6\}$. These configurations correspond to full-tree budgets $B\in\{15,63,121,341\}$. Figure~\ref{fig:pixel-speedup-full-tree} shows higher mean speedups for RASS than D-GRS in every evaluated configuration. On CIFAR-10, the relative improvement increases with churn, reaching $20.9\%$ at $\varepsilon=0.6$ and $B=63$. 
On FFHQ, whose state dimension is 4$\times$ larger, absolute speedups are lower than CIFAR-10, but RASS consistently improves on D-GRS by $1.4\%$ to $4.6\%$. RASS also uses the parallel budget more efficiently. In fact, on CIFAR-10, its mean speedup at $B=63$ exceeds D-GRS's at $B=341$ for all tested churn values. On FFHQ, the same comparison holds with RASS at $B=121$ for $\varepsilon\le0.3$. 

\paragraph{High-dimensional Latent-space Text-to-image.}
\label{sec:exp-sd3}
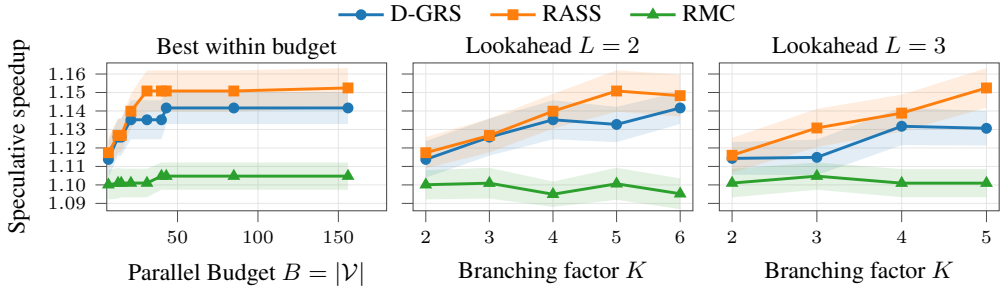
\begin{figure}[t]
    \centering
%
%
%

\begingroup

\def\RuleDGRS{d-grs}
\def\RuleRASS{paws}
\def\RuleRMC{rmc}

\def\SDMaxBudget{156}

\definecolor{DGRSColor}{RGB}{31,119,180}
\definecolor{RASSColor}{RGB}{255,127,14}
\definecolor{RMCColor}{RGB}{44,160,44}

\def\DataFile{figures/latent-space/data/sd3-speedup-full-tree.csv}

\pgfplotsset{
    sd select/.style 2 args={
        x filter/.code={
            \edef\sdRow{\thisrow{rule}}%
            \edef\sdWant{#1}%
            \ifx\sdRow\sdWant\else\def\pgfmathresult{nan}\fi
            \edef\sdLWant{#2}%
            \ifx\sdLWant\empty\else
                \ifnum\thisrow{L}=\sdLWant\relax\else\def\pgfmathresult{nan}\fi
            \fi
            \ifnum\thisrow{budget}>\SDMaxBudget\relax\def\pgfmathresult{nan}\fi
        },
        unbounded coords=discard,
        filter discard warning=false,
    },
}

\newcommand{\sdSeries}[8]{%
    \addplot[name path=#1-#2-#3-hi, draw=none, forget plot, sd select={#1}{#2}]
        table[x=#3, y=#6] {\DataFile};
    \addplot[name path=#1-#2-#3-lo, draw=none, forget plot, sd select={#1}{#2}]
        table[x=#3, y=#5] {\DataFile};
    \addplot[draw=none, fill=#7, fill opacity=0.13, forget plot]
        fill between[of=#1-#2-#3-hi and #1-#2-#3-lo];
    \addplot[color=#7, #8, sd select={#1}{#2}]
        table[x=#3, y=#4] {\DataFile};
}

\begin{tikzpicture}

\begin{groupplot}[
    group style={
        group size=3 by 1,
        horizontal sep=0.35cm,
        yticklabels at=edge left,
        ylabels at=edge left,
    },
    scale only axis,
    width=0.265\linewidth,
    height=0.14\linewidth,
    ymin=1.086, ymax=1.166,
    ytick={1.09,1.10,...,1.161},
    yticklabel style={
        /pgf/number format/.cd, fixed, fixed zerofill, precision=2,
    },
    grid=major,
    grid style={line width=.1pt, draw=gray!20},
    axis line style={black!70},
    tick align=outside,
    tick pos=left,
    tick style={black!70},
    table/col sep=comma,
    title style={font=\small, yshift=-6pt},
    label style={font=\small},
    tick label style={font=\scriptsize},
    every axis plot/.append style={
        line width=1.1pt,
    },
    legend columns=3,
    legend style={
        draw=none,
        font=\small,
        /tikz/every even column/.append style={column sep=0.8em},
    },
]

\nextgroupplot[
    title={Best within budget},
    title style={font=\small, yshift=-2pt},
    xlabel={Parallel Budget $B=|\mathcal{V}|$},
    ylabel={Speculative speedup},
    xmin=6, xmax=180,
    legend to name=sd3legend,
]
\sdSeries{\RuleDGRS}{}{budget}{best_speedup}{best_ci_low}{best_ci_high}{DGRSColor}{mark=*, mark size=1.5pt}
\addlegendentry{D-GRS}
\sdSeries{\RuleRASS}{}{budget}{best_speedup}{best_ci_low}{best_ci_high}{RASSColor}{mark=square*, mark size=1.5pt}
\addlegendentry{RASS}
\sdSeries{\RuleRMC}{}{budget}{best_speedup}{best_ci_low}{best_ci_high}{RMCColor}{mark=triangle*, mark size=1.7pt}
\addlegendentry{RMC}

\nextgroupplot[
    title={Lookahead $L=2$},
    xlabel={Branching factor $K$},
    xmin=1.8, xmax=6.2,
    xtick={2,...,6},
]
\sdSeries{\RuleDGRS}{2}{K}{mean_speedup}{ci_low}{ci_high}{DGRSColor}{mark=*, mark size=1.5pt}
\sdSeries{\RuleRASS}{2}{K}{mean_speedup}{ci_low}{ci_high}{RASSColor}{mark=square*, mark size=1.5pt}
\sdSeries{\RuleRMC}{2}{K}{mean_speedup}{ci_low}{ci_high}{RMCColor}{mark=triangle*, mark size=1.7pt}

\nextgroupplot[
    title={Lookahead $L=3$},
    xlabel={Branching factor $K$},
    xmin=1.85, xmax=5.15,
    xtick={2,...,5},
]
\sdSeries{\RuleDGRS}{3}{K}{mean_speedup}{ci_low}{ci_high}{DGRSColor}{mark=*, mark size=1.5pt}
\sdSeries{\RuleRASS}{3}{K}{mean_speedup}{ci_low}{ci_high}{RASSColor}{mark=square*, mark size=1.5pt}
\sdSeries{\RuleRMC}{3}{K}{mean_speedup}{ci_low}{ci_high}{RMCColor}{mark=triangle*, mark size=1.7pt}

\end{groupplot}

\node[anchor=south, inner sep=0pt] at ([yshift=-3pt]group c2r1.outer north)
    {\pgfplotslegendfromname{sd3legend}};

\end{tikzpicture}

\endgroup
    \caption{Mean target-call speedup for SD3.5-medium. Left: best observed mean speedup within parallel budget $B=|\cV|$. Middle and right: speedup versus branching factor $K$ at lookahead $L=2$ and $L=3$, respectively. Shading indicates $95\%$ confidence intervals over $100$ paired caption--noise inputs.}
    \label{fig:sd3}
\end{figure}



We use Stable Diffusion 3.5 (SD3.5)\footnote{The model is available in Huggingface as \texttt{stabilityai/stable-diffusion-3.5-medium}.} with the linear-interpolant SDE parameterization of \citet{bullo2026drafttrees}. We generate $512\times512$ images in a $16\times64\times64$ latent space ($d=65{,}536$).
We use $N=50$ steps, classifier-free guidance $7.0$, timestep shift $3.0$, and churn $\varepsilon=0.8$, evaluating $K\in\{2,3,4,5,6\}$ and $L\in\{2,3\}$ on $100$ shared COCO~2014 validation captions with paired initial noise. Figure~\ref{fig:sd3} shows higher mean speedups for RASS than both baselines in every evaluated configuration, reaching $1.152\times$ versus $1.142\times$ for D-GRS and $1.105\times$ for RMC. RASS exceeds D-GRS's best observed mean speedup at $B=31$ using $28\%$ less parallel budget. At matched configurations, its relative improvement over D-GRS reaches $1.93\%$, at $(K,L)=(5,3)$, i.e., $B=156$.




\begin{table}[t]
  \centering
  \small
  \setlength{\tabcolsep}{3pt}
  \begin{tabular}{lccccccc}
    \toprule
    & \multicolumn{2}{c}{FID $\downarrow$}
    & \multicolumn{5}{c}{SD3.5 CLIP $\uparrow$} \\
    \cmidrule(lr){2-3}\cmidrule(lr){4-8}
    Sampler & CIFAR-10 & FFHQ
      & $K=2$ & $K=3$ & $K=4$ & $K=5$ & $K=6$ \\
    \midrule
    Target & 1.838 & 2.929
      & \multicolumn{5}{c}{$26.68\pm0.33$} \\
    RMC & 1.871 & 2.898
      & $26.53\pm0.33$ & $26.74\pm0.32$ & $26.76\pm0.35$
      & $26.48\pm0.33$ & $26.80\pm0.33$ \\
    D-GRS & 1.854 & 2.902
      & $26.50\pm0.29$ & $26.69\pm0.37$ & $26.48\pm0.33$
      & $26.71\pm0.33$ & $26.82\pm0.34$ \\
    RASS & 1.844 & 2.936
      & $26.53\pm0.33$ & $26.49\pm0.35$ & $26.31\pm0.33$
      & $26.73\pm0.33$ & $26.82\pm0.33$ \\
    \bottomrule
  \end{tabular}
  \caption{Sample quality.
FID ($\downarrow$) compares the empirical means and covariances of Inception-v3 features from $50$k generated images and real reference images on CIFAR-10 and FFHQ, with $(K,L)=(2,3)$ and $\varepsilon=0.3$. For SD3.5, CLIP score ($\uparrow$) embeddings are obtained from \texttt{openai/clip-vit-large-patch14}.
We report $\textrm{mean}\pm\textrm{standard error}$ over $100$ shared COCO~2014
prompts at $L=2$ and $\varepsilon=0.8$.
The target reports the standard sampler without speculation.}
    \label{tab:quality}
\end{table}

\section{Limitations}
\label{sec:limitations}

Following the convention in the growing literature on speculative diffusion sampling~\citep{debortoli2025speculative,bullo2026drafttrees}, we report speed-up in terms of reductions in sequential target-model evaluations at matched parallel compute budgets. This metric captures the algorithmic benefit of speculation, but translating it into wall-clock acceleration is inherently hardware dependent. Speculative sampling trades fewer sequential target calls for larger parallel batches, additional drafting, and verification overhead; latency gains therefore materialize only when sufficiently capable accelerators can execute these operations efficiently, a requirement which is challenging in an academic setting. Our reported speed-ups should therefore be interpreted primarily as improvements in target-call efficiency rather than guaranteed wall-clock acceleration. Appendix~\ref{app:overhead} separately quantifies the scalar verification overhead, while optimizing the full systems pipeline remains an important direction for future work. Beyond these systems constraints, RASS also relies on specific distributional assumptions. In particular, the drafts must be conditionally independent Gaussian samples sharing the isotropic covariance of the target transition. This restricts its direct applicability to stochastic samplers satisfying this structure and motivates extensions to more general proposal and covariance families.


\section{Conclusion and Future Work}
\label{sec:conclusion}

We introduced RASS, a rank-aware verification rule for speculative diffusion draft trees. By optimizing a distribution over candidate ranks and maximally coupling the selected proposal with the target, RASS improves acceptance while preserving the target law for every feasible choice of rank weights. We also characterized the optimal acceptance probability attainable from independent Gaussian drafts, providing a common benchmark for exact list-based verifiers. Across Gaussian-mixture, pixel-space, and latent text-to image experiments, RASS consistently improves mean target call speedup over D-GRS and RMC at matched parallel budgets.\\
Several directions remain open. First, the gap between optimized rank mixtures and the optimal list-coupling bound motivates richer selection rules that exploit more information from the realized draft list. Developing computationally efficient verifiers that close this gap would further improve the use of parallel proposals. Second, integrating RASS with efficient batching, rank-weight precomputation, and adaptive tree construction may help translate target-call savings into lower generation latency.
Finally, continuous diffusion language models offer a promising application:
RePlaid's stochastic sampler without self-conditioning has Gaussian transitions
compatible with RASS when proposals retain the target covariance \citep{yang2026continuous}. Evaluation and handling self-conditioning
dependencies remain future work.


\section{Disclosure of using AI}
We used generative AI to assist with writing, literature searches, numerical checks, and interpreting results and limitations. The authors reviewed all AI-assisted material and take full responsibility for the manuscript’s accuracy and integrity.

\bibliographystyle{iclr2027_conference}
\bibliography{iclr2027_conference}

@article{bullo2026drafttrees,
  title   = {Accelerating Diffusion Sampling via Speculative Draft Trees},
  author  = {Bullo, Marcello and Liu, Yanxiao and G{\"u}ner, {\"O}yk{\"u} S{\i}la and Mukherjee, Arpan and G{\"u}nd{\"u}z, Deniz},
  journal = {arXiv preprint arXiv:2609.17691},
  year    = {2026},
  url     = {https://arxiv.org/abs/2609.17691}
}

@inproceedings{debortoli2025speculative,
  title     = {Accelerated Diffusion Models via Speculative Sampling},
  author    = {De Bortoli, Valentin and Galashov, Alexandre and Gretton, Arthur and Doucet, Arnaud},
  booktitle = {International Conference on Machine Learning (ICML)},
  pages     = {12590--12631},
  year      = {2025},
  publisher = {PMLR}
}

@article{soen2026block,
  title   = {Accelerating Speculative Diffusions via Block Verification},
  author  = {Soen, Alexander and Husain, Hisham and De Bortoli, Valentin and Doucet, Arnaud},
  journal = {arXiv preprint arXiv:2606.13426},
  year    = {2026}
}

@inproceedings{leviathan2023fast,
  title     = {Fast Inference from Transformers via Speculative Decoding},
  author    = {Leviathan, Yaniv and Kalman, Matan and Matias, Yossi},
  booktitle = {International Conference on Machine Learning (ICML)},
  pages     = {19274--19286},
  year      = {2023},
  publisher = {PMLR}
}

@article{chen2023accelerating,
  title   = {Accelerating Large Language Model Decoding with Speculative Sampling},
  author  = {Chen, Charlie and Borgeaud, Sebastian and Irving, Geoffrey and Lespiau, Jean-Baptiste and Sifre, Laurent and Jumper, John},
  journal = {arXiv preprint arXiv:2302.01318},
  year    = {2023}
}

@inproceedings{sun2023spectr,
  title     = {SpecTr: Fast Speculative Decoding via Optimal Transport},
  author    = {Sun, Ziteng and Suresh, Ananda Theertha and Ro, Jae Hun and Beirami, Ahmad and Jain, Himanshu and Yu, Felix},
  booktitle = {Advances in Neural Information Processing Systems},
  volume    = {36},
  pages     = {30222--30242},
  year      = {2023}
}

@inproceedings{miao2024specinfer,
  title     = {SpecInfer: Accelerating Large Language Model Serving with Tree-based Speculative Inference and Verification},
  author    = {Miao, Xupeng and Oliaro, Gabriele and Zhang, Zhihao and Cheng, Xinhao and Wang, Zeyu and Zhang, Zhengxin and Wong, Rae Ying Yee and Zhu, Alan and Yang, Lijie and Shi, Xiaoxiang and others},
  booktitle = {Proceedings of the 29th ACM International Conference on Architectural Support for Programming Languages and Operating Systems, Volume 3},
  pages     = {932--949},
  year      = {2024}
}

@inproceedings{khisti2025multidraft,
  title     = {Multi-Draft Speculative Sampling: Canonical Decomposition and Theoretical Limits},
  author    = {Khisti, Ashish and Ebrahimi, M Reza and Dbouk, Hassan and Behboodi, Arash and Memisevic, Roland and Louizos, Christos},
  booktitle = {International Conference on Learning Representations (ICLR)},
  year      = {2025}
}

@inproceedings{hu2025optimal,
  title     = {Towards Optimal Multi-Draft Speculative Decoding},
  author    = {Hu, Zhengmian and Zheng, Tong and Viswanathan, Vignesh and Chen, Ziyi and Rossi, Ryan and Wu, Yihan and Manocha, Dinesh and Huang, Heng},
  booktitle = {International Conference on Learning Representations (ICLR)},
  pages     = {3181--3203},
  year      = {2025}
}

@article{lin2026spectrgbv,
  title   = {SpecTr-GBV: Multi-Draft Block Verification Accelerating Speculative Decoding},
  author  = {Lin, Yijun and Sheng, Jinhao and Cai, Qingyue and Zhou, Feng},
  journal = {arXiv preprint arXiv:2604.25925},
  year    = {2026}
}

@inproceedings{rowan2025list,
  title     = {List-Level Distribution Coupling with Applications to Speculative Decoding and Lossy Compression},
  author    = {Rowan, Joseph and Phan, Truong Buu and Khisti, Ashish},
  booktitle = {Advances in Neural Information Processing Systems},
  volume    = {38},
  pages     = {5576--5619},
  year      = {2025}
}

@inproceedings{daliri2025coupling,
  title     = {Coupling without Communication and Drafter-Invariant Speculative Decoding},
  author    = {Daliri, Majid and Musco, Christopher and Suresh, Ananda Theertha},
  booktitle = {2025 IEEE International Symposium on Information Theory (ISIT)},
  pages     = {1--6},
  year      = {2025},
  organization = {IEEE}
}

@article{harsha2010communication,
  title   = {The Communication Complexity of Correlation},
  author  = {Harsha, Prahladh and Jain, Rahul and McAllester, David and Radhakrishnan, Jaikumar},
  journal = {IEEE Transactions on Information Theory},
  volume  = {56},
  number  = {1},
  pages   = {438--449},
  year    = {2010}
}

@inproceedings{flamich2023adaptive,
  title     = {Adaptive Greedy Rejection Sampling},
  author    = {Flamich, Gergely and Theis, Lucas},
  booktitle = {2023 IEEE International Symposium on Information Theory (ISIT)},
  pages     = {454--459},
  year      = {2023},
  organization = {IEEE}
}

@inproceedings{flamich2023greedy,
  title     = {Greedy Poisson Rejection Sampling},
  author    = {Flamich, Gergely},
  booktitle = {Advances in Neural Information Processing Systems},
  volume    = {36},
  pages     = {37089--37127},
  year      = {2023}
}

@article{flamich2026stochastic,
  title   = {Data Compression with Stochastic Codes},
  author  = {Flamich, Gergely and G{\"u}nd{\"u}z, Deniz},
  journal = {IEEE BITS the Information Theory Magazine},
  year    = {2026},
  doi     = {10.1109/MBITS.2026.3702602}
}

@article{li2024channel,
  title   = {Channel Simulation: Theory and Applications to Lossy Compression and Differential Privacy},
  author  = {Li, Cheuk Ting},
  journal = {Foundations and Trends in Communications and Information Theory},
  volume  = {21},
  number  = {6},
  pages   = {847--1106},
  year    = {2024}
}

@article{li2018strong,
  title   = {Strong Functional Representation Lemma and Applications to Coding Theorems},
  author  = {Li, Cheuk Ting and El Gamal, Abbas},
  journal = {IEEE Transactions on Information Theory},
  volume  = {64},
  number  = {11},
  pages   = {6967--6978},
  year    = {2018}
}

@article{hill2026rejection,
  title   = {Rejection Sampling is Optimal for Relative Entropy Coding},
  author  = {Hill, Spencer and Alajaji, Fady and Linder, Tam{\'a}s and Flamich, Gergely},
  journal = {To appear at the 2026 IEEE Information Theory Workshop (ITW)},
  year    = {2026}
}

@inproceedings{sohl2015deep,
  title     = {Deep Unsupervised Learning using Nonequilibrium Thermodynamics},
  author    = {Sohl-Dickstein, Jascha and Weiss, Eric and Maheswaranathan, Niru and Ganguli, Surya},
  booktitle = {International Conference on Machine Learning},
  pages     = {2256--2265},
  year      = {2015},
  publisher = {PMLR}
}

@inproceedings{ho2020denoising,
  title     = {Denoising Diffusion Probabilistic Models},
  author    = {Ho, Jonathan and Jain, Ajay and Abbeel, Pieter},
  booktitle = {Advances in Neural Information Processing Systems},
  volume    = {33},
  pages     = {6840--6851},
  year      = {2020}
}

@inproceedings{song2021score,
  title     = {Score-Based Generative Modeling through Stochastic Differential Equations},
  author    = {Song, Yang and Sohl-Dickstein, Jascha and Kingma, Diederik P and Kumar, Abhishek and Ermon, Stefano and Poole, Ben},
  booktitle = {International Conference on Learning Representations},
  year      = {2021},
  url       = {https://openreview.net/forum?id=PxTIG12RRHS}
}

@inproceedings{song2021ddim,
  title     = {Denoising Diffusion Implicit Models},
  author    = {Song, Jiaming and Meng, Chenlin and Ermon, Stefano},
  booktitle = {International Conference on Learning Representations},
  year      = {2021},
  url       = {https://openreview.net/forum?id=St1giarCHLP}
}

@inproceedings{karras2022elucidating,
  title     = {Elucidating the Design Space of Diffusion-Based Generative Models},
  author    = {Karras, Tero and Aittala, Miika and Aila, Timo and Laine, Samuli},
  booktitle = {Advances in Neural Information Processing Systems},
  volume    = {35},
  pages     = {26565--26577},
  year      = {2022}
}

@inproceedings{lipman2022flow,
  title     = {Flow Matching for Generative Modeling},
  author    = {Lipman, Yaron and Chen, Ricky TQ and Ben-Hamu, Heli and Nickel, Maximilian and Le, Matthew},
  booktitle = {International Conference on Learning Representations},
  year      = {2023}
}

@article{yang2023diffusion,
  title   = {Diffusion Models: A Comprehensive Survey of Methods and Applications},
  author  = {Yang, Ling and Zhang, Zhilong and Song, Yang and Hong, Shenda and Xu, Runsheng and Zhao, Yue and Zhang, Wentao and Cui, Bin and Yang, Ming-Hsuan},
  journal = {ACM Computing Surveys},
  volume  = {56},
  number  = {4},
  pages   = {1--39},
  year    = {2023}
}

@inproceedings{zhang2023fast,
  title     = {Fast Sampling of Diffusion Models with Exponential Integrator},
  author    = {Zhang, Qinsheng and Chen, Yongxin},
  booktitle = {International Conference on Learning Representations},
  year      = {2023}
}

@article{lu2025dpmpp,
  title   = {DPM-Solver++: Fast Solver for Guided Sampling of Diffusion Probabilistic Models},
  author  = {Lu, Cheng and Zhou, Yuhao and Bao, Fan and Chen, Jianfei and Li, Chongxuan and Zhu, Jun},
  journal = {Machine Intelligence Research},
  volume  = {22},
  number  = {4},
  pages   = {730--751},
  year    = {2025}
}

@inproceedings{song2023consistency,
  title     = {Consistency Models},
  author    = {Song, Yang and Dhariwal, Prafulla and Chen, Mark and Sutskever, Ilya},
  booktitle = {International Conference on Machine Learning},
  pages     = {32211--32252},
  year      = {2023},
  publisher = {PMLR}
}

@inproceedings{salimans2022progressive,
  title     = {Progressive Distillation for Fast Sampling of Diffusion Models},
  author    = {Salimans, Tim and Ho, Jonathan},
  booktitle = {International Conference on Learning Representations},
  year      = {2022}
}

@inproceedings{sabour2025align,
  title     = {Align Your Flow: Scaling Continuous-Time Flow Map Distillation},
  author    = {Sabour, Amirmojtaba and Fidler, Sanja and Kreis, Karsten},
  booktitle = {Advances in Neural Information Processing Systems},
  volume    = {38},
  pages     = {146459--146512},
  year      = {2025}
}

@inproceedings{geng2026improved,
  title     = {Improved Mean Flows: On the Challenges of Fastforward Generative Models},
  author    = {Geng, Zhengyang and Lu, Yiyang and Wu, Zongze and Shechtman, Eli and Kolter, J. Zico and He, Kaiming},
  booktitle = {Proceedings of the IEEE/CVF Conference on Computer Vision and Pattern Recognition (CVPR)},
  pages     = {30467--30476},
  year      = {2026}
}

@inproceedings{zhou2026terminal,
  title     = {Terminal Velocity Matching},
  author    = {Zhou, Linqi and Parger, Mathias and Haque, Ayaan and Song, Jiaming},
  booktitle = {The Fourteenth International Conference on Learning Representations},
  year      = {2026}
}

@inproceedings{yin2024improved,
  title     = {Improved Distribution Matching Distillation for Fast Image Synthesis},
  author    = {Yin, Tianwei and Gharbi, Micha{\"e}l and Park, Taesung and Zhang, Richard and Shechtman, Eli and Durand, Fr{\'e}do and Freeman, William T},
  booktitle = {Advances in Neural Information Processing Systems},
  pages     = {47455--47487},
  year      = {2024}
}

@inproceedings{sauer2024adversarial,
  title     = {Adversarial Diffusion Distillation},
  author    = {Sauer, Axel and Lorenz, Dominik and Blattmann, Andreas and Rombach, Robin},
  booktitle = {European Conference on Computer Vision},
  pages     = {87--103},
  year      = {2024},
  publisher = {Springer}
}

@misc{dieleman2026flowmaps,
  title  = {Learning the Integral of a Diffusion Model},
  author = {Dieleman, Sander},
  year   = {2026},
  url    = {https://sander.ai/2026/05/06/flow-maps.html}
}

@inproceedings{shih2023parallel,
  title     = {Parallel Sampling of Diffusion Models},
  author    = {Shih, Andy and Belkhale, Suneel and Ermon, Stefano and Sadigh, Dorsa and Anari, Nima},
  booktitle = {Advances in Neural Information Processing Systems},
  volume    = {36},
  pages     = {4263--4276},
  year      = {2023}
}

@inproceedings{tang2024accelerating,
  title     = {Accelerating Parallel Sampling of Diffusion Models},
  author    = {Tang, Zhiwei and Tang, Jiasheng and Luo, Hao and Wang, Fan and Chang, Tsung-Hui},
  booktitle = {International Conference on Machine Learning (ICML)},
  year      = {2024}
}

@inproceedings{chen2024accelerating,
  title     = {Accelerating Diffusion Models with Parallel Sampling: Inference at Sub-Linear Time Complexity},
  author    = {Chen, Haoxuan and Ren, Yinuo and Ying, Lexing and Rotskoff, Grant M},
  booktitle = {Advances in Neural Information Processing Systems},
  volume    = {37},
  pages     = {133661--133709},
  year      = {2024}
}

@inproceedings{zhou2025parallel,
  title     = {Parallel Simulation for Log-Concave Sampling and Score-Based Diffusion Models},
  author    = {Zhou, Huanjian and Sugiyama, Masashi},
  booktitle = {International Conference on Machine Learning (ICML)},
  pages     = {79192--79225},
  year      = {2025},
  publisher = {PMLR}
}

@article{huangfu2018highs,
  title   = {Parallelizing the Dual Revised Simplex Method},
  author  = {Huangfu, Qi and Hall, J. A. Julian},
  journal = {Mathematical Programming Computation},
  volume  = {10},
  number  = {1},
  pages   = {119--142},
  year    = {2018}
}

@article{albergo2023stochastic,
  title   = {Stochastic Interpolants: A Unifying Framework for Flows and Diffusions},
  author  = {Albergo, Michael S. and Boffi, Nicholas M. and Vanden-Eijnden, Eric},
  journal = {arXiv preprint arXiv:2303.08797},
  year    = {2023}
}

@article{phan2026multi,
  title={Multi-Marginal Couplings for Metropolis-Hastings},
  author={Phan, Buu and Flamich, Gergely and Khisti, Ashish and Asoodeh, Shahab},
  journal={arXiv preprint arXiv:2605.12807},
  year={2026}
}

@article{li2021unified,
  title={A unified framework for one-shot achievability via the {P}oisson matching lemma},
  author={Li, Cheuk Ting and Anantharam, Venkat},
  journal={IEEE Transactions on Information Theory},
  volume={67},
  number={5},
  pages={2624--2651},
  year={2021},
  publisher={IEEE}
}

@article{liu2025one,
  title={One-shot coding over general noisy networks},
  author={Liu, Yanxiao and Li, Cheuk Ting},
  journal={IEEE Transactions on Information Theory},
  volume={71},
  number={11},
  pages={8346--8357},
  year={2025},
  publisher={IEEE}
}

@article{zhou2026dual,
  title={Dual-Representation Image Compression at Ultra-Low Bitrates via Explicit Semantics and Implicit Textures},
  author={Zhou, Chuqin and Ling, Xiaoyue and Chen, Yunuo and Dai, Jincheng and Lu, Guo and Zhang, Wenjun},
  journal={arXiv preprint arXiv:2602.05213},
  year={2026}
}

@article{flamich2026scalable,
  title={Scalable Differentially Private Data Compression via Diffusion and Stochastic Codes},
  author={Flamich, Gergely and G{\"u}ner, Oyk{\"u} S{\i}la and Liu, Yanxiao and G{\"u}nd{\"u}z, Deniz},
  journal={arXiv preprint arXiv:2607.03392},
  year={2026}
}

@article{yang2026continuous,
  title={Continuous diffusion scales competitively with discrete diffusion for language},
  author={Yang, Zhihan and Guo, Wei and Zhang, Shuibai and Sahoo, Subham Sekhar and Chen, Yongxin and Vahdat, Arash and Mardani, Morteza and Thickstun, John},
  journal={arXiv preprint arXiv:2605.18530},
  year={2026}
}

\newpage

\appendix

\section{Related Work}
\label{sec:related}

\paragraph{Accelerating diffusion sampling.}
Training-free solvers reduce the number of steps \citep{song2021ddim,zhang2023fast,lu2025dpmpp}; distillation and flow-map methods learn few-step generators \citep{salimans2022progressive,song2023consistency,sabour2025align,geng2026improved,zhou2026terminal,yin2024improved,sauer2024adversarial,dieleman2026flowmaps}; parallel-in-time methods trade compute for latency through iterative refinement \citep{shih2023parallel,tang2024accelerating,chen2024accelerating,zhou2025parallel}. All of these change the sampling distribution or require training. Speculative diffusion sampling \citep{debortoli2025speculative} is exact; \citet{soen2026block} accelerate it with block verification and \citet{bullo2026drafttrees} with draft trees and greedy rejection sampling. RASS is a verification rule for the latter framework and is orthogonal to the proposal used to draft. 

\paragraph{Speculative decoding and multi-draft coupling.}
Speculative decoding for language models \citep{leviathan2023fast,chen2023accelerating} has been extended to tree-structured drafts \citep{miao2024specinfer} and to multi-draft verification, where several candidate tokens are coupled with the target through optimal transport \citep{sun2023spectr}, canonical decompositions with theoretical limits \citep{khisti2025multidraft,hu2025optimal}, block verification \citep{lin2026spectrgbv}, and list-level couplings \citep{rowan2025list,daliri2025coupling}. Those constructions operate on discrete distributions over a vocabulary, where the optimal list coupling is a transport problem of vocabulary size. In the continuous, high-dimensional Gaussian setting of diffusion, the rank-1 structure reduces the problem to a scalar, and RASS's rank family turns it into a $K$-dimensional LP that does not depend on $d$; to our knowledge, rank selection with optimized weights has not been used as a coupling device before.

\paragraph{Relative entropy coding.}
GRS \citep{harsha2010communication,flamich2023adaptive,flamich2023greedy} develops relative entropy coding, also known as channel simulation, schemes that exactly communicate samples from a target probability distribution to a remote party and are based on rejection sampling.
Another widely used relative entropy coding scheme is based on the Poisson functional representation~\citep{li2018strong}, which has recently found a wide range of applications in information theory~\citep{li2021unified, liu2025one}, machine learning~\citep{zhou2026dual, flamich2026scalable}, and differential privacy.
Its optimality~\citep{hill2026rejection} underlies D-GRS; see also the general references by \citet{li2024channel,flamich2026stochastic}.
RASS replaces the sequential residual-measure recursion with a single maximal coupling of a selected law. 

\section{Speculative Sampling Algorithm}
\label{app:algorithm}
Algorithm~\ref{alg:speculative-template} implements the draft-tree sampler described in Section~\ref{sec:background}, following \citet[Algorithm~1]{bullo2026drafttrees}. It accommodates any \textsc{Verify} rule with the interface of Algorithm~\ref{alg:rass-verify}, including RASS, D-GRS, and, when $K=1$, the maximal coupling of \citet{debortoli2025speculative}. We use the notation of Section~\ref{sec:background} and write $\mathcal T|_{L_n}$ for the tree truncated at depth $L_n=\min\{L,N-n\}$, ensuring that each round stays within the sampling horizon. At each node $u$, we enumerate its children in drafting order as $(v_1,\ldots,v_{K_u})$, where $K_u=|\mathcal C(u)|$.

\begin{algorithm}[t]
\caption{Speculative diffusion sampler with a pluggable \textsc{Verify} rule.\\
Adapted from \citet[Algorithm~1]{bullo2026drafttrees}.}
\label{alg:speculative-template}
\begin{algorithmic}[1]
\Require Initial law $q_0$; means $m^p,m^q$; scales $\sigma_{0:N-1}>0$;
         horizon $N$; tree $\mathcal T$; rule \textsc{Verify}
\Ensure Target trajectory $Y_{1:N}$
\State $n\gets0$; draw $Y_0\sim q_0$
\While{$n<N$}
  \State $L_n\gets\min\{L,N-n\}$, $\mathcal T_n\gets\mathcal T|_{L_n}$, $\widehat Y_{\mathrm r}\gets Y_n$
  \Statex \textcolor{black!50}{(i) \textit{\underline{Drafting}: expand each layer in parallel.}}
  \For{$\ell=1,\ldots,L_n$}
    \ForAll{nodes $u$ at depth $\ell-1$ in $\mathcal T_n$, in parallel}
      \State Draw $\{\widehat Y_v:v\in\mathcal C(u)\}
        \overset{\mathrm{i.i.d.}}{\sim}
        \mathcal N(m^p_{n+|u|}(\widehat Y_u),\sigma_{n+|u|}^2 I_d)$
    \EndFor
  \EndFor
  \Statex \textcolor{black!50}{(ii) \textit{\underline{Parallel Evaluation}: one batched target-model call.}}
  \State Compute $m^q_{n+|u|}(\widehat Y_u)$ for all $u\in\mathcal E_n$
         in parallel
         \Comment{$\mathcal E_n=\mathcal V(\mathcal T_n)$ or
         $\mathcal I(\mathcal T_n)$}
  \Statex \textcolor{black!50}{(iii) \textit{\underline{Verification}: walk down the tree until rejection or a leaf.}}
  \State $u\gets\mathrm r$; $N_\alpha\gets0$
  \While{$u\in\mathcal I(\mathcal T_n)$}
    \State $k\gets n+|u|$
    \State $(Y_{k+1},\texttt{accepted},j)\gets
      \textsc{Verify}\bigl(m^p_{k}(\widehat Y_u),m^q_{k}(\widehat Y_u),
      \sigma_{k};\widehat Y_{v_1},\ldots,\widehat Y_{v_{K_u}}\bigr)$
    \State $N_\alpha\gets N_\alpha+1$
    \If{$\neg$ \texttt{accepted}}
      \State \textbf{break}
    \EndIf
    \State $u\gets v_j$
  \EndWhile
  \State $n\gets n+N_\alpha$
\EndWhile
\State \Return $Y_{1:N}$
\end{algorithmic}
\end{algorithm}
The evaluation set $\mathcal E_n$ determines the parallel budget. We evaluate the full tree, $\mathcal E_n=\mathcal V(\mathcal T_n)$, in the main experiments, and report the internal-node alternative, $\mathcal E_n=\mathcal I(\mathcal T_n)$, in Appendix~\ref{app:internal-node-budget}. Verification requires only the target means at internal nodes, while evaluations at leaves provide additional values for the prefetch cache described in Appendix~\ref{app:delayed-drift-prefetch}. Each round commits $1\leq N_\alpha\leq L_n$ transitions, including the correction step when a proposal is rejected. The template also supports unequal branching factors and leaves at different depths.

To preserve the target trajectory law, a plug-in rule must satisfy the exactness requirement of Section~\ref{sec:background} at each node. Its decision may depend only on the current transition parameters, the node's children, and auxiliary randomness. Consequently, conditional on an accepted child, its descendants retain their proposal-tree law, allowing verification to continue recursively. Deterministic steps ($\sigma_n=0$) are handled separately by applying the target transition directly.



\section{Theoretical Analysis}
\label{ap:th_analysis}
\subsection{Proof of Proposition~\ref{prop:selected-law}}
\childlawprop*

\begin{proof}
We first characterize the CDF of the \(r\)-th order statistic.

{\color{black}

Fix a rank $r$ and define
\[
    M_a:=\sum_{k=1}^{K}\mathds{1}\{\widehat U_k\le a\}.
\]
Since $\widehat U_k\sim\mathcal U(0,1)$ independently,
$M_a\sim\operatorname{Binomial}(K,a)$. Moreover,
\[
    \{\widehat U_{(r)}\le a\}=\{M_a\ge r\}.
\]
Therefore,
\[
    \Pr(\widehat U_{(r)}\le a)
    =
    \sum_{j=r}^{K}\binom Kj a^j(1-a)^{K-j},
\]
which is the CDF of a $\operatorname{Beta}(r,K+1-r)$ random variable.
}
that is the CDF of a \(\mathrm{Beta}(r,K+1-r)\) random variable, whose density is
\begin{equation}
    f_r(a)
    :=
    \frac{K!}{(r-1)!\,(K-r)!}
    a^{r-1}(1-a)^{K-r},
    \qquad a\in(0,1).
\end{equation}
Since $R$ is independent of the children,
$\widehat U_{(R)}$ has density
\[
    b_\omega(a)
    =
    \sum_{r=1}^{K}\omega_r f_r(a).
\]
Since $\widehat S_{\Jsel}=\Phi^{-1}(\widehat U_{(R)})$, and equivalently \(\widehat U_{(R)}=\Phi(\widehat S_{J})\), the monotone change-of-variables formula with \(a=\Phi(s)\) yields
\[
    p_\omega(s)
    =
    b_\omega\!\left(\Phi(s)\right)\Phi'(s)
    =
    \phi(s)b_\omega\!\left(\Phi(s)\right).
\]

Finally, the variables $Z_{\perp,1:K}$ are i.i.d.\
$\cN(0,\Id-ee^\top)$ and are independent of the projections
$\widehat S_{1:K}$. Since $\Jsel$ is determined by
$(\widehat S_{1:K},R)$, selecting an index according to $\Jsel$ does
not change the conditional distribution of its orthogonal component.
Hence
$Z_{\perp,\Jsel}\sim\cN(0,\Id-ee^\top)$.
\end{proof}

\subsection{Proof of Theorem \ref{thm:exact}}
\exactnesstheorem*
\begin{proof}
If $\delta=0$, the shared preparation in Algorithm~\ref{alg:rass-verify} returns the first child, and $P=Q$, so the claim is immediate.  Assume $\delta>0$.  Let $S^{\mathrm{out}}$ and $Z^{\mathrm{out}}_\perp$ denote the scalar and orthogonal components of the returned state, so that 
\[
    Y=\mu_p+\sigma\bigl(S^{\mathrm{out}}e+Z^{\mathrm{out}}_\perp\bigr).
\]
On the acceptance branch, $S^{\mathrm{out}}=\widehat S_{\Jsel}$ has density $p_\omega$ by Proposition~\ref{prop:selected-law} and is retained with probability $a_\omega(s)=1\wedge q_\delta(s)/p_\omega(s)$.  Hence the accepted scalar submeasure has density
\[
    p_\omega(s)a_\omega(s)=\min\{q_\delta(s),p_\omega(s)\}.
\]
Moreover, conditional on $(\widehat S_{1:K},R)$, the selected orthogonal component $Z_{\perp,\Jsel}$ has law $\cN(0,\Id-ee^\top)$ and is independent of the acceptance coin.  Thus the accepted joint submeasure factors as the above scalar submeasure times the orthogonal Gaussian law.

On rejection, the rejection probability is $\eta_\omega=1-\alpha_K (\delta;\omega)$.  If $\eta_\omega>0$, the residual scalar has density $r_\omega$, so its scalar contribution is
\[
    \eta_\omega r_\omega(s)
    =\bigl(q_\delta(s)-p_\omega(s)\bigr)_+.
\]
The orthogonal component $Z_{\perp,1}$ is independent of all projections, the rank draw, the acceptance coin, and the independent residual draw, and has law $\cN(0,\Id-ee^\top)$.  If $\eta_\omega=0$, this branch has zero probability.  Therefore the total joint law of $(S^{\mathrm{out}},Z^{\mathrm{out}}_\perp)$ is the product of $\cN(\delta,1)$ and $\cN(0,\Id-ee^\top)$, since
\[
    \min\{q_\delta(s),p_\omega(s)\}
    +\bigl(q_\delta(s)-p_\omega(s)\bigr)_+
    =q_\delta(s).
\]
It follows that
\[
    Y\sim\cN(\mu_p+\sigma\delta e,\sigma^2\Id)
      =\cN(\mu_q,\sigma^2\Id).
\]

At every visited node, the preceding argument shows that the next committed state has exactly the corresponding target transition law.  If a child is accepted, the verifier returns that drafted child itself and its decision uses only the children of the current node. Hence, conditional on the accepted child, its already drafted descendants retain their original proposal-tree conditional law.  Verification may therefore continue one level deeper with the same procedure.  If a proposal is rejected, the exact residual draw is committed and the speculative round terminates.  Induction over committed steps therefore gives the target Markov-chain law.
\end{proof}

\subsection{Proof of Proposition~\ref{prop:optimal-list-acceptance}}
\optachievablelistcouplingprop*
\begin{proof}
We first reduce the problem to one dimension.  Recall that
\[
    P = \mathcal N(\mu_p,\sigma^2 I_d),
    \qquad
    Q = \mathcal N(\mu_q,\sigma^2 I_d),
    \qquad
    \delta = \frac{\|\mu_q-\mu_p\|}{\sigma}.
\]
The case $\delta=0$ is immediate, since then $P=Q$ and one may always
return, say, the first draft.  Hence suppose $\delta>0$, and let
\[
    e = \frac{\mu_q-\mu_p}{\sigma\delta}=\frac{\mu_q-\mu_p}{\|\mu_q-\mu_p\|}.
\]
For i.i.d.\ drafts $X_1,\ldots,X_K\sim P$ and a target variable
$Y\sim Q$, define the scalar coordinates
\[
    \widehat S_i
    :=
    \frac{\langle X_i-\mu_p,e\rangle}{\sigma},
    \qquad
    S
    :=
    \frac{\langle Y-\mu_p,e\rangle}{\sigma}.
\]
Then
\[
    \widehat S_1,\ldots,\widehat S_K
    \stackrel{\mathrm{i.i.d.}}{\sim}\mathcal N(0,1),
    \qquad
    S\sim\mathcal N(\delta,1).
\]
Apply the probability integral transform
\[
    \widehat U_i := \Phi(\widehat S_i),
    \qquad
    V := \Phi(S).
\]
Then $\widehat U_1,\ldots,\widehat U_K$ are i.i.d.\ $\mathcal{U}(0,1)$, while $V$ has
CDF
\begin{align}
    F_\delta(a):&=\Pr(V\le a)=\Pr(\Phi(S)\le a)\\
    &=\Pr(S\le \Phi^{-1}(a))=\Pr\left(\frac{S-\delta}{\sigma}\le \frac{\Phi^{-1}(a)-\delta}{\sigma}\right)\\
    &=\Phi\!\left(\Phi^{-1}(a)-\delta\right),
    \qquad a\in(0,1),
\end{align}
Its density can be obtained directly by a likelihood-ratio argument.
Indeed, for any Borel set $A\subseteq(0,1)$,
\begin{align*}
    \Pr(V\in A)
    &=
    \int_{\{s:\,\Phi(s)\in A\}}
        \phi(s-\delta)\,ds \\
    &=
    \int_{\{s:\,\Phi(s)\in A\}}
        \frac{\phi(s-\delta)}{\phi(s)}
        \phi(s)\,ds \\
    &=
    \int_{\{s:\,\Phi(s)\in A\}}
        \exp\!\left\{
            \delta s-\frac{\delta^2}{2}
        \right\}\phi(s)\,ds.
\end{align*}
Since $\Phi$ pushes the standard Gaussian measure
$\phi(s)\,ds$ forward to the uniform measure on $(0,1)$, this becomes
\[
    \Pr(V\in A)
    =
    \int_A
        \exp\!\left\{
            \delta\Phi^{-1}(a)-\frac{\delta^2}{2}
        \right\}\,da.
\]
Hence $V$ has density
\[
    h_\delta(a)
    =
    \exp\!\left\{
        \delta\Phi^{-1}(a)-\frac{\delta^2}{2}
    \right\}.
\]
For $\delta>0$, this density is strictly increasing because
$\Phi^{-1}$ is strictly increasing.
Let
\[
    E:=\{V\in\{U_1,\ldots,U_K\}\}
\]
be the acceptance event.  For any $a\in[0,1]$, acceptance with
$V>a$ requires at least one draft above $a$.  Hence
\[
    \Pr(E,V>a)
    \le
    \Pr\!\left(\max_{1\le i\le K}\widehat U_i>a\right)
    =
    1-a^K.
\]
It follows that
\begin{align*}
    1-\Pr(E)
    &\ge
    \Pr(E^c,V>a) \\
    &=
    \Pr(V>a)-\Pr(E,V>a) \\
    &\ge
    \bigl(1-F_\delta(a)\bigr)-\bigl(1-a^K\bigr) \\
    &=
    a^K-F_\delta(a).
\end{align*}
Therefore every valid coupling satisfies
\[
    \Pr(E)
    \le
    1-\max_{a\in[0,1]}
        \left\{a^K-F_\delta(a)\right\}.
\]
Substituting $a=\Phi(s)$ gives
\[
    \Pr(E)
    \le
    1-\sup_{s\in\mathbb R}
    \left\{
        \Phi(s)^K-\Phi(s-\delta)
    \right\}.
\]
It remains to prove that this bound is attainable.

More generally, let $\nu$ be any probability measure on $[0,1]$ and
consider $K$ i.i.d.\ uniform drafts.  For a Borel set
$A\subseteq[0,1]$, write $|A|$ for its Lebesgue measure and define
\[
    d(\nu)
    :=
    \sup_{A\subseteq[0,1]\ {\rm Borel}}
    \left\{
        \nu(A)
        -
        \bigl[1-(1-|A|)^K\bigr]
    \right\}.
\]
The term
\[
    1-(1-|A|)^K
\]
is precisely the probability that at least one of the $K$ drafts lies
in $A$.  Thus, for any coupling of the list and a target $V\sim\nu$,
\[
    \nu(A)
    \le
    \Pr(E^c)
    +
    \Pr\!\left(\{U_1,\ldots,U_K\}\cap A\neq\varnothing\right),
\]
and consequently
\[
    \Pr(E^c)\ge d(\nu).
\]
Hence the acceptance probability is at most $1-d(\nu)$.

We now show that this general bound is sharp.  Partition $[0,1]$ into
$n$ intervals $I_{n,1},\ldots,I_{n,n}$ of length $1/n$.  Consider the
finite coupling problem between the vector of draft-bin labels
\[
    (J_1,\ldots,J_K),
    \qquad
    J_i=j \ \Longleftrightarrow\ U_i\in I_{n,j},
\]
and the target-bin label $J$, where $V\sim\nu$.  Connect a draft-label
vector $(j_1,\ldots,j_K)$ to a target bin $j$ exactly when
$j\in\{j_1,\ldots,j_K\}$.  The finite max-flow/min-cut theorem implies
that the largest mass that can be coupled along such admissible edges
is
\[
    1-d_n,
\]
where
\[
    d_n
    =
    \max_{\substack{A\subseteq[0,1]\\
                    A\text{ a union of partition bins}}}
    \left\{
        \nu(A)-\bigl[1-(1-|A|)^K\bigr]
    \right\}.
\]
Since the maximization defining $d_n$ is over a subclass of the Borel
sets,
\[
    d_n\le d(\nu).
\]
Thus there exists a coupling of the bin labels under which, with
probability at least $1-d(\nu)$, the target bin appears among the draft
bins.  Complete any unmatched source and target mass arbitrarily, for
example by coupling the normalized residual marginals independently,
so that the full source and target marginals remain unchanged.

Lift this finite coupling from bin labels to points as follows.
Conditionally on their labels, draw each $U_i$ uniformly inside its
assigned bin, independently across $i$, and draw $V$ according to
$\nu$ conditioned on its assigned target bin.  The resulting coupling
$\pi_n$ has exactly the marginal
$\mathrm{Unif}[0,1]^{\otimes K}$ for the draft list and marginal
$\nu$ for $V$.  Moreover, whenever the target bin is one of the draft
bins,
\[
    \min_{1\le i\le K}|V-U_i|
    \le \frac1n.
\]
Hence
\[
    \pi_n\!\left(
        \min_{1\le i\le K}|V-U_i|
        \le \frac1n
    \right)
    \ge
    1-d(\nu).
\]

Since $[0,1]^{K+1}$ is compact, the sequence $(\pi_n)$ has a weakly
convergent subsequence, say $\pi_{n_j}\Rightarrow\pi$.  The limiting
coupling $\pi$ preserves both marginals.  For $\varepsilon>0$, let
\[
    R_\varepsilon
    :=
    \left\{
        (u_1,\ldots,u_K,v):
        \min_i|v-u_i|\le\varepsilon
    \right\}.
\]
This set is closed.  For all sufficiently large $j$,
$1/n_j\le\varepsilon$, so
\[
    \pi_{n_j}(R_\varepsilon)\ge1-d(\nu).
\]
By the closed-set part of the Portmanteau theorem,
\[
    \pi(R_\varepsilon)
    \ge
    \limsup_{j\to\infty}\pi_{n_j}(R_\varepsilon)
    \ge
    1-d(\nu).
\]
Since
\[
    R_0
    :=
    \{v=u_i\text{ for some }i\}
    =
    \bigcap_{m=1}^\infty R_{1/m},
\]
continuity from above yields
\[
    \pi(R_0)\ge1-d(\nu).
\]
Together with the preceding upper bound, this proves
\[
    \sup_\pi
    \Pr_\pi\{V\in\{U_1,\ldots,U_K\}\}
    =
    1-d(\nu),
\]
and the supremum is attained.

It remains to specialize this characterization to the Gaussian target
law $\nu_\delta$.  Since $h_\delta$ is increasing, among all measurable
sets $A\subseteq[0,1]$ of a fixed Lebesgue measure $m$, the target mass
$\nu_\delta(A)$ is maximized by the upper interval
\[
    T_m=[1-m,1].
\]
Indeed, $A\setminus T_m$ and $T_m\setminus A$ have equal Lebesgue
measure, while every point of $A\setminus T_m$ lies to the left of
every point of $T_m\setminus A$.  Therefore
\[
    \nu_\delta(A)
    =
    \int_A h_\delta(a)\,da
    \le
    \int_{T_m}h_\delta(a)\,da
    =
    1-F_\delta(1-m).
\]
The draft-availability term depends on $A$ only through $m=|A|$.
Consequently,
\begin{align*}
    d(\nu_\delta)
    &=
    \max_{m\in[0,1]}
    \left\{
        1-F_\delta(1-m)
        -
        \bigl[1-(1-m)^K\bigr]
    \right\} \\
    &=
    \max_{a\in[0,1]}
    \left\{
        a^K-F_\delta(a)
    \right\} \\
    &=
    \sup_{s\in\mathbb R}
    \left\{
        \Phi(s)^K-\Phi(s-\delta)
    \right\},
\end{align*}
where in the second line we set $a=1-m$, and in the last line
$a=\Phi(s)$.

Therefore the maximal achievable acceptance probability is
\[
    \alpha_K^{\mathrm{up}}(\delta)
    =
    1-
    \sup_{s\in\mathbb R}
    \left\{
        \Phi(s)^K-\Phi(s-\delta)
    \right\}.
\]
The weak-limit construction above supplies a coupling attaining this
value, completing the proof.
\end{proof}

\subsection{Proof of Theorem \ref{thm:acceptance}}
\begin{restatable}[Acceptance probability]{theorem}{acceptancetheorem}
\label{thm:acceptance}
Fix $\delta>0$ and $\omega\in\cW_K$. 
We take 
\[
    \cA_\omega
    :=
    \{s\in\R:q_\delta(s)>p_\omega(s)\}
    =
    \{s:\rho(s)>b_\omega(\Phi(s))\},
\]
and let $\{(a_j,b_j)\}_j$ be its disjoint connected components.
Define $C_\omega(a) := \sum_{r=1}^K \omega_r I_a(r,K+1-r)$. 
Then the acceptance probability satisfies 
\begin{equation}
    \alpha_K(\delta;\omega)
    =
    1-
    \sum_j
    \bigl[
        H_\omega(b_j)-H_\omega(a_j)
    \bigr]
    \label{eq:acceptance-closed-a}
\end{equation}
where $H_\omega(s) := \Phi(s-\delta) - C_\omega(\Phi(s))$ satisfies
$H_\omega' = q_\delta - p_\omega$, so each term of the sum is the residual mass
$\int_{a_j}^{b_j} (q_\delta - p_\omega)\,ds$ on one component of $A_\omega$. 
Moreover, we have $\alpha_K(\delta;\omega^{\mathrm{unif}})=2\phibar(\delta/2)$.
\end{restatable}

\begin{proof}
Since $p_\omega$ and $q_\delta$ are probability densities,
\begin{align*}
    \alpha_K(\delta;\omega)
    &=
    \int_{\mathbb R}
    \min\{q_\delta(s),p_\omega(s)\}\,\dd s \\
    &=
    1-
    \int_{\mathbb R}
    \bigl(q_\delta(s)-p_\omega(s)\bigr)_+\,\dd s.
\end{align*}
By definition,
\[
    \cA_\omega
    =
    \{s\in\mathbb R:q_\delta(s)>p_\omega(s)\},
\]
and hence
\[
    \alpha_K(\delta;\omega)
    =
    1-
    \int_{\cA_\omega}
    \bigl(q_\delta(s)-p_\omega(s)\bigr)\,\dd s.
\]

Recall that
\[
    p_\omega(s)
    =
    \phi(s)b_\omega(\Phi(s)),
    \qquad
    b_\omega(a)
    =
    \sum_{r=1}^K\omega_r f_r(a),
\]
and
\[
    C_\omega(a)
    =
    \sum_{r=1}^K
    \omega_r I_a(r,K+1-r).
\]
Since $I_a(r,K+1-r)$ is the CDF corresponding to the density
$f_r(a)$,
\[
    \frac{\dd}{\dd a}C_\omega(a)
    =
    b_\omega(a).
\]
Therefore
\[
    \frac{\dd}{\dd s}C_\omega(\Phi(s))
    =
    b_\omega(\Phi(s))\phi(s)
    =
    p_\omega(s).
\]
Moreover,
\[
    \frac{\dd}{\dd s}\Phi(s-\delta)
    =
    \phi(s-\delta)
    =
    q_\delta(s).
\]
Thus, for
\[
    H_\omega(s)
    :=
    \Phi(s-\delta)-C_\omega(\Phi(s)),
\]
we have
\[
    H_\omega'(s)
    =
    q_\delta(s)-p_\omega(s).
\]

The set $\cA_\omega$ is open, since $q_\delta-p_\omega$ is
continuous, and therefore it can be written as a disjoint union of
open intervals,
\[
    \cA_\omega
    =
    \bigcup_j(a_j,b_j).
\]
Here an endpoint may be infinite, in which case $H_\omega(\pm\infty)$
is understood as the corresponding limit.  Consequently,
\begin{align*}
    \int_{\cA_\omega}
    \bigl(q_\delta(s)-p_\omega(s)\bigr)\,\dd s
    &=
    \sum_j
    \int_{a_j}^{b_j}
    \bigl(q_\delta(s)-p_\omega(s)\bigr)\,\dd s \\
    &=
    \sum_j
    \bigl[
        H_\omega(b_j)-H_\omega(a_j)
    \bigr].
\end{align*}
Every term in the sum is nonnegative, since
$q_\delta>p_\omega$ on $(a_j,b_j)$, and the sum is bounded by one.
Hence
\[
    \alpha_K(\delta;\omega)
    =
    1-
    \sum_j
    \bigl[
        H_\omega(b_j)-H_\omega(a_j)
    \bigr],
\]
which proves~\eqref{eq:acceptance-closed-a}.
We now consider the two special choices of rank weights.
For uniform weights,
\[
    \omega^{\mathrm{unif}}
    =
    \left(\frac1K,\ldots,\frac1K\right).
\]
The equally weighted mixture of the $K$ order-statistic densities is
uniform on $(0,1)$, so
\[
    b_{\omega^{\mathrm{unif}}}(a)=1,
    \qquad
    C_{\omega^{\mathrm{unif}}}(a)=a.
\]
Thus
\[
    p_{\omega^{\mathrm{unif}}}(s)
    =
    \phi(s).
\]
Since
\[
    \rho(s)
    :=
    \frac{q_\delta(s)}{\phi(s)}
    =
    \exp\!\left(
        \delta s-\frac{\delta^2}{2}
    \right)
\]
is strictly increasing for $\delta>0$, we have
\[
    q_\delta(s)>\phi(s)
    \quad\Longleftrightarrow\quad
    s>\frac{\delta}{2}.
\]
Therefore
\[
    \cA_{\omega^{\mathrm{unif}}}
    =
    \left(\frac{\delta}{2},\infty\right).
\]
Furthermore,
\[
    H_{\omega^{\mathrm{unif}}}(s)
    =
    \Phi(s-\delta)-\Phi(s),
\]
and hence
\[
    H_{\omega^{\mathrm{unif}}}(\infty)=0,
\]
while
\begin{align*}
    H_{\omega^{\mathrm{unif}}}(\delta/2)
    &=
    \Phi(-\delta/2)-\Phi(\delta/2) \\
    &=
    1-2\Phi(\delta/2).
\end{align*}
It follows that the rejection probability is
\[
    H_{\omega^{\mathrm{unif}}}(\infty)
    -
    H_{\omega^{\mathrm{unif}}}(\delta/2)
    =
    2\Phi(\delta/2)-1,
\]
and therefore
\begin{align*}
    \alpha_K(\delta;\omega^{\mathrm{unif}})
    &=
    1-\bigl(2\Phi(\delta/2)-1\bigr) \\
    &=
    2\bigl(1-\Phi(\delta/2)\bigr) \\
    &=
    2\phibar(\delta/2).
\end{align*}

Finally, consider the top-rank weights
\[
    \omega^{\max}=(0,\ldots,0,1).
\]
The largest of $K$ i.i.d.\ uniform random variables has CDF $a^K$
and density $Ka^{K-1}$, so
\[
    C_{\omega^{\max}}(a)=a^K,
    \qquad
    b_{\omega^{\max}}(a)=Ka^{K-1}.
\]
Consequently,
\[
    p_{\omega^{\max}}(s)
    =
    K\phi(s)\Phi(s)^{K-1}.
\]
Every finite boundary point of $\cA_{\omega^{\max}}$ satisfies
$q_\delta(s)=p_{\omega^{\max}}(s)$.  Dividing by $\phi(s)>0$ gives
\[
    \frac{\phi(s-\delta)}{\phi(s)}
    =
    K\Phi(s)^{K-1},
\]
or equivalently
\[
    \exp\!\left(
        \delta s-\frac{\delta^2}{2}
    \right)
    =
    K\Phi(s)^{K-1}.
\]

Suppose now that $K>1$.  To characterize the solutions of this
equation, define
\[
    L_K(s)
    :=
    \log K
    +(K-1)\log\Phi(s)
    -\delta s
    +\frac{\delta^2}{2}.
\]
Then
\[
    L_K(s)
    =
    \log
    \frac{p_{\omega^{\max}}(s)}{q_\delta(s)},
\]
so the boundary equation is precisely
\[
    L_K(s)=0.
\]
Differentiating,
\[
    L_K'(s)
    =
    (K-1)\frac{\phi(s)}{\Phi(s)}-\delta.
\]
The inverse Mills ratio $\phi(s)/\Phi(s)$ is strictly decreasing,
with
\[
    \frac{\phi(s)}{\Phi(s)}
    \longrightarrow \infty
    \quad\text{as }s\to-\infty,
    \qquad
    \frac{\phi(s)}{\Phi(s)}
    \longrightarrow 0
    \quad\text{as }s\to\infty.
\]
Hence $L_K'$ vanishes at exactly one point.  Thus $L_K$ is strictly
increasing up to its unique maximizer and strictly decreasing
afterwards.

Moreover,
\[
    \lim_{s\to\infty}L_K(s)=-\infty.
\]
Also, using the standard Gaussian left-tail asymptotics
$\log\Phi(s)=-s^2/2+o(s^2)$ as $s\to-\infty$, and since $K>1$,
\[
    \lim_{s\to-\infty}L_K(s)=-\infty.
\]
The maximum of $L_K$ must be strictly positive.  Indeed, if
$L_K(s)\le0$ for every $s$, then
\[
    p_{\omega^{\max}}(s)\le q_\delta(s)
    \qquad\text{for all }s.
\]
Since both sides are probability densities with total mass one, this
would force equality almost everywhere.  But
$p_{\omega^{\max}}$ and $q_\delta$ are distinct for $K>1$ and
$\delta>0$, a contradiction.

Therefore $L_K$ crosses zero exactly twice.  Denote the two roots by
\[
    s_-<s_+.
\]
Since $L_K$ is positive between its roots and negative outside them,
\[
    p_{\omega^{\max}}(s)>q_\delta(s)
    \quad\Longleftrightarrow\quad
    s\in(s_-,s_+).
\]
Equivalently,
\[
    \cA_{\omega^{\max}}
    =
    (-\infty,s_-)\cup(s_+,\infty).
\]
This completes the proof.
\end{proof}
The acceptance guarantee provided by Theorem \ref{thm:acceptance} has the similar form as the D-GRS guarantee by \citet[Theorem~2]{bullo2026drafttrees}, i.e., 
$$1-G_{K+1}=\Phi(\ln\lambda_K-\delta/2)+\lambda_K\phibar(\ln\lambda_K+\delta/2),$$ with $\{\lambda_j\}_{j=0}^K$ a sequence of scalar levels defined as in \citep{bullo2026drafttrees}, and neither requires quantities dependent on $d$. 
The two guarantees are explicitly compared in Figure \ref{fig:iid-comparison}. 
\section{Implementing RASS}
\label{app:rass-impl}

We describe how to compute the rank weights and sample the residual
distribution in lines~5 and~11 of Algorithm~\ref{alg:rass-verify}.
We use the scalar densities $q_\delta$ and $p_\omega$ defined in
Sections~2 and~3.

\subsection{Rank Weights}
\label{app:rankweights-impl}

For fixed $(\delta,K)$, the optimization
in~\eqref{eq:rank-optimization} chooses weights
$\omega\in\mathcal W_K$ to maximize the overlap between the
selected-proposal density $p_\omega$ and the target density $q_\delta$.
The following property allows us to formulate its quadrature
approximation as a linear program.

\begin{proposition}
\label{prop:concavity}
For fixed $(\delta,K)$, the acceptance probability
$\alpha_K(\delta;\omega)$ is concave in $\omega$ on $\mathcal W_K$.
\end{proposition}

\begin{proof}
By Proposition~\ref{prop:selected-law}, $p_\omega(s)$ is linear in
$\omega$ for each $s$. Thus,
$\min\{q_\delta(s),p_\omega(s)\}$ is concave in $\omega$.
Integrating over $s$ preserves concavity.
\end{proof}

Evaluating the overlap through its closed-form expression requires locating the boundaries of $\{s:q_\delta(s)>p_\omega(s)\}$, which depend on $\omega$. Quadrature instead evaluates both densities at each grid point, so the optimization does not require solving $q_\delta(s)=p_\omega(s)$ as the weights change.

\paragraph{Linear program.}
Let $s_1<\cdots<s_M$ be the quadrature grid points and $\nu_1,\ldots,\nu_M>0$ the corresponding quadrature weights. Using the order-statistic densities $f_r$ from Proposition~\ref{prop:selected-law}, define
\[
  \mathsf B_r(s_i):=\phi(s_i)f_r(\Phi(s_i)),
  \qquad
  p_\omega(s_i)=\sum_{r=1}^K\omega_r\mathsf B_r(s_i).
\]
The quadrature approximation to the acceptance probability is
\[
  \widetilde\alpha_K(\delta;\omega)
  :=\sum_{i=1}^M\nu_i
    \min\left\{q_\delta(s_i),
               \sum_{r=1}^K\omega_r\mathsf B_r(s_i)\right\}.
\]
Maximizing this objective over $\mathcal W_K$ is equivalent to
solving
\begin{equation}
\begin{aligned}
  \underset{\omega\in\mathcal W_K,\;z\in\mathbb R^M}
    {\operatorname{maximize}}
  \quad & \sum_{i=1}^M\nu_i z_i\\
  \text{subject to}\quad
  & 0\leq z_i\leq q_\delta(s_i),\\
  & z_i\leq\sum_{r=1}^K\omega_r\mathsf B_r(s_i),
    \qquad i=1,\ldots,M.
\end{aligned}
\label{eq:rank-lp}
\end{equation}
Since every $\nu_i$ is positive, each $z_i$ equals the smaller of
its two upper bounds at an optimum. Hence~\eqref{eq:rank-lp}
represents the quadrature objective exactly. Only the replacement of the integral by quadrature introduces an approximation.

The program has $K+M$ variables and $2M$ upper-bound inequalities,
in addition to nonnegativity and the simplex constraint. Its size is independent of the state dimension $d$, and its coefficients depend only on $(\delta,K)$ and the quadrature rule.

\paragraph{Approximation quality.}
Let $\widetilde\omega$ and $\omega^\star$ solve
\eqref{eq:rank-lp} and \eqref{eq:rank-optimization}, respectively.
Define the uniform quadrature error
\[
  \epsilon_M
  :=\sup_{\omega\in\mathcal W_K}
    \bigl|\widetilde\alpha_K(\delta;\omega)
          -\alpha_K(\delta;\omega)\bigr|.
\]
Optimality of $\widetilde\omega$ for the quadrature objective gives
\[
  \alpha_K(\delta;\widetilde\omega)
  \geq \widetilde\alpha_K(\delta;\widetilde\omega)-\epsilon_M
  \geq \widetilde\alpha_K(\delta;\omega^\star)-\epsilon_M
  \geq \alpha_K(\delta;\omega^\star)-2\epsilon_M.
\]
Thus, the loss in acceptance probability is at most $2\epsilon_M$.
This bound accounts for both integration error within the grid
and any omitted tail mass.

We assess the numerical loss by comparing the weights returned by
Algorithm~\ref{alg:rankweights} with those obtained using a finer
quadrature grid of $6001$ points per interval. We evaluate the
acceptance probabilities of both sets of weights using
Theorem~\ref{thm:acceptance}.
For $K\in\{2,3,4,6,8\}$ and $14$ values of
$\delta\in[0.01,6]$, the largest observed loss is
$2.1\times10^{-5}$.
By comparison, the gap between optimized rank mixtures and the
optimal list-coupling acceptance probability $\alpha_K^{\rm up}$
of Proposition~\ref{prop:optimal-list-acceptance} reaches
$3.7\times10^{-2}$ (Figure~\ref{fig:iid-comparison}).
These results indicate that quadrature error has a negligible
effect on acceptance in the tested settings.

\paragraph{Implementation.}
Algorithm~\ref{alg:rankweights} implements the routine
$\RankWeights(\delta,K)$ used in line~5 of
Algorithm~\ref{alg:rass-verify}.
To reuse weights across transitions with similar mean displacements,
we round $\delta$ to three \textit{significant} digits, obtaining
$\widetilde\delta$, and cache the weights by
$(\widetilde\delta,K)$.
For $\delta>0$, this gives a relative rounding error of at most
$5\times10^{-3}$, including for small displacements.

The quadrature grid is the sorted union, with duplicates removed,
of two equally spaced grids, each containing $361$ points:
one on $[-9,9]$ and the other on
$[\widetilde\delta-9,\widetilde\delta+9]$.
These intervals extend nine standard deviations around the
proposal and target means used in the optimization.
On the merged grid, the composite trapezoidal weights are
\[
  \nu_1=\frac{s_2-s_1}{2},
  \qquad
  \nu_i=\frac{s_{i+1}-s_{i-1}}{2}
  \quad (2\leq i\leq M-1),
  \qquad
  \nu_M=\frac{s_M-s_{M-1}}{2}.
\]
To improve numerical stability in the tails, we compute
$\mathsf B_r(s_i)$ through its logarithm, using
$\log\Phi(s_i)$ and $\log\Phi(-s_i)$ for the beta-density factors.
We solve~\eqref{eq:rank-lp} using the dual simplex method of
HiGHS~\citep{huangfu2018highs}.


\begin{algorithm}[t]
\caption{$\RankWeights(\delta,K)$: finite-grid rank optimization}
\label{alg:rankweights}
\small
\begin{algorithmic}[1]
\Require Normalized mean gap $\delta\geq0$; integer $K\geq1$
\Ensure Rank weights $\omega\in\mathcal W_K$
\State $\widetilde\delta\gets\operatorname{round}_3(\delta)$
  \Comment{Three significant digits.}
\State $\omega^{\rm unif}\gets(1/K,\ldots,1/K)$;
       $\omega^{\rm max}\gets(0,\ldots,0,1)$
\If{$K=1$ \textbf{or} $\widetilde\delta<10^{-7}$}
  \State \Return $\omega^{\rm unif}$
\EndIf
\If{$\widetilde\delta>12$}
  \State \Return $\omega^{\rm max}$\Comment{Optimal for large $\delta$; here $\alpha_K<10^{-8}$.}
\EndIf
\State $(s_1,\ldots,s_M)\gets\operatorname{sort\_unique}\bigl(
  \operatorname{linspace}(-9,9,361)
  \cup\operatorname{linspace}(\widetilde\delta-9,\widetilde\delta+9,361)\bigr)$
\State $\nu_1\gets\tfrac12(s_2-s_1)$;
       $\nu_M\gets\tfrac12(s_M-s_{M-1})$;
       $\nu_i\gets\tfrac12(s_{i+1}-s_{i-1})$, $i=2,\ldots,M-1$
\State $q_i\gets\phi(s_i-\widetilde\delta)$;
       $\mathsf B_{r}(s_i)\gets\phi(s_i)f_r(\Phi(s_i))$,
       $i=1,\ldots,M$, $r=1,\ldots,K$
\State $\mathcal W_{\rm cand}\gets\{\omega^{\rm unif},\omega^{\rm max}\}$
\If{the program~\eqref{eq:rank-lp} is solved with weights $\widehat\omega$}
  \State $v_r\gets\max\{\widehat\omega_r,0\}$, $r=1,\ldots,K$
    \Comment{Remove solver round-off.}
  \State $\mathcal W_{\rm cand}\gets\mathcal W_{\rm cand}\cup
         \{v/\textstyle\sum_{r}v_r\}$
\EndIf
\State \Return $\displaystyle\arg\max_{w\in\mathcal W_{\rm cand}}
  \sum_{i=1}^M\nu_i\min\Bigl\{q_i,\sum_{r=1}^K\mathsf B_{r}(s_i)w_r\Bigr\}$
\end{algorithmic}
\end{algorithm}

\paragraph{Optimized weights.}
Figure~\ref{fig:rank-weights} shows the numerically optimized weights
returned by Algorithm~\ref{alg:rankweights} for $K=4$ and $K=8$.
As $\delta\to0$, the weights approach the uniform distribution,
recovering the proposal density $\phi$.
Increasing $\delta$ shifts weight toward higher ranks, with nearly
all mass assigned to the highest rank for $\delta\gtrsim2$
when $K=4$ and $\delta\gtrsim2.5$ when $K=8$.

The irregular variation of intermediate-rank weights at small
$\delta$ can arise from nonuniqueness of the linear-program solution.
When several solutions are optimal, the simplex solver selects
a vertex of the optimal face, and this choice may change as
$\delta$ varies. Such changes need not indicate a change in
acceptance: for a fixed $(\delta,K)$, all solutions on the optimal
face yield the same quadrature objective.

\begin{figure}[t]
  \centering
  \includegraphics[width=0.8\linewidth]{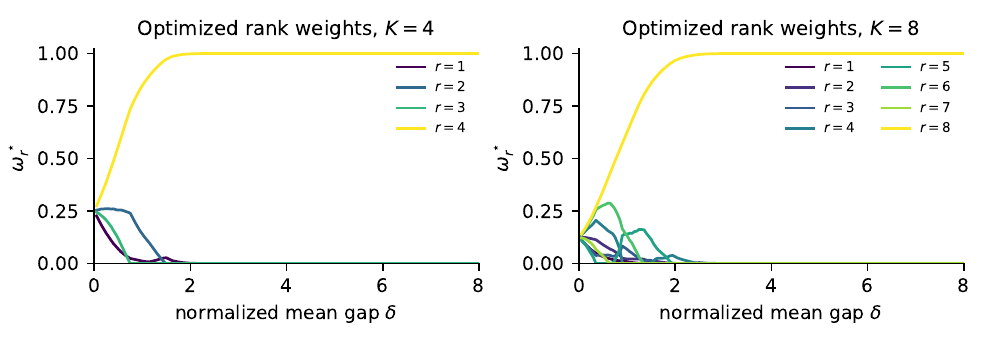}
  \caption{Numerically optimized rank weights returned by Algorithm~\ref{alg:rankweights} for $K=4$ and $K=8$. The weights approach uniformity as $\delta\to0$ and concentrate on the highest rank for $\delta\gtrsim2$ ($K=4$) and $\delta\gtrsim2.5$ ($K=8$). Irregular changes in intermediate-rank weights at small $\delta$ can reflect transitions between distinct optimal solutions, which yield the same quadrature objective at a given $\delta$.}
  \label{fig:rank-weights}
\end{figure}

\subsection{Residual Sampling}
\label{app:residual}

On rejection, line~11 of Algorithm~\ref{alg:rass-verify} samples
the scalar projection from
\[
  r_\omega(s)
  =\frac{(q_\delta(s)-p_\omega(s))_+}{\eta_\omega},
  \qquad
  \eta_\omega=1-\alpha_K(\delta;\omega).
\]
This branch is reached only when $\eta_\omega>0$.
Recall from Theorem~\ref{thm:acceptance} that the residual is
positive on
\[
  A_\omega=\{s:q_\delta(s)>p_\omega(s)\}
           =\bigcup_j(a_j,b_j),
\]
where the disjoint intervals are ordered from left to right.
The function
\[
  H_\omega(s)=\Phi(s-\delta)-C_\omega(\Phi(s))
\]
satisfies $H_\omega'=q_\delta-p_\omega$ and
$H_\omega(\pm\infty)=0$. Hence, the unnormalized residual mass
on interval $(a_j,b_j)$ is
\[
  m_j=H_\omega(b_j)-H_\omega(a_j),
  \qquad
  \sum_jm_j=\eta_\omega.
\]
The beta CDFs defining $C_\omega$ have integer parameters, giving
the polynomial expression
\[
  C_\omega(a)
  =\sum_{r=1}^K\omega_r
    \sum_{\ell=r}^K\binom K\ell a^\ell(1-a)^{K-\ell}.
\]
We can therefore compute interval masses without numerical
integration. In the implementation, we evaluate $H_\omega$ as
\[
  H_\omega(s)
  =-[\Phi(s)-\Phi(s-\delta)]
   -[C_\omega(\Phi(s))-\Phi(s)].
\]

\paragraph{Locating the residual intervals.}
Every finite endpoint satisfies
\[
  q_\delta(s)=p_\omega(s),
  \qquad\text{equivalently}\qquad
  \exp(\delta s-\delta^2/2)=b_\omega(\Phi(s)).
\]
The residual need not be confined to a single right tail.
For example, with $K=2$, $\omega=(0,1)$, and $\delta=1$,
\[
  p_\omega(s)=2\phi(s)\Phi(s),
  \qquad
  H_\omega(s)=\Phi(s-1)-\Phi(s)^2,
\]
and
\[
  A_\omega\approx(-\infty,-1.486)\cup(1.028,\infty).
\]
The left interval contains approximately $0.8\%$ of the residual
mass, so sampling only from the right interval would not preserve
the target law.
To locate the intervals numerically, we evaluate
$\log q_\delta-\log p_\omega$ on a grid of $513$ points and
refine each detected sign change using Brent's method.
This is a numerical search: a finite grid alone does not
guarantee that every crossing is detected.
At a simple crossing, an endpoint error of size $\epsilon$
changes the corresponding interval mass by $O(\epsilon^2)$,
since $q_\delta-p_\omega$ vanishes at the endpoint.
We cache the interval endpoints and masses using the unrounded
pair $(\delta,\omega)$.

\paragraph{Inverse-CDF sampler.}
Draw $U\sim\mathcal U(0,1)$ and set $v=U\eta_\omega$.
First, select the interval $j$ satisfying
\[
  \sum_{i<j}m_i<v<\sum_{i\leq j}m_i.
\]
Equality at an interval boundary has probability zero and may
be resolved arbitrarily. Then solve
\[
  H_\omega(S)-H_\omega(a_j)
  =v-\sum_{i<j}m_i,
  \qquad S\in(a_j,b_j),
\]
using Brent's method. Since $H_\omega'>0$ on this interval,
the solution is unique. With exact interval masses and inversion,
this construction samples from $r_\omega$.
Each draw uses one uniform random variable and one scalar root
solve, with no numerical integration or additional target-model
evaluation.

\paragraph{Alternative rejection sampler.}
One can instead repeatedly draw $S\sim\mathcal N(\delta,1)$
and accept with probability
\[
  \left(1-\frac{p_\omega(S)}{q_\delta(S)}\right)_+.
\]
The acceptance probability of each trial is $\eta_\omega$,
so this procedure requires $1/\eta_\omega$ trials on average.
Conditional on needing a correction, this can be expensive
when $\alpha_K(\delta;\omega)$ is close to one.
The inverse-CDF method avoids these repeated trials: after
computing or retrieving the residual intervals, it performs
one root solve per draw.

\paragraph{Orthogonal component.}
We reconstruct the corrected state as
\[
  Y=\mu_p+\sigma(Se+Z_{\perp,1}),
\]
reusing the orthogonal component of the first child in drafting
order. This component retains its Gaussian law independently
of the scalar verification and residual draw.
An independent sample from
$\mathcal N(0,I_d-ee^\top)$ can be used instead, with the same
output law.

\subsection{RASS Computational Overhead}
\label{app:overhead}
RASS introduces two scalar computations: selecting the rank weights
and sampling the residual after rejection. The weights can be
cached, while the residual admits either inverse-CDF or rejection
sampling. We benchmark these computations separately and measure
their combined contribution to verification. With cached weights
and rejection sampling, the mean verification time of RASS is
within $1.5\%$ of D-GRS for every tested branching factor.

\paragraph{Benchmark setup.}
We measure scalar verification costs on synthetic inputs, using
one core of an Intel Core i9-10980XE and no diffusion model.
We consider $K\in\{2,3,4,6,8\}$ and normalized mean displacements
$\delta\in\{0.01,0.05,0.1,0.25,0.5,1,2,4,8\}$.
For each pair, we set $\mu_p=0$, $\mu_q=\delta e_1$, and $\sigma=1$
in dimension $d=16$, and generate proposal children using a fixed
random seed.
Timings exclude the $d$-dimensional projections, which are common
to RASS, D-GRS and RMC.
Each reported value averages over the tested displacements and
repetitions, with the half-width of a $95\%$ confidence interval.

\paragraph{Rank-weight computation.}
Table~\ref{tab:rass-subroutines} reports the cost of individual
subroutines, using $1{,}000$ repetitions per $(K,\delta)$.
Solving $\RankWeights$ takes $5.8$--$8.0$\,ms.
As described above, this computation depends only on
$(\widetilde\delta,K)$: once computed, the weights are reused
whenever the same pair occurs.
This cost is therefore incurred on a cache miss rather than
at every verified node. Its contribution during sampling depends
on how frequently new pairs are encountered.

\paragraph{Residual sampling.}
An inverse-CDF residual draw takes $0.30$--$0.66$\,ms.
Its cost grows with $K$, consistent with the increasing number
of terms in $b_\omega$ and $C_\omega$.
Rejection sampling has a different cost profile.
Each trial succeeds with probability $\eta_\omega$, so a residual
draw requires $1/\eta_\omega$ trials on average.
The measured cost is below $14\,\mu$s for $\delta\geq2$, but reaches
$4$--$44$\,ms at $\delta=0.01$, where the residual mass is small.
These conditional costs do not directly determine the average
verification overhead. A residual draw is needed only with
probability $\eta_\omega$. Thus, for $\eta_\omega>0$, the expected
number of residual rejection-sampling trials per verified node is
\[
  \Pr(\text{rejection})\,
  \mathbb E[\text{trials}\mid\text{rejection}]
  =\frac{1}{\eta_\omega}\eta_\omega=1.
\]
Consequently, the large number of trials required by an individual
correction at small $\delta$ is offset, in expectation, by the
rarity of such corrections. The time per trial can still vary
with $(\delta,K)$.

\paragraph{Verification cost.}
Table~\ref{tab:verify-overhead} measures scalar verification with
cached rank weights, using $3{,}000$ calls per $(K,\delta)$ and
fresh children at each call.
The timings include residual sampling whenever verification
rejects the selected proposal.
Averaged over the tested displacements, RMC takes approximately
$12.6\,\mu$s per call and D-GRS takes $30$--$55\,\mu$s.
RASS with rejection sampling also takes $30$--$55\,\mu$s,
with mean times within $1.5\%$ of D-GRS for every tested $K$.
RASS with inverse-CDF sampling takes $117$--$168\,\mu$s.

These results show that the choice of residual sampler materially
affects verification overhead. Both constructions sample the same
residual law in exact arithmetic and therefore preserve the same
acceptance probabilities and target law.

With rejection sampling and cached weights, RASS therefore adds
little average scalar verification cost relative to D-GRS in this
benchmark. Computing previously unseen rank weights remains an
additional cost. The measured scalar computations are independent
of the state dimension and run entirely on the CPU.

\begin{table}[t]
\centering
\small
\begin{tabular}{@{}r c cc@{}}
\toprule
 & $\RankWeights$ & \multicolumn{2}{c}{Residual draw ($\mu$s)} \\
\cmidrule(l){3-4}
$K$ & (ms) & inverse CDF & rejection sampler \\
\midrule
2 & 5.78\,{\scriptsize$\pm$0.04} & 295\,{\scriptsize$\pm$3} & 645\,{\scriptsize$\pm$42} \\
3 & 6.94\,{\scriptsize$\pm$0.04} & 314\,{\scriptsize$\pm$1} & 682\,{\scriptsize$\pm$43} \\
4 & 7.22\,{\scriptsize$\pm$0.05} & 400\,{\scriptsize$\pm$3} & 1820\,{\scriptsize$\pm$117} \\
6 & 7.42\,{\scriptsize$\pm$0.04} & 513\,{\scriptsize$\pm$2} & 3654\,{\scriptsize$\pm$224} \\
8 & 7.95\,{\scriptsize$\pm$0.04} & 657\,{\scriptsize$\pm$4} & 6617\,{\scriptsize$\pm$426} \\
\bottomrule
\end{tabular}
\caption{Costs of individual RASS subroutines on one CPU core.
Entries average $1{,}000$ repetitions per $(K,\delta)$ over the
nine displacements listed in the benchmark setup and report
mean $\pm$ the half-width of a $95\%$ confidence interval.
$\RankWeights$ measures one uncached linear-program solve.
Residual timings measure one correction draw, conditional on
requiring a correction; they are not weighted by the probability
of rejection.}
\label{tab:rass-subroutines}
\end{table}

\begin{table}[t]
\centering
\small
\begin{tabular}{@{}r cc cc@{}}
\toprule
 & & & \multicolumn{2}{c}{RASS} \\
\cmidrule(l){4-5}
$K$ & RMC & D-GRS & inverse CDF & rejection sampler \\
\midrule
2 & 12.55\,{\scriptsize$\pm$0.02} & 29.90\,{\scriptsize$\pm$0.20} & 117\,{\scriptsize$\pm$2} & 29.53\,{\scriptsize$\pm$0.30} \\
3 & 12.60\,{\scriptsize$\pm$0.02} & 33.90\,{\scriptsize$\pm$0.21} & 121\,{\scriptsize$\pm$2} & 33.84\,{\scriptsize$\pm$0.42} \\
4 & 12.51\,{\scriptsize$\pm$0.02} & 37.72\,{\scriptsize$\pm$0.21} & 129\,{\scriptsize$\pm$2} & 38.21\,{\scriptsize$\pm$0.55} \\
6 & 12.73\,{\scriptsize$\pm$0.02} & 47.05\,{\scriptsize$\pm$0.24} & 159\,{\scriptsize$\pm$3} & 47.71\,{\scriptsize$\pm$0.67} \\
8 & 12.58\,{\scriptsize$\pm$0.02} & 54.69\,{\scriptsize$\pm$0.25} & 168\,{\scriptsize$\pm$3} & 55.49\,{\scriptsize$\pm$1.36} \\
\bottomrule
\end{tabular}
\caption{Scalar verification time per node ($\mu$s) on one CPU
core, with rank weights already cached.
Entries average $3{,}000$ calls with fresh children per
$(K,\delta)$ over the nine tested displacements and report
mean $\pm$ the half-width of a $95\%$ confidence interval.
Residual sampling is included only when verification rejects.
Rank-weight computation and $d$-dimensional projections are
excluded. RMC uses one child, whereas RASS and D-GRS use $K$.}
\label{tab:verify-overhead}
\end{table}

\subsection{Exactness Under Numerical Approximation}
\label{app:rass-numerics}

Theorem~\ref{thm:exact} guarantees exactness for every feasible
$\omega\in\mathcal W_K$, regardless of how the weights are obtained.
Thus, quadrature and rounding in $\RankWeights$ affect acceptance but do not invalidate the exactness guarantee. Neither does the choice among optimal solutions.
The acceptance test and residual sampler use the actual displacement $\delta$ and the returned weights $\omega$, rather than the rounded
displacement $\widetilde\delta$ used during optimization.
Numerical errors may nevertheless arise when implementing these
steps, through floating-point arithmetic, the finite-grid search
for residual intervals, and root-solver tolerances.
The exactness theorem applies to the ideal verifier. However, its numerical
implementation is subject to these errors.
Table~\ref{tab:quality} compares sample quality across methods using
FID and CLIP scores, while
Appendix~\ref{app:qualitative-exactness} provides visual comparisons.
These results complement the theoretical guarantee but do not establish distributional exactness.

\section{Experimental details and additional results}
\label{app:exp-details}

\subsection{Delayed-drift proposals and prefetching}
\label{app:delayed-drift-prefetch}

We specify the proposal construction and reuse policy used within the draft-tree sampler of Section~\ref{sec:background} and Algorithm~\ref{alg:speculative-template}. Delayed drift determines how a cached target evaluation is used to draft a tree. Prefetching determines which evaluation is retained for the next round. For simplicity, conditioning information, such as class labels or text prompts, is held fixed and suppressed below.

\paragraph{A general proposal template.}
Write the target mean in Eq.~(3) as
\begin{equation}
    m_k^q(x)=\mathcal M_k\bigl(x,h_k(x)\bigr),
    \label{eq:cached-target-template}
\end{equation}
where $h_k(x)$ is the expensive model field and $\mathcal M_k$ is an inexpensive reconstruction map incorporating the state and sampling schedule. The field may be a drift, score, denoiser output, or velocity. At the start of a round at $Y_n$, let $\bar h=h_{\tilde n}(\tilde y)$ be an available evaluation, with $\tilde n\leq n$. We hold this value fixed throughout drafting and use
\begin{equation}
    m_k^p(x;\bar h)=\mathcal M_k(x,\bar h),
    \qquad
    P_k(\cdot\mid x,\bar h)
    =\mathcal N\!\left(m_k^p(x;\bar h),\sigma_k^2 I_d\right).
    \label{eq:cached-proposal-template}
\end{equation}
Thus, at a parent $u$ of depth $|u|$, drafting uses $k=n+|u|$ and $x=\widehat Y_u$. Each child is sampled according to \eqref{eq:tree-proposal}, with the model field frozen: the state and schedule arguments of $\mathcal M_k$ remain those of the drafted node. Conditional on the cache, this construction needs no further model evaluation during drafting. If no evaluation is available initially, we obtain $\bar h=h_n(Y_n)$ with one target evaluation before drafting.

For the Euler-Maruyama parameterization in \eqref{eq:target-kernel}, take $h_k=b^q_{t_k}$ and $\mathcal M_k(x,z)=x+\gamma z$, giving $m_k^p(x;\bar h)=x+\gamma\bar h$. More generally, an affine velocity parameterization,
\begin{equation}
    m_k^q(x)=a_kx+b_kv_k(x),
    \qquad
    m_k^p(x;\bar v)=a_kx+b_k\bar v,
    \label{eq:cached-affine-velocity}
\end{equation}
uses known scalar coefficients $a_k,b_k$ and a cached velocity $\bar v$. This retains the current coefficients and state-dependent term even when the velocity was evaluated at a different state or timestep.

\paragraph{Prefetching between rounds.}
Let $\mathcal E\subseteq V$ contain the nodes whose target evaluations are available from the round's parallel evaluation, including any exact evaluations reused from earlier rounds. For $v\in\mathcal E$, write
\[
    H_v=h_{n+|v|}(\widehat Y_v).
\]
These values can be retained when the target means are evaluated, or recovered from those means when the reconstruction map is invertible in its second argument. After the round commits $N_\alpha$ transitions, using the convention of Algorithm~2, its next root is $y^+=Y_{n^+}$ with $n^+=n+N_\alpha$. No update is needed if $n^+=N$.

Let $u$ be the parent of the final verified transition and set $\mathcal A=C(u)\cap\mathcal E$. The nearest-node policy used in our experiments selects the next cache as follows:
\begin{enumerate}
    \item If the round ends by accepting a leaf $v_\star\in\mathcal E$, retain its own evaluation: $$\bar h^+=H_{v_\star}=h_{n^+}(y^+).$$
    \item If the round ends by rejection and $\mathcal A\neq\varnothing$, retain the evaluation at the closest evaluated candidate from that final sibling list:
    \[
        v^\dagger\in\operatorname*{arg\,min}_{v\in\mathcal A}
        \|\widehat Y_v-y^+\|_2,
        \qquad \bar h^+=H_{v^\dagger}.
    \]
    These candidates are at timestep $n^+$, although their states generally differ from the residual output $y^+$.
    \item Otherwise, retain the final parent's evaluation: $\bar h^+=H_u=h_{n^+-1}(\widehat Y_u)$.
\end{enumerate}
Ties may be resolved in a fixed order. A parent-only policy always uses the third choice. Disabling prefetching instead evaluates $h_{n^+}(y^+)$ before drafting each new round.

With internal-node-only evaluation, $\mathcal E=I(\mathcal T)$ away from the terminal horizon, so accepted leaves and candidates at the last level are unavailable for reuse. Evaluating the leaves, $\mathcal E=I(\mathcal V)$, makes their fields available in the same parallel phase, supporting the first two choices also at the last level. Before horizon truncation, this adds $K^L$ rows for a full $(K,L)$ tree and gives the full-tree budget $B=|V|$ used in the main experiments. The cache-selection operation itself requires no additional target evaluation.

\paragraph{Specialization to the experimental samplers.}
All the experiments use the linear interpolation
\[
    X_\varsigma=(1-\varsigma)X_0+\varsigma Z,
    \qquad
    X_0\sim q_{\mathrm{data}},\quad
    Z\sim\mathcal N(0,I_d),
\]
with independent $X_0$ and $Z$. This path connects data at $\varsigma=0$ to Gaussian noise at $\varsigma=1$, and its velocity is
\[
    v_\varsigma(x)
    =\mathbb E[Z-X_0\mid X_\varsigma=x].
\]
Sampling proceeds in decreasing $\varsigma$. The Gaussian-mixture experiment evaluates this velocity analytically, the pixel-space experiments approximate it through a pretrained denoiser \citep{karras2022elucidating}, and SD3.5 supplies a learned flow velocity.

Let $\varsigma_k$ denote the interpolation parameter at step $k$, different from the transition standard deviation $\sigma_k$, and set $\Delta\varsigma_k=\varsigma_{k+1}-\varsigma_k<0$. On stochastic steps, the target mean has the affine form
\[
    m_k^q(x)=a_kx+b_kv_{\varsigma_k}(x),
\]
where
\[
    a_k=1+\Delta\varsigma_k\frac{\epsilon^2}{1-\varsigma_k},
    \qquad
    b_k=\Delta\varsigma_k(1+\epsilon^2),
    \qquad
    \sigma_k=\epsilon
    \sqrt{\frac{2\varsigma_k}{1-\varsigma_k}
          (-\Delta\varsigma_k)}.
\]
The cached quantity is a previous velocity evaluation, $\bar v=v_{\tilde\varsigma}(\tilde y)$. The delayed-drift proposal therefore uses
\[
    m_k^p(x;\bar v)=a_kx+b_k\bar v,
\]
holding $\bar v$ fixed during drafting while retaining the current state $x$ and schedule coefficients $a_k,b_k$. All experiments use the prefetch policy described above to select $\bar v$ for the next round, with their respective schedules and conditioning. 

\subsection{Results with internal-node evaluation}
\label{app:internal-node-budget}

Section~\ref{sec:background} presesnts two options for the parallel evaluation phase. The main experiments evaluate the full tree, with budget $B = |\mathcal{V}|$. This section gives the results when the sampler evaluates only the internal nodes $\mathcal{I}(\mathcal{T})$, with budget $B = |\mathcal{I}|$.

\paragraph{Setup.}
Before horizon truncation, a full $(K,L)$ tree has
\begin{equation*}
  |\mathcal{V}| = \sum_{\ell=0}^{L} K^{\ell},
  \qquad
  |\mathcal{I}| = \sum_{\ell=0}^{L-1} K^{\ell} = |\mathcal{V}| - K^{L}.
\end{equation*}
Internal-node evaluation thus removes the $K^{L}$ leaves from the batch. It also removes their target fields from the prefetch cache (Appendix~\ref{app:delayed-drift-prefetch}). If a round accepts a leaf or rejects at the last level, no evaluated candidate is available at the next root. The next round then drafts from the evaluation at the final parent. This evaluation is one step older and is at a different state. Internal-node evaluation thus decreases the batch size, but it can decrease the accuracy of the next proposal.

All other settings are the same as in Section~\ref{sec:experiments}. These settings include the target discretization, the delayed-drift proposal, the prefetch policy, the $(K,L)$ configurations, and the number of samples. RASS and D-GRS use identical trees. RMC uses a chain matched to the same internal-node budget. We compare the two budget definitions at a fixed configuration $(K,L)$ and at a fixed budget $B$.

\paragraph{Gaussian mixture.}
Figures~\ref{fig:gm-speedup-internal-tree-appendix} and~\ref{fig:gm-speedup-heatmap-internal-tree-appendix} give the results with $B = |\mathcal{I}|$. Figures~\ref{fig:gm-speedup-full-tree-appendix} and~\ref{fig:gm-speedup-heatmap-full-tree-appendix} give the corresponding full-tree results. Both sets use churn $\varepsilon \in \{0.06, 0.1, 0.3, 0.6\}$ and the grid $K \in \{1, \dots, 7\}$, $L \in \{2, \dots, 7\}$.
At a fixed configuration, internal-node evaluation decreases the speedup of the two tree methods, mainly for shallow trees. For RASS, the median decrease is $14.2\%$ at $L = 2$, $7.5\%$ at $L = 3$, and $3.0\%$ at $L = 4$. For $L \geq 5$, the median decrease is less than $1\%$. A probable cause is that a shallow tree ends more rounds at its last level, where the leaf evaluations are necessary for an accurate cache.
At a fixed budget, internal-node evaluation gives a higher speedup at moderate budgets. Within $B \leq 100$, the best mean speedup of RASS is $1.1\%$ to $2.1\%$ higher with $B = |\mathcal{I}|$ than with $B = |\mathcal{V}|$. Over the full grid, the highest speedups for the two definitions differ by at most $0.6\%$.
With the internal-node budget, RASS gives a higher mean speedup than D-GRS in all $144$ configurations with $K \geq 2$, by $0.6\%$ to $6.7\%$. In $115$ of these configurations, the $95\%$ confidence intervals do not overlap. For every churn value, RASS matches or exceeds D-GRS at every budget at or above the number of diffusion steps. Its highest mean speedup is $3.9\%$ to $6.6\%$ above that of D-GRS. At $K = 1$, each node has one child, so no ranking is possible. The two rules then give the same speedup within the confidence intervals. RMC exceeds both tree methods at $L = 2$ for every $K$, and in some configurations with $L = 3$ and $K \leq 3$. In these configurations, the matched chain is longer than the tree and can commit more transitions per round.
\begin{figure}[t]
    \centering
    \input{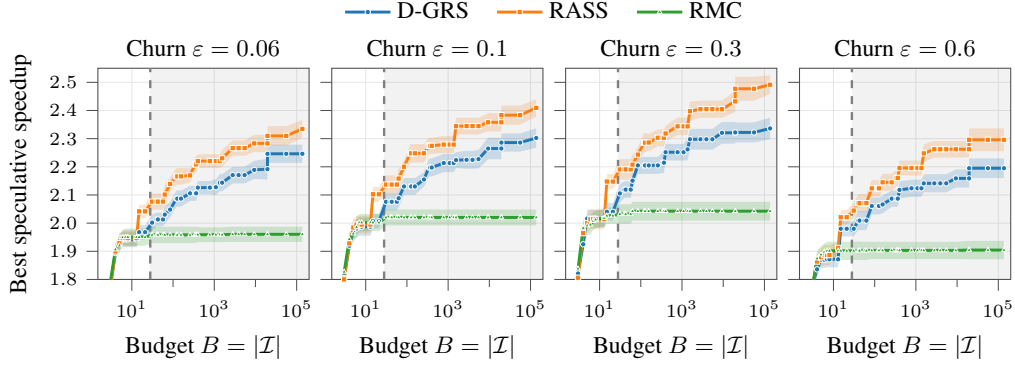}    
    \caption{Best mean target-call speedup achievable within parallel budget $B=|\mathcal{I}|$ on the Gaussian-mixture target, for churn $\varepsilon\in\{0.06,0.1,0.3,0.6\}$. Each point shows the best configuration $(K,L)$ whose budget does not exceed $B$. Shading indicates $95\%$ confidence intervals over $100$ trajectories. The dashed line marks the number of diffusion steps, and the shaded region is budgets at or above it. For every churn value, RASS matches or exceeds D-GRS at every budget in the shaded region.}
    \label{fig:gm-speedup-internal-tree-appendix}
\end{figure}
\begin{figure}[t]
    \centering
    \includegraphics[width=\linewidth]{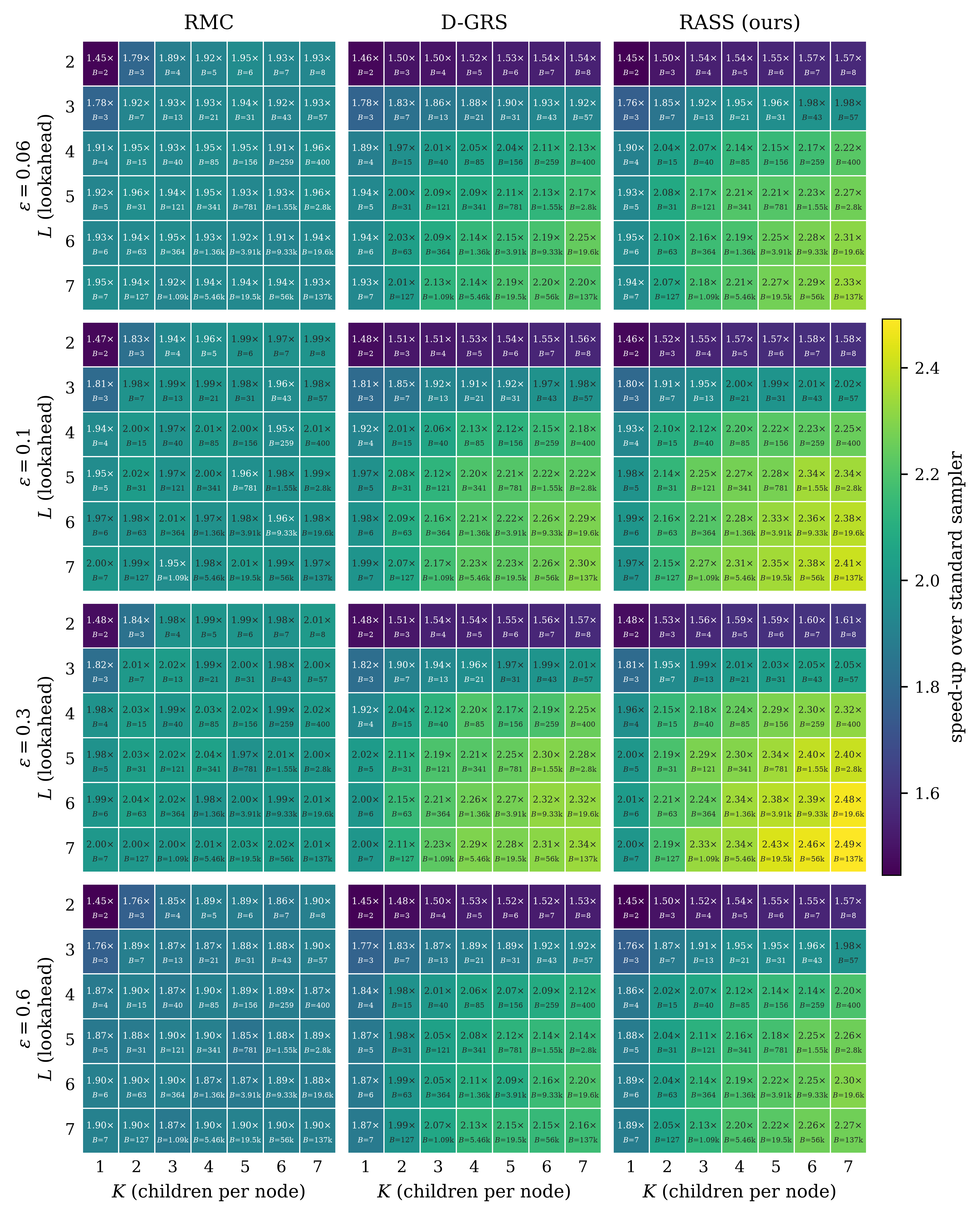}
    \caption{Mean target-call speedup against branching factor $K$ and lookahead $L$, achieved using parallel budget $B=|\mathcal{I}|$.}
    \label{fig:gm-speedup-heatmap-internal-tree-appendix}
\end{figure}
\begin{figure}[t]
    \centering
    \input{figures/gm/gm-speedup-full-tree}
    \caption{Best mean target-call speedup achievable within parallel budget $B=|\mathcal{V}|$ on the Gaussian-mixture target, for churn $\varepsilon\in\{0.06,0.1,0.3,0.6\}$. Each point shows the best configuration $(K,L)$ whose budget does not exceed $B$. Shading indicates $95\%$ confidence intervals over $100$ trajectories. The dashed line marks the number of diffusion steps, and the shaded region is budgets at or above it. For every churn value, RASS matches or exceeds D-GRS at every budget in the shaded region. Its highest speedup is $4.5\%$ to $6.4\%$ above D-GRS's, while RMC stays at about $1.9\times$--$2.0\times$.}
    \label{fig:gm-speedup-full-tree-appendix}
\end{figure}
\begin{figure}[t]
    \centering
    \includegraphics[width=\linewidth]{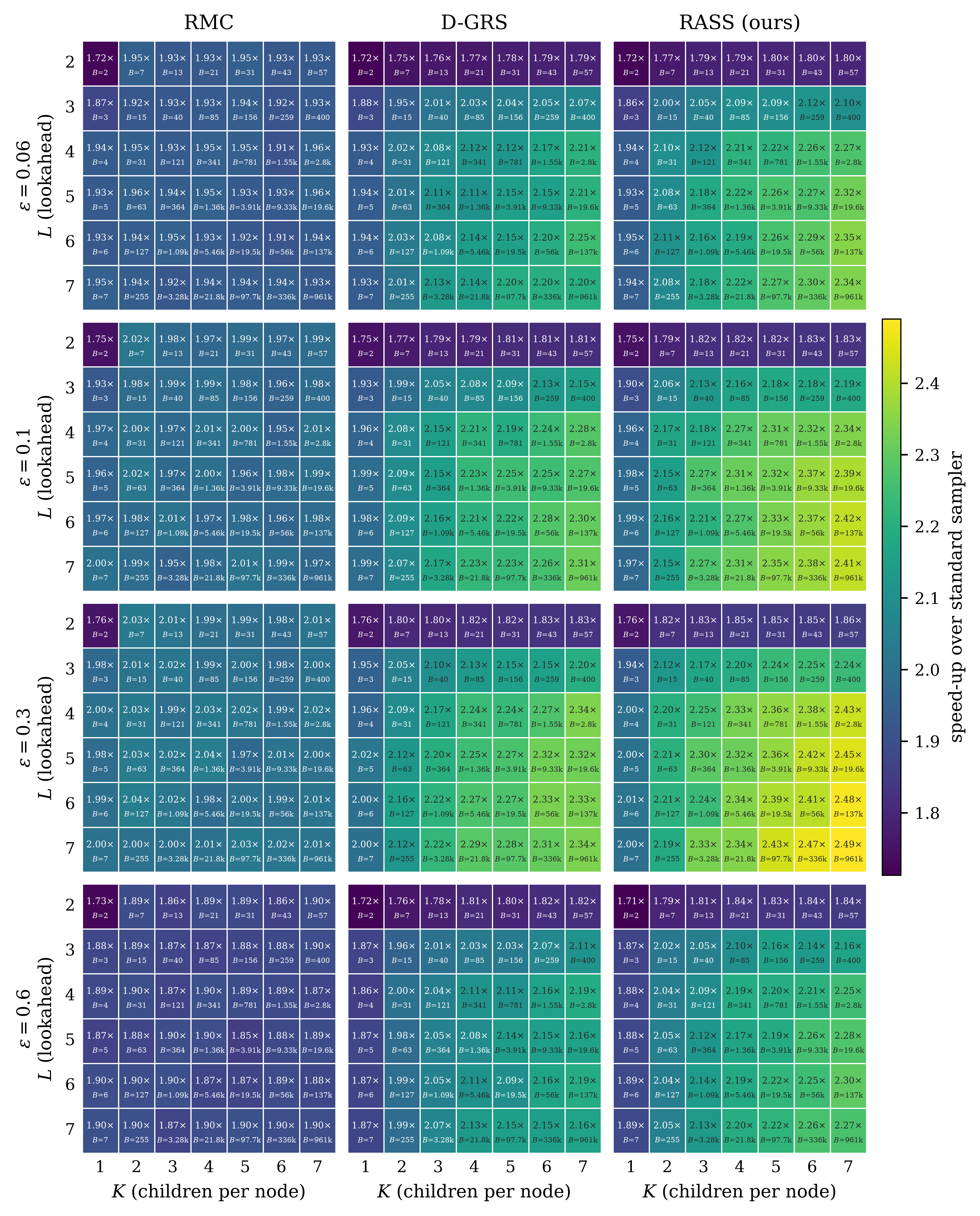}
    \caption{Mean target-call speedup against branching factor $K$ and lookahead $L$, achieved using parallel budget $B=|\mathcal{V}|$.}
    \label{fig:gm-speedup-heatmap-full-tree-appendix}
\end{figure}

\paragraph{Pixel-space diffusion.}
Figures~\ref{fig:pixel-space-speedup-internal} and \ref{fig:pixel-space-speedup-full-internal-trade-off} show the results for CIFAR-10 and FFHQ. The four configurations give $B = |\mathcal{I}| \in \{7, 31, 40, 85\}$, instead of $B = |\mathcal{V}| \in \{15, 63, 121, 341\}$.
At a fixed configuration, internal-node evaluation decreases the speedup of the two tree methods. For RASS, the decrease is between $1.9\%$ and $8.5\%$ on CIFAR-10 and at most $7.1\%$ on FFHQ. The largest decrease occurs at $(K,L) = (2,3)$. A probable cause is that a shallow tree ends more rounds at its last level, where the leaf evaluations are necessary for an accurate cache. The speedup of RMC changes by less than $2\%$. Its matched chains have a length of at least seven, and few rounds reach the end of the chain.
At a fixed budget, internal-node evaluation uses the parallel budget more efficiently. On CIFAR-10, RASS with $B = |\mathcal{I}| = 31$ is within $2.6\%$ of RASS with $B = |\mathcal{V}| = 63$. On FFHQ, the difference between these two values is at most $0.5\%$. Thus, internal-node evaluation gives almost the same speedup with half the budget. On CIFAR-10, RASS with $B = |\mathcal{I}| = 31$ also exceeds the highest full-tree speedup of D-GRS, which requires $B = |\mathcal{V}| = 341$. This result holds for all three churn values.
With the internal-node budget, RASS gives a higher mean speedup than D-GRS in every configuration. The relative gain is between $3.9\%$ and $18.7\%$ on CIFAR-10 and between $1.4\%$ and $2.4\%$ on FFHQ. On CIFAR-10, the gain increases with churn, as with the full-tree budget. It is between $3.9\%$ and $6.6\%$ at $\varepsilon = 0.1$ and between $10.3\%$ and $18.7\%$ at $\varepsilon = 0.6$. At the smallest budget, $B = |\mathcal{I}| = 7$, RMC gives a higher speedup than both tree methods in five of the six panels. A $(2,3)$ tree can commit at most three transitions per round, but the matched chain can commit up to seven. For $B \geq 31$, both tree methods exceed RMC in every seting we tested.
\begin{figure}[t]
  \centering
  \input{figures/pixel-space/pixel-space-style}
\pgfplotstableread[col sep=comma]{figures/pixel-space/data/pixel-space-speedup-results-internal-tree.csv}\PixelData
\begin{tikzpicture}
  \begin{groupplot}[
    group style={
      group size=3 by 2,
      horizontal sep=0.3cm,vertical sep=0.6cm,
      xlabels at=edge bottom,xticklabels at=edge bottom,
      ylabels at=edge left,yticklabels at=edge left
    },
    width=0.255\linewidth,height=0.18\linewidth,
    scale only axis,
    xmin=0,xmax=90,xtick={0,20, 40, 60, 80},
    xlabel={Parallel budget $B=|\mathcal{I}|$},
    ylabel={Mean speedup},
    axis lines=box,tick pos=left,tick align=outside,
    tick label style={font=\scriptsize},
    label style={font=\small},title style={font=\small,inner sep=0pt,yshift=2pt},
    every axis title shift=0pt,
    scaled ticks=false,
    yticklabel style={/pgf/number format/fixed,
      /pgf/number format/precision=1,/pgf/number format/fixed zerofill},
    unbounded coords=discard,filter discard warning=false,
  ]
    \nextgroupplot[
        title={CIFAR-10, $\varepsilon=0.1$},
        ymin=1.9,
        ymax=2.85,
        ytick={2.0,2.2,2.4,2.6,2.8},
        legend to name=pixelLegend, 
        legend columns=3,
        legend style={draw=none,font=\small,/tikz/every even column/.append style={column sep=0.4em}},
        grid=major,
        major grid style={densely dotted,draw=black,opacity=0.2}
    ]
    \PixelPanel{cifar10}{0.1}
    \addlegendimage{pixelDGRS,thick,mark=*}\addlegendentry{D-GRS}
    \addlegendimage{pixelRASS,thick,mark=*}\addlegendentry{RASS}
    \addlegendimage{pixelRMC,thick,mark=*}\addlegendentry{RMC}

    \nextgroupplot[
        title={CIFAR-10, $\varepsilon=0.3$},
        ymin=1.9,
        ymax=2.85,
        ytick={2.0,2.2,2.4,2.6,2.8},
        grid=major,
        major grid style={densely dotted,draw=black,opacity=0.2}
    ]
    \PixelPanel{cifar10}{0.3}

    \nextgroupplot[
        title={CIFAR-10, $\varepsilon=0.6$},
        ymin=1.9,
        ymax=2.85,
        ytick={2.0,2.2,2.4,2.6,2.8},
        grid=major,
        major grid style={densely dotted,draw=black,opacity=0.2}
    ]
    \PixelPanel{cifar10}{0.6}

    \nextgroupplot[
        title={FFHQ, $\varepsilon=0.1$},
        ymin=1.9,
        ymax=2.32,
        ytick={2.0,2.1,2.2,2.3},
        grid=major,
        major grid style={densely dotted,draw=black,opacity=0.2}
    ]
    \PixelPanel{ffhq}{0.1}

    \nextgroupplot[
        title={FFHQ, $\varepsilon=0.3$},
        ymin=1.9,
        ymax=2.32,
        ytick={2.0,2.1,2.2,2.3},
        grid=major,
        major grid style={densely dotted,draw=black,opacity=0.2}
    ]
    \PixelPanel{ffhq}{0.3}

    \nextgroupplot[title={FFHQ, $\varepsilon=0.6$},
        ymin=1.9,
        ymax=2.32,
        ytick={2.0,2.1,2.2,2.3},
        grid=major,
        major grid style={densely dotted,draw=black,opacity=0.2}
    ]
    \PixelPanel{ffhq}{0.6}
  \end{groupplot}
  \node[anchor=south] at ([yshift=0mm]current bounding box.north){\pgfplotslegendfromname{pixelLegend}};

\end{tikzpicture}
  \caption{Mean target-call speedup on CIFAR-10 (top) and FFHQ (bottom) across parallel budgets and churn values. Curves average over $500$ images per configuration, and shading shows $95\%$ confidence intervals.}
  \label{fig:pixel-space-speedup-internal}
\end{figure}
\begin{figure}[t]
  \centering
  \input{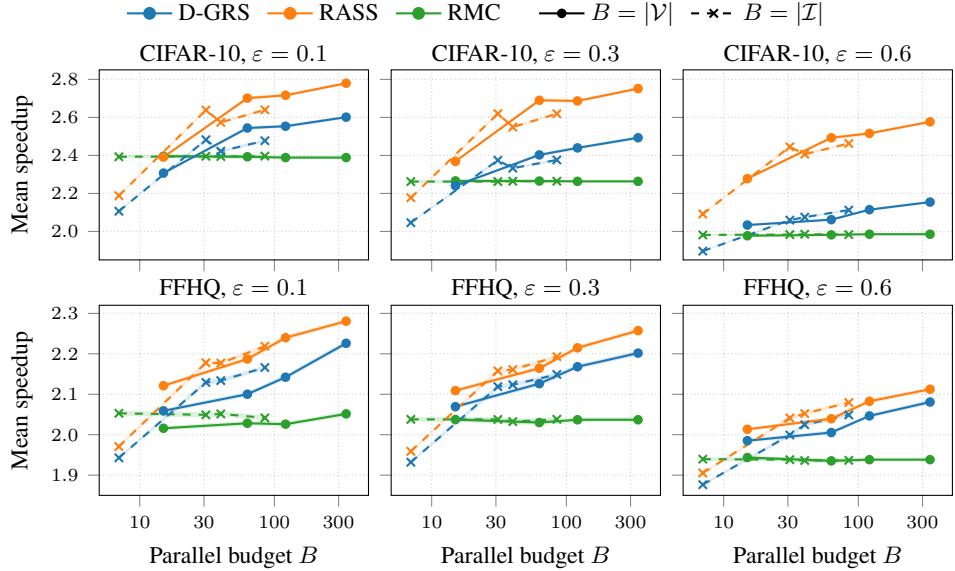}
\pgfplotstableread[col sep=comma]{figures/pixel-space/data/pixel-space-speedup-results-full-tree.csv}\PixelDataFull
\pgfplotstableread[col sep=comma]{figures/pixel-space/data/pixel-space-speedup-results-internal-tree.csv}\PixelDataInternal

\pgfplotsset{
  pixel full/.style={solid,mark=*,mark size=1.4pt},
  pixel internal/.style={dashed,mark=x,mark size=2.2pt,mark options={solid}},
}

\providecommand{\PixelMatchCurve}[6]{%
  \addplot[draw=none,forget plot,name path=pixelLower,
    pixel keep={dataset}{#2},pixel keep={eps}{#3},pixel keep={rule}{#4}]
    table[x=B,y=ci_low] {#1};
  \addplot[draw=none,forget plot,name path=pixelUpper,
    pixel keep={dataset}{#2},pixel keep={eps}{#3},pixel keep={rule}{#4}]
    table[x=B,y=ci_high] {#1};
  \addplot[draw=none,fill=#5,fill opacity=0.16,forget plot]
    fill between[of=pixelLower and pixelUpper];
  \addplot[color=#5,thick,#6,forget plot,
    pixel keep={dataset}{#2},pixel keep={eps}{#3},pixel keep={rule}{#4}]
    table[x=B,y=speedup_mean] {#1};
}

\providecommand{\PixelBothPanel}[2]{%
  \PixelMatchCurve{\PixelDataFull}{#1}{#2}{rmc}{pixelRMC}{pixel full}%
  \PixelMatchCurve{\PixelDataInternal}{#1}{#2}{rmc}{pixelRMC}{pixel internal}%
  \PixelMatchCurve{\PixelDataFull}{#1}{#2}{d-grs}{pixelDGRS}{pixel full}%
  \PixelMatchCurve{\PixelDataInternal}{#1}{#2}{d-grs}{pixelDGRS}{pixel internal}%
  \PixelMatchCurve{\PixelDataFull}{#1}{#2}{paws}{pixelRASS}{pixel full}%
  \PixelMatchCurve{\PixelDataInternal}{#1}{#2}{paws}{pixelRASS}{pixel internal}%
}

\begin{tikzpicture}
  \begin{groupplot}[
    group style={
      group size=3 by 2,
      horizontal sep=0.3cm,vertical sep=0.6cm,
      xlabels at=edge bottom,xticklabels at=edge bottom,
      ylabels at=edge left,yticklabels at=edge left
    },
    width=0.255\linewidth,height=0.18\linewidth,
    scale only axis,
    xmode=log,log basis x=10,
    xmin=5,xmax=500,xtick={10,30,100,300},
    log ticks with fixed point,
    xlabel={Parallel budget $B$},
    ylabel={Mean speedup},
    axis lines=box,tick pos=left,tick align=outside,
    tick label style={font=\scriptsize},
    label style={font=\small},title style={font=\small,inner sep=0pt,yshift=2pt},
    every axis title shift=0pt,
    scaled ticks=false,
    yticklabel style={/pgf/number format/fixed,
      /pgf/number format/precision=1,/pgf/number format/fixed zerofill},
    unbounded coords=discard,filter discard warning=false,
    grid=major,
    major grid style={densely dotted,draw=black,opacity=0.2},
    cifar/.style={ymin=1.85,ymax=2.85,ytick={2.0,2.2,2.4,2.6,2.8}},
    ffhq/.style={ymin=1.85,ymax=2.32,ytick={1.9,2.0,2.1,2.2,2.3}},
  ]
    \nextgroupplot[
        cifar,
        title={CIFAR-10, $\varepsilon=0.1$},
        legend to name=pixelBothLegend,
        legend columns=6,
        legend style={draw=none,font=\small,/tikz/every even column/.append style={column sep=0.4em}},
    ]
    \PixelBothPanel{cifar10}{0.1}
    \addlegendimage{pixelDGRS,thick,mark=*}\addlegendentry{D-GRS}
    \addlegendimage{pixelRASS,thick,mark=*}\addlegendentry{RASS}
    \addlegendimage{pixelRMC,thick,mark=*}\addlegendentry{RMC}
    \addlegendimage{empty legend}\addlegendentry{}
    \addlegendimage{black,thick,pixel full}\addlegendentry{$B=|\mathcal{V}|$}
    \addlegendimage{black,thick,pixel internal}\addlegendentry{$B=|\mathcal{I}|$}

    \nextgroupplot[cifar,title={CIFAR-10, $\varepsilon=0.3$}]
    \PixelBothPanel{cifar10}{0.3}

    \nextgroupplot[cifar,title={CIFAR-10, $\varepsilon=0.6$}]
    \PixelBothPanel{cifar10}{0.6}

    \nextgroupplot[ffhq,title={FFHQ, $\varepsilon=0.1$}]
    \PixelBothPanel{ffhq}{0.1}

    \nextgroupplot[ffhq,title={FFHQ, $\varepsilon=0.3$}]
    \PixelBothPanel{ffhq}{0.3}

    \nextgroupplot[ffhq,title={FFHQ, $\varepsilon=0.6$}]
    \PixelBothPanel{ffhq}{0.6}
  \end{groupplot}
  \node[anchor=south] at (current bounding box.north){\pgfplotslegendfromname{pixelBothLegend}};
\end{tikzpicture}
  \caption{Pixel-space speedup under the two parallel-budget definitions. Mean target-call speedup on CIFAR-10 (top) and FFHQ (bottom) for churn $\varepsilon \in \{0.1, 0.3, 0.6\}$ and $(K,L) \in \{(2,3),(2,5),(3,4),(4,4)\}$. Solid lines evaluate the full tree, $B = |\mathcal{V}| \in \{15, 63, 121, 341\}$, as in Figure~\ref{fig:pixel-speedup-full-tree}. Dashed lines evaluate only the internal nodes, $B = |\mathcal{I}| \in \{7, 31, 40, 85\}$, as in Figure\ref{fig:pixel-space-speedup-internal}. RMC uses a chain matched to each budget. Curves average over 500 images per configuration, and shading shows 95\% confidence intervals. RASS exceeds D-GRS in every configuration under both budget definitions.}
  \label{fig:pixel-space-speedup-full-internal-trade-off}
\end{figure}

\paragraph{Latent-space text-to-image.}
Figure~\ref{fig:sd3-speedup-internal-full-tradeoff} compares the best speedup within budget
for SD3.5 under the two budget definitions. Figure~\ref{fig:sd3-speedup-internal-tree} shows all configurations with $B = |\mathcal{I}|$. The grid gives $B = |\mathcal{I}| = K + 1 \in \{3, \dots, 7\}$ for $L = 2$ and $B \in \{7, 13, 21, 31, 43\}$ for $L = 3$.
At a fixed configuration, the effect of internal-node evaluation depends on the lookahead. For $L = 2$, it decreases the speedup of RASS and D-GRS by $2.1\%$ to $3.7\%$. For $L = 3$, the change is at most $1.2\%$ in magnitude. The mean speedups on SD3.5 are near $1.1\times$, so few rounds probably reach the third level. The leaf evaluations then seldom change the next proposal. The speedup of RMC changes by less than $1\%$.
At a fixed budget, the two definitions give similar results. Within $B \leq 43$, the best mean speedup of RASS is $1.152\times$ with $B = |\mathcal{I}|$ and $1.151\times$ with $B = |\mathcal{V}|$. With the internal-node budget, RASS gives a higher mean speedup than D-GRS in all ten configurations, by $0.2\%$ to $1.2\%$. The $95\%$ confidence intervals overlap in all ten configurations, so no single difference is significant. RASS reaches $1.144\times$ at $B = 31$. This value exceeds the highest speedup of D-GRS, $1.141\times$ at $B = 43$, with $28\%$ less budget. RMC stays between $1.093\times$ and $1.099\times$. It exceeds RASS only at $(K,L) = (2,2)$, where $B = 3$ ($1.098\times$ against $1.094\times$, with overlapping confidence intervals).
\begin{figure}[t]
  \centering
%
%
%

\begingroup

\def\RuleDGRS{d-grs}
\def\RuleRASS{paws}
\def\RuleRMC{rmc}

\def\SDMaxBudget{43}

\definecolor{DGRSColor}{RGB}{31,119,180}
\definecolor{RASSColor}{RGB}{255,127,14}
\definecolor{RMCColor}{RGB}{44,160,44}

\def\DataFile{figures/latent-space/data/sd3-speedup-internal-tree.csv}

\pgfplotsset{
    sd extra/.style={},
    sd select/.style 2 args={
        x filter/.code={
            \edef\sdRow{\thisrow{rule}}%
            \edef\sdWant{#1}%
            \ifx\sdRow\sdWant\else\def\pgfmathresult{nan}\fi
            \edef\sdLWant{#2}%
            \ifx\sdLWant\empty\else
                \ifnum\thisrow{L}=\sdLWant\relax\else\def\pgfmathresult{nan}\fi
            \fi
            \ifnum\thisrow{budget}>\SDMaxBudget\relax\def\pgfmathresult{nan}\fi
        },
        sd extra,
        unbounded coords=discard,
        filter discard warning=false,
    },
    sd skip/.style 2 args={
        x filter/.append code={
            \ifnum\thisrow{K}=#1\relax
                \ifnum\thisrow{L}=#2\relax\def\pgfmathresult{nan}\fi
            \fi
        },
    },
}

\newcommand{\sdSeries}[8]{%
    \addplot[name path=#1-#2-#3-hi, draw=none, forget plot, sd select={#1}{#2}]
        table[x=#3, y=#6] {\DataFile};
    \addplot[name path=#1-#2-#3-lo, draw=none, forget plot, sd select={#1}{#2}]
        table[x=#3, y=#5] {\DataFile};
    \addplot[draw=none, fill=#7, fill opacity=0.13, forget plot]
        fill between[of=#1-#2-#3-hi and #1-#2-#3-lo];
    \addplot[color=#7, #8, sd select={#1}{#2}]
        table[x=#3, y=#4] {\DataFile};
}

\begin{tikzpicture}

\begin{groupplot}[
    group style={
        group size=3 by 1,
        horizontal sep=0.35cm,
        yticklabels at=edge left,
        ylabels at=edge left,
    },
    scale only axis,
    width=0.265\linewidth,
    height=0.16\linewidth,
    ymin=1.076, ymax=1.166,
    ytick={1.08,1.09,...,1.161},
    yticklabel style={
        /pgf/number format/.cd, fixed, fixed zerofill, precision=2,
    },
    grid=major,
    grid style={line width=.1pt, draw=gray!20},
    axis line style={black!70},
    tick align=outside,
    tick pos=left,
    tick style={black!70},
    table/col sep=comma,
    title style={font=\small, yshift=-6pt},
    label style={font=\small},
    tick label style={font=\scriptsize},
    every axis plot/.append style={
        line width=1.1pt,
    },
    legend columns=3,
    legend style={
        draw=none,
        font=\small,
        /tikz/every even column/.append style={column sep=0.8em},
    },
]

\nextgroupplot[
    title={Best within budget},
    title style={font=\small, yshift=-2pt},
    xlabel={Parallel Budget $B=|\mathcal{I}|$},
    ylabel={Speculative speedup},
    xmode=log,
    xmin=2.6, xmax=50,
    xtick={3,5,10,20,40},
    log ticks with fixed point,
    legend to name=sd3internallegend,
    sd extra/.style={sd skip={6}{2}},
]
\sdSeries{\RuleDGRS}{}{budget}{best_speedup}{best_ci_low}{best_ci_high}{DGRSColor}{mark=*, mark size=1.5pt}
\addlegendentry{D-GRS}
\sdSeries{\RuleRASS}{}{budget}{best_speedup}{best_ci_low}{best_ci_high}{RASSColor}{mark=square*, mark size=1.5pt}
\addlegendentry{RASS}
\sdSeries{\RuleRMC}{}{budget}{best_speedup}{best_ci_low}{best_ci_high}{RMCColor}{mark=triangle*, mark size=1.7pt}
\addlegendentry{RMC}

\nextgroupplot[
    title={Lookahead $L=2$},
    xlabel={Branching factor $K$},
    xmin=1.8, xmax=6.2,
    xtick={2,...,6},
]
\sdSeries{\RuleDGRS}{2}{K}{mean_speedup}{ci_low}{ci_high}{DGRSColor}{mark=*, mark size=1.5pt}
\sdSeries{\RuleRASS}{2}{K}{mean_speedup}{ci_low}{ci_high}{RASSColor}{mark=square*, mark size=1.5pt}
\sdSeries{\RuleRMC}{2}{K}{mean_speedup}{ci_low}{ci_high}{RMCColor}{mark=triangle*, mark size=1.7pt}

\nextgroupplot[
    title={Lookahead $L=3$},
    xlabel={Branching factor $K$},
    xmin=1.8, xmax=6.2,
    xtick={2,...,6},
]
\sdSeries{\RuleDGRS}{3}{K}{mean_speedup}{ci_low}{ci_high}{DGRSColor}{mark=*, mark size=1.5pt}
\sdSeries{\RuleRASS}{3}{K}{mean_speedup}{ci_low}{ci_high}{RASSColor}{mark=square*, mark size=1.5pt}
\sdSeries{\RuleRMC}{3}{K}{mean_speedup}{ci_low}{ci_high}{RMCColor}{mark=triangle*, mark size=1.7pt}

\end{groupplot}

\node[anchor=south, inner sep=0pt] at ([yshift=-3pt]group c2r1.outer north)
    {\pgfplotslegendfromname{sd3internallegend}};

\end{tikzpicture}

\endgroup
  \caption{Mean target-call speedup for SD3.5-medium when verification evaluates only the internal nodes, $B = |\mathcal{I}|$. Left: best observed mean speedup within parallel budget $B$ (logarithmic axis). Middle and right: speedup versus branching factor $K$ at lookahead $L=2$ and $L=3$, respectively. Shading indicates 95\% confidence intervals over 100 paired caption--noise inputs. RASS attains a higher mean speedup than D-GRS in every configuration.}
  \label{fig:sd3-speedup-internal-tree}
\end{figure}
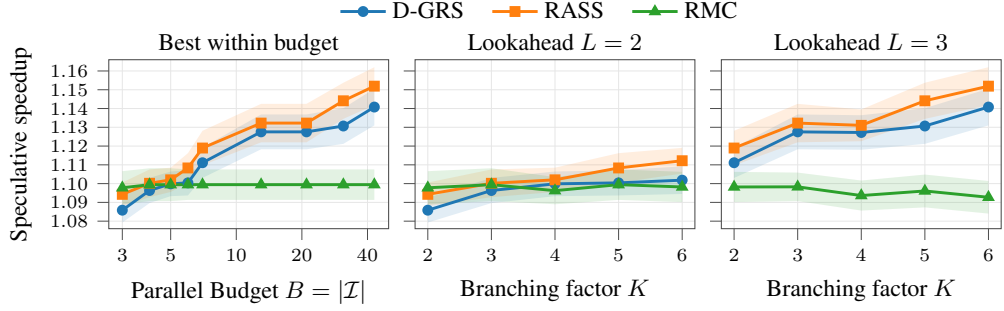
\begin{figure}[t]
  \centering
  \input{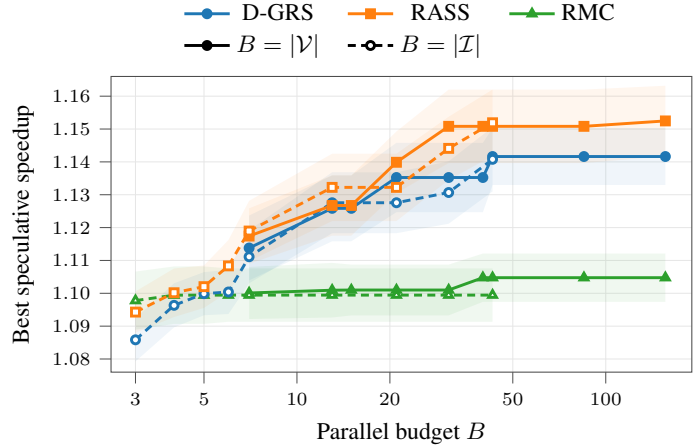}
  \caption{Best mean target-call speedup within parallel budget $B$ for SD3.5-medium under two budget definitions. Solid lines with filled markers evaluate the full tree, $B = |\mathcal{V}|$, as in Figure~\ref{fig:sd3-speedup}; dashed lines with open markers evaluate only the internal nodes, $B = |\mathcal{I}|$. Each point shows the best configuration $(K,L)$ whose budget does not exceed $B$; the budget axis is logarithmic. Shading indicates 95\% confidence intervals over 100 paired caption--noise inputs. Evaluating only internal nodes, RASS reaches its highest mean speedup, $1.152\times$, at $B = 43$ rather than $B = 156$.}
    \label{fig:sd3-speedup-internal-full-tradeoff}
\end{figure}

\paragraph{Summary.}
Neither budget definition is better in all settings.
Internal-node evaluation uses a smaller batch per round.
Full-tree evaluation gives a more accurate cache for the next round.
The better choice depends on the model, the lookahead, and the hardware cost of a
larger batch.
The order of the methods does not change.
Under both definitions, RASS gives a higher mean speedup than D-GRS in every
evaluated configuration.

\subsection{Qualitatively Assessment of Exactness}
\label{app:qualitative-exactness}

Figures~\ref{fig:exactness-cifar10}, \ref{fig:exactness-ffhq}, \ref{fig:exactness-sd3} show samples from the
standard target sampler alongside RMC, D-GRS and RASS. Rows are chosen as follows:
we draw four sample indices uniformly at random without replacement from the 500 images of each pixel-space run and from the 100 prompts of the SD3.5 run. We use the main-paper protocol with $\varepsilon=0.3$,
$K=4$, $L=4$ for CIFAR-10 and FFHQ, and $\varepsilon=0.8$, $K=6$, $L=3$ for SD3.5.

All three speculative samplers are exact, so each produces a sample from the same distribution as the target sampler. On CIFAR-10 the class label is shared across methods, but the four columns are independent draws. On SD3.5 the prompt and the
initial noise are shared, so images often share a composition but differ in
detail, because each sampler consumes its randomness differently along the trajectory. Across all three datasets the speculative samples are indistinguishable in quality and content from those of the target sampler, and the
text-conditioned samples are consistent with their prompts.

\begin{figure}[h]
  \centering
  \includegraphics[width=0.45\linewidth]{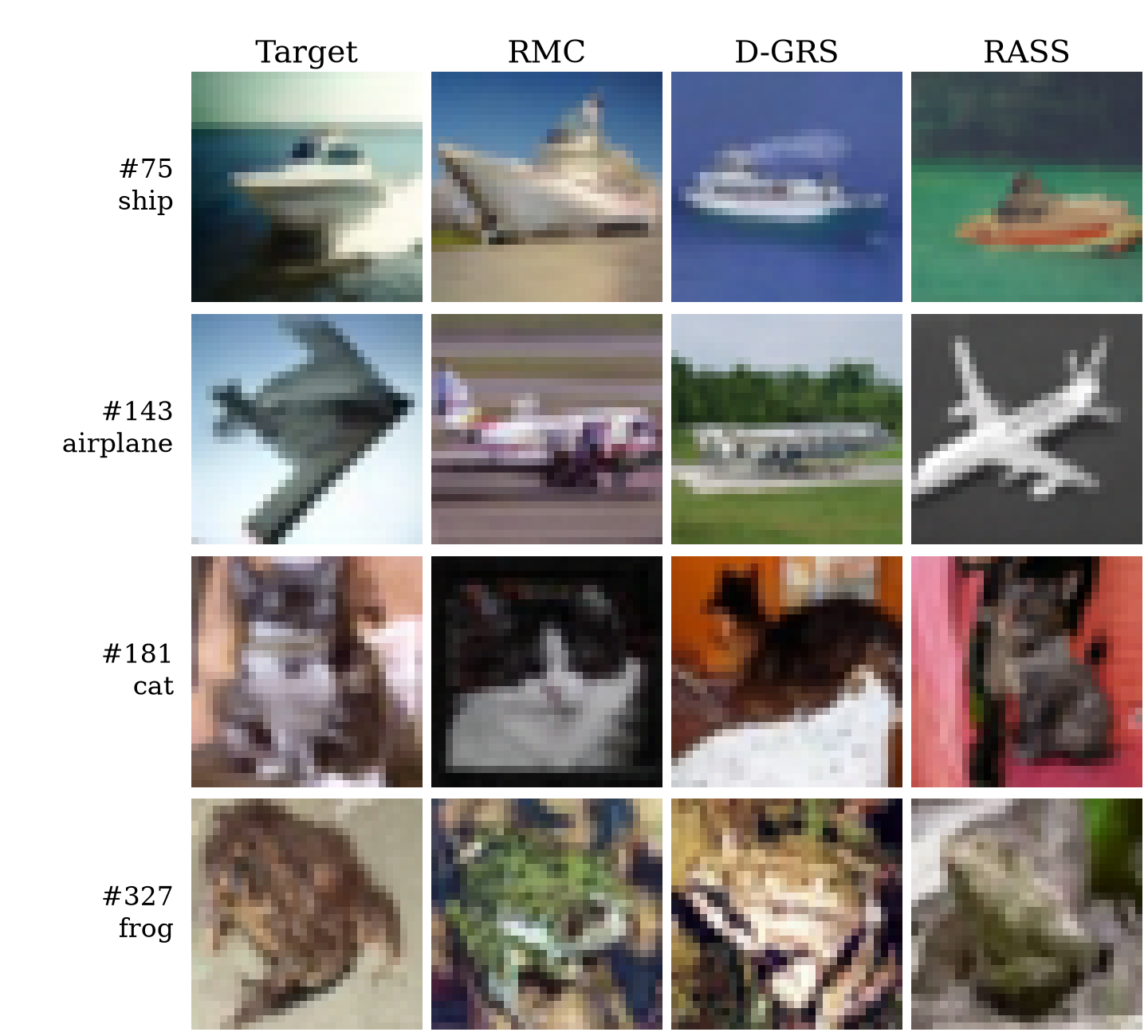}
  \caption{CIFAR-10 ($32\times32$, class-conditional), $\varepsilon=0.3$, $K=4$, $L=4$.}
  \label{fig:exactness-cifar10}
\end{figure}

\begin{figure}[h]
  \centering
  \includegraphics[width=0.6\linewidth]{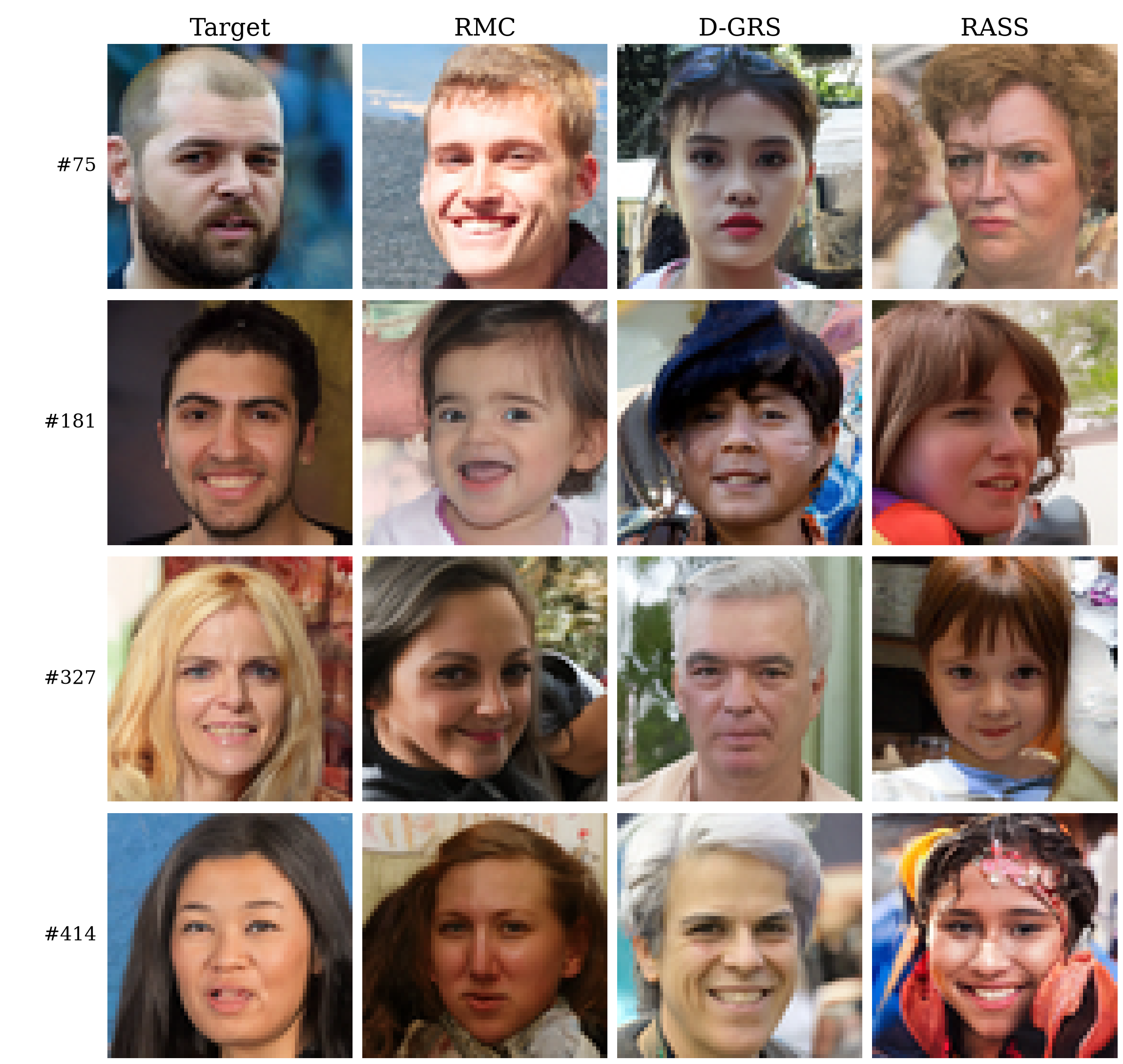}
  \caption{FFHQ ($64\times64$, unconditional), $\varepsilon=0.3$, $K=4$, $L=4$.}
  \label{fig:exactness-ffhq}
\end{figure}

\begin{figure}[h]
  \centering
  \includegraphics[width=0.85\linewidth]{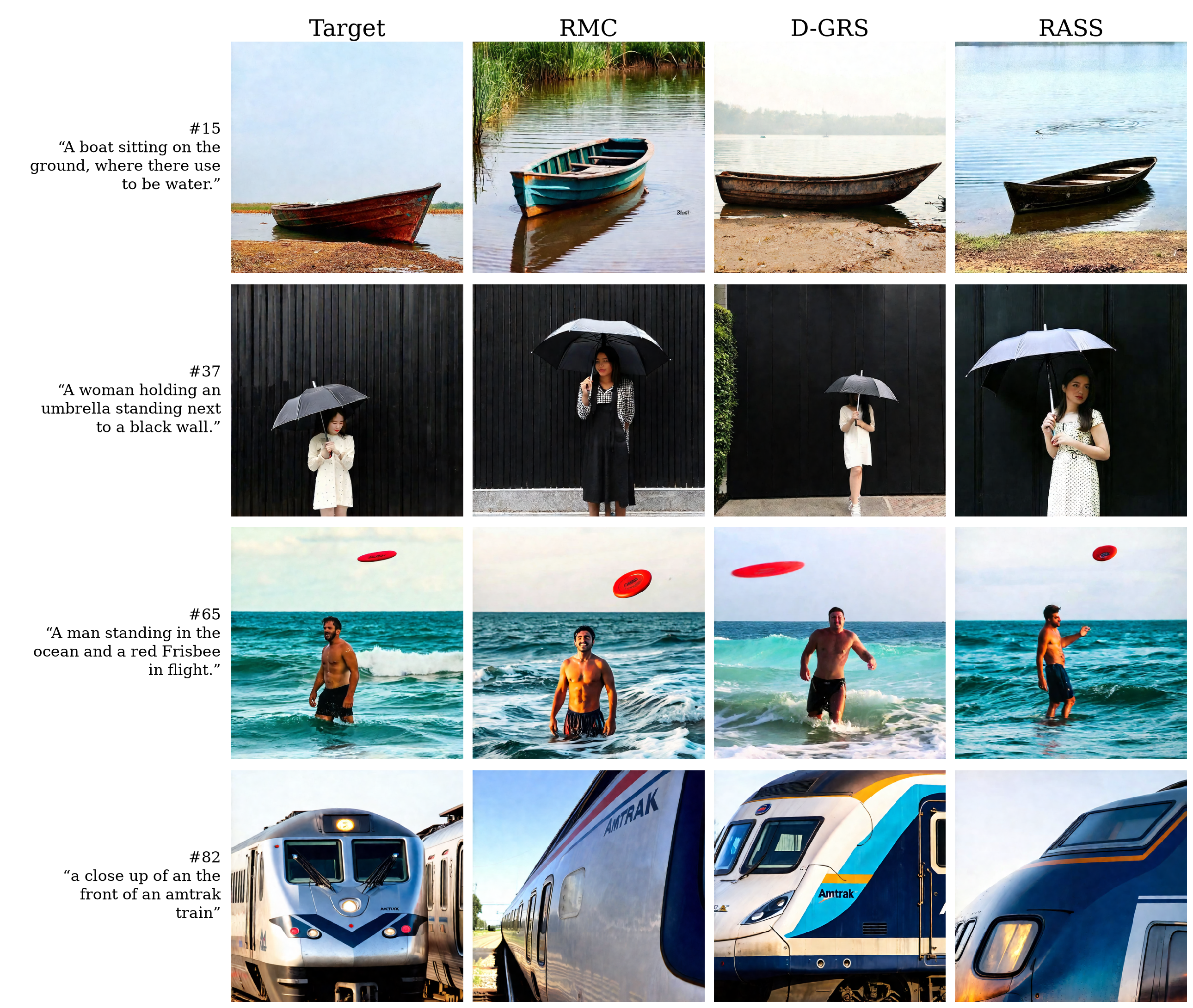}
  \caption{Stable Diffusion 3.5 Medium ($512\times512$), $\varepsilon=0.8$, $K=6$, $L=3$.}
  \label{fig:exactness-sd3}
\end{figure}



\end{document}